%% file: ijcai25.tex
\documentclass{article}
\usepackage{ijcai25}

\usepackage{times}
\usepackage{soul}
\usepackage{url}
\usepackage[hidelinks]{hyperref}
\usepackage[utf8]{inputenc}
\usepackage[small]{caption}
\usepackage{graphicx}
\usepackage{amsmath}
\usepackage{amsthm}
\usepackage{booktabs}
\usepackage{algorithm}
\usepackage{algorithmic}
\usepackage[switch]{lineno}

\usepackage{algorithm}
\usepackage{algorithmic}
\usepackage{subfigure}
\usepackage{epsfig}
\usepackage{graphicx}
\usepackage{url}
\usepackage{multirow}
\usepackage{amssymb}
\usepackage{mathrsfs}
\usepackage{epstopdf}
\usepackage{multicol}
\usepackage{amsmath}
\usepackage{mathtools}
\usepackage{bm}
\usepackage{color}
\usepackage{multicol}
\usepackage{wrapfig}
\usepackage{wrapfig,lipsum,booktabs}

\usepackage{multirow}

\usepackage{pifont}

\allowdisplaybreaks[4]

\theoremstyle{plain}
\newtheorem{theorem}{Theorem}[section]

\newtheorem{lemma}[theorem]{Lemma}

\theoremstyle{definition}
\newtheorem{definition}[theorem]{Definition}
\newtheorem{assumption}[theorem]{Assumption}
\theoremstyle{remark}
\newtheorem{remark}[theorem]{Remark}

\title{Federated Stochastic Bilevel Optimization with Fully First-Order Gradients}

\author{
	Yihan Zhang$^1$
	\and
	Rohit Dhaipule$^2$\and
	Chiu C. Tan$^1$\and
    Haibin Ling$^2$\And
	Hongchang Gao$^1$\footnote{Corresponding author}\notag \\
	\affiliations
	$^1$Temple University\notag \\
	$^2$Stony Brook University\notag \\
	\emails
    \{rdhaipule, hling\}@cs.stonybrook.edu,
    \{yihan.zhang0002, chiu.tan, hongchang.gao\}@temple.edu
}

\begin{document}
	
\maketitle

\begin{abstract}
	Federated stochastic bilevel optimization has been actively studied in recent years due to its widespread applications in machine learning. However, most existing federated stochastic bilevel optimization algorithms require the computation of second-order Hessian and Jacobian matrices, which leads to longer running times in practice. To address these challenges, we propose a novel federated stochastic variance-reduced bilevel gradient descent algorithm that relies solely on first-order oracles. Specifically, our approach does not require the computation of second-order Hessian and Jacobian matrices, significantly reducing running time.   Furthermore, we introduce a novel learning rate mechanism, i.e., a constant single-timescale learning rate, to coordinate the update of different variables. We also present a new strategy to establish the convergence rate of our algorithm.  Finally, the extensive experimental results confirm the efficacy of our proposed algorithm. 
\end{abstract}

\section{Introduction}
In this paper, we focus on the following federated stochastic bilevel optimization (FedSBO) problem:
\begin{align} \label{eq:loss}
	& \min_{x\in \mathbb{R}^{d_x}}  f(x, y^*(x)) \triangleq \frac{1}{N}\sum_{n=1}^{N} f^{(n)}(x, y^*(x)) \notag \\
	& s.t., y^*(x) = \arg\min_{y\in \mathbb{R}^{d_y}} g(x, y) \triangleq \frac{1}{N}\sum_{n=1}^{N}g^{(n)}(x, y) \  .
\end{align}
Here, $g(x, y) \triangleq \frac{1}{N}\sum_{n=1}^{N}g^{(n)}(x, y)$ denotes the lower-level loss function, where $g^{(n)}(x, y)=\mathbb{E}[g^{(n)}(x, y; \xi^{(n)})]$ denotes the lower-level loss function on the $n$-th  device, and  $f(x, y^*(x)) \triangleq \frac{1}{N}\sum_{n=1}^{N} f^{(n)}(x, y^*(x))$ represents the upper-level loss function, where $f^{(n)}(x, y^*(x))=\mathbb{E}[f^{(n)}(x, y^*(x); \xi^{(n)})]$ represents the upper-level loss function on the $n$-th device. Throughout this paper, we assume that the upper-level loss function is nonconvex with respect to both variables and the lower-level loss function is strongly convex with respect to $y$, which is a commonly used assumption in the existing literature \cite{ghadimi2018approximation,ji2021bilevel,chen2021tighter,tarzanagh2022fednest}.

The stochastic bilevel optimization problem has attracted increasing attention recently because many machine learning models belong to this class of optimization problem, such as model-agnostic meta-learning \cite{finn2017model},  hyperparameter optimization \cite{ji2021bilevel}, neural architecture search \cite{liu2018darts}, etc.  To solve the stochastic bilevel optimization problem, numerous optimization algorithms \cite{ghadimi2018approximation,ji2021bilevel,chen2021tighter,hong2020two,khanduri2021near,yang2021provably,guo2021randomized,dagreou2022framework,chu2024spaba,dagreou2024lower} have been developed in the single machine setting in the past few years.  However, enabling federated learning for the stochastic bilevel optimization problem is much more challenging than the standard single-level optimization problem.  Specifically, as shown in Eq.~(\ref{eq:loss}), the upper-level problem \textit{on each device} depends on the optimal solution $y^*(x)$ of the \textit{global} lower-level problem.  As a result, computing the stochastic hypergradient for the upper-level loss function on each device requires the global Jacobian and Hessian matrices, which introduce significant challenges for local computation and global communication in federated learning.

In recent years, numerous efforts \cite{gao2022convergence,tarzanagh2022fednest,li2024communication,yang2024simfbo,huang2023achieving} have been made to address the unique challenges in FedSBO problems. For example, \cite{gao2022convergence} developed a local stochastic bilevel gradient descent with momentum algorithms under the homogeneous setting, which can remove the dependence on global Jacobian and Hessian matrices due to the homogeneous data distribution. \cite{tarzanagh2022fednest} proposed FedNEST in the heterogeneous setting. It employs the Neumann series expansion approach in an inner loop to estimate Hessian-inverse-vector product. However, this method suffers from a large communication complexity because it requires to communicate every intermediate Hessian-inverse-vector product in that inner loop. \cite{li2024communication} proposed a single-loop algorithm to address the issue of large communication complexity. In particular, it introduces an additional gradient descent procedure to replace the Neumann series expansion approach to estimate the Hessian inverse vector product.

Nevertheless,  the aforementioned federated stochastic bilevel optimization algorithms still suffer from high computation costs. Specifically, those algorithms need to compute the second-order Jacobian and Hessian matrices. It is computationally expensive, especially when the problem is high-dimensional. In fact, to avoid computing the second-order Jacobian and Hessian matrices, \cite{kwon2023fully,shen2023penalty,kwon2024complexity,chen2023near} developed a fully first-order method in the single-machine setting. Specifically, \cite{kwon2023fully} converts the bilevel optimization problem into a single-level one,  and then only the first-order gradient is required to solve it, significantly reducing computational costs. More specifically, \cite{kwon2023fully} converts the lower-level optimization problem as a constraint and then leverages the penalty approach to further convert it to an unconstrained minimax optimization problem, which can be solved by the  first-order stochastic gradient descent ascent algorithm. Inspired by this, we aim to develop an efficient federated stochastic bilevel optimization algorithm that uses  only first-order gradients to address the aforementioned challenges. 
However, developing first-order methods for the FedSBO problem in Eq.~(\ref{eq:loss}) presents unique challenges, which are outlined below.
\begin{itemize}
	\item First, existing first-order methods \cite{kwon2023fully} for the single-machine setting employ the \textit{iteration-dependent learning rate} to guarantee convergence. However, this strategy is not practical in federated learning. Specifically, from the iteration-dependent learning rate, the participating device is able to infer the current training stage, which could be used to attack the training procedure by attackers. Therefore, it is necessary to develop \textit{a constant learning rate} to avoid this problem while guaranteeing convergence. 
	\item  Second, existing first-order methods \cite{kwon2023fully} under the single-machine setting employ a \textit{two-timescale learning rate} to guarantee convergence. Specifically, some variables have an order-wise larger learning rate than the others. This kind of learning rate is difficult to tune for practical federated learning practitioners. Thus, it is necessary to propose a  \textit{single-timescale learning rate} to make it easy to tune while ensuring convergence. 
	\item Third, the first-order method under the single-machine setting requires a penalty hyperparameter to guarantee convergence. However,  it is unclear how this penalty hyperparameter affects the consensus error in federated learning. Specifically, it is unclear whether the impact of the penalty hyperparameter on the consensus error will lead to a slower convergence rate. Therefore, it is necessary to establish the theoretical convergence rate, revealing how the penalty hyperparameter affects the consensus error and the convergence rate in federated learning. 
\end{itemize}

To address the aforementioned unique challenges, we develop a novel federated stochastic variance-reduced bilevel gradient descent algorithm with fully first-order oracles. Specifically, on the algorithm design side, we develop a novel single-timescale constant learning rate, where we demonstrate how this new learning rate depends on the penalty hyperparameter. On the theoretical analysis side, we propose a novel strategy to establish the convergence rate of our algorithm. In particular, we develop a novel potential function for convergence analysis, where we demonstrate how to combine different components together with carefully designed coefficients. With these novel algorithmic and theoretical designs, our algorithm can achieve the $O(\frac{1}{N\epsilon^5})$ convergence rate to obtain the $\epsilon$-accuracy solution, which indicates the linear speedup with respect to the number of devices $N$. Finally, extensive experimental results confirm the efficacy of our new algorithm.  In summary, our paper has made the following contributions.
\begin{itemize}
	\item We develop a novel federated stochastic variance-reduced bilevel gradient descent algorithm, which uses the fully first-order gradient and the single-timescale constant learning rate. To the best of our knowledge, this is the first time a single-timescale constant learning rate has been shown to be applied to the first-order  bilevel optimization algorithms. 
	\item We establish the convergence rate of our algorithm, where we demonstrate how to use a potential function to combine different estimation errors together with carefully designed coefficients. As far as we know, this is the first work to provide convergence guarantees for first-order federated bilevel optimization algorithms with the single-timescale constant learning rate . 
	\item Our extensive experimental results on various tasks confirm the efficacy of our algorithm, i.e., our new algorithm is more computationally efficient than existing second-order methods. 
\end{itemize}

\section{Related Works}
\subsection{Stochastic Bilevel Optimization}
Since the upper-level loss function depends on the optimal solution of the lower-level optimization problem, the hypergradient of the upper-level loss function needs to compute $\partial y^*(x)/\partial x$, which brings unique challenges for optimizing this class of optimization problems. Specifically, computing $\partial y^*(x)/\partial x$ depends on Jacobian and Hessian inverse matrices. To compute them, many efforts \cite{ghadimi2018approximation,ji2021bilevel,chen2021tighter,hong2020two,khanduri2021near,yang2021provably,guo2021randomized,dagreou2022framework,chu2024spaba,dagreou2024lower} have been made in the single machine setting recently.  For example, \cite{ghadimi2018approximation} proposed using the Neumann series expansion approach to approximate the Hessian-inverse-vector product, based on which the convergence rate of the stochastic gradient descent is established.  
In this direction, several improved algorithms have been developed. 
For example,  \cite{ji2021bilevel} developed a mini-batch stochastic gradient descent method to improve the convergence rate with a large batch size. \cite{hong2020two} studied the convergence rate of the stochastic gradient descent with a two-timescale learning rate. \cite{chen2021tighter} established the convergence rate of the stochastic gradient descent with an alternating update strategy.  To further improve the convergence rate, \cite{khanduri2021near,yang2021provably,guo2021randomized} introduces the variance reduction techniques \cite{cutkosky2019momentum,fang2018spider,nguyen2017sarah} into stochastic bilevel optimization, whose convergence rate $O(\epsilon^{-3})$ can match their counterparts for the single-level optimization problem. In addition to the Neumann series expansion approach, there exists another approach for estimating the hypergradient. Specifically, the Hessian-inverse-vector product is viewed as the optimal solution of a quadratic optimization problem, and then an additional gradient descent procedure is introduced to estimate the Hessian-inverse-vector product.  Based on this approach, a couple of algorithms have been developed. For example, \cite{dagreou2022framework,chu2024spaba,dagreou2024lower} combined it with variance reduction techniques, whose convergence rates $O(\epsilon^{-3})$ can also match their counterparts for the single-level optimization problem. 

Since the aforementioned approaches need to compute the second-order Jacobian and Hessian inverse matrices, \cite{kwon2023fully,shen2023penalty}  developed the fully first-order method to address this issue. In particular, they transformed the lower-level optimization problem as a constraint and then used the penalty approach to further convert it as an unconstrained minimax optimization problem. However, the penalty hyperparameter significantly affects the convergence rate. For example, \cite{kwon2023fully} shows that the standard stochastic gradient descent ascent algorithm can only achieve the convergence rate of $O(\epsilon^{-7})$ and can be improved to $O(\epsilon^{-5})$  with the variance reduction technique.  However, these methods do not require computation of the second-order Jacobian and Hessian inverse matrices. As a result, their practical performance is more efficient in terms of running time, especially when the problem is high-dimensional.

\subsection{Federated Stochastic Bilevel Optimization}
To enable distributed optimization for stochastic bilevel optimization problems, numerous algorithms \cite{gao2022convergence,tarzanagh2022fednest,li2024communication,yang2024simfbo,huang2023achieving,gao2023convergence,zhang2023communication} have recently been developed. For example, based on the Neumann series expansion approach, \cite{gao2022convergence} developed local stochastic bilevel gradient descent with momentum algorithms under the homogeneous setting, which only need to communicate two variables and their gradient estimators. Its communication complexity $O(\epsilon^{-1})$ can match its counterpart for the single machine setting.   \cite{tarzanagh2022fednest} proposed FedNEST for the heterogeneous setting, where every intermediate Hessian-inverse-vector product in the the Neumann series expansion approach should be communicated, leading to a large communication complexity. Later, \cite{huang2023achieving} developed the FedMBO algorithm, which does not need to communicate the intermediate Hessian-inverse-vector product. However, the upper-level variable in FedMBO needs to be communicated in every iteration, which also leads to a large communication complexity.  Based on the other approach for estimating Hessian-inverse-vector product, \cite{li2024communication} developed a federated stochastic bilevel optimization algorithm with the variance reduction technique, whose communication complexity can match that of its counterpart for the single-level problem. However, all these existing algorithms need to compute the second-order Jacobian and Hessian matrices, which is computationally expensive for high-dimensional problems. As far as we know, there do not exist first-order federated bilevel optimization algorithms.


\begin{algorithm*}[ht]
	\caption{FedSVRBGD-FO}
	\label{alg_sim}
	\begin{algorithmic}[1]
		\REQUIRE $x_0$, $y_0$, $z_0$, $\eta>0$, $\alpha_x>0$, $\alpha_y>0$, $\alpha_{z}>0$,  $\beta_x>0$, $\beta_y>0$, $\beta_z>0$.
		
		\STATE Initialization when $t=0$:  compute the stochastic gradient with a mini-batch of samples where the batch size is $B$: 
		$m^{(n)}_{x, 0}  =\nabla_1 f^{(n)}(x^{(n)}_{0}, y^{(n)}_{0}; \xi^{(n)}_{0}) + \lambda (\nabla_1 g^{(n)}(x^{(n)}_{ 0}, y^{(n)}_{ 0}; \xi^{(n)}_{ 0}) -  \nabla_1 g^{(n)}(x^{(n)}_{ 0}, z^{(n)}_{ 0}; \xi^{(n)}_{ 0}) )$ \ , \\ 
		$m^{(n)}_{y, 0} = \nabla_2 f^{(n)}(x^{(n)}_{ 0}, y^{(n)}_{ 0}; \xi^{(n)}_{ 0}) + \lambda  \nabla_2 g^{(n)}(x^{(n)}_{ 0}, y^{(n)}_{ 0}; \xi^{(n)}_{ 0})$ \ , 
		$m^{(n)}_{z, 0} =\lambda  \nabla_2 g^{(n)}(x^{(n)}_{ 0}, z^{(n)}_{ 0}; \xi^{(n)}_{ 0}) $  \ ,  \\
		\FOR{$t=0,\cdots, T-1$, each device $n$} 
		\STATE  Update $x$:\quad ${x}^{(n)}_{t+1} = x^{(n)}_{t} -\alpha_x \eta m^{(n)}_{x, t}$ \ ,    \\
		\STATE  Update $y$ :\quad $y^{(n)}_{t+1} = y^{(n)}_{t} - \alpha_y\eta m^{(n)}_{y, t}$ \ ,\\
		\STATE  Update $z$:\quad $z^{(n)}_{t+1} = z^{(n)}_{t} -  \alpha_z\eta  m^{(n)}_{z, t}$ \ ,  \\
		\STATE Update the variance-reduced gradient $m^{(n)}_{x, t+1}$:  \\
		$ u^{(n)}_{1, t+1} =  (1-\beta_x\eta^2)(u^{(n)}_{1, t} - \nabla_1 f^{(n)}(x^{(n)}_{t}, y^{(n)}_{t}; \xi^{(n)}_{t+1})) + \nabla_1 f^{(n)}(x^{(n)}_{t+1}, y^{(n)}_{t+1}; \xi^{(n)}_{t+1}) $ \ , \\
		$ u^{(n)}_{2, t+1} =  (1-\beta_x\eta^2)(u^{(n)}_{2, t} - \nabla_1 g^{(n)}(x^{(n)}_{t}, y^{(n)}_{t}; \xi^{(n)}_{t+1})) + \nabla_1 g^{(n)}(x^{(n)}_{t+1}, y^{(n)}_{t+1}; \xi^{(n)}_{t+1}) $ \ , \\
		$ u^{(n)}_{3, t+1} =  (1-\beta_x\eta^2)(u^{(n)}_{3, t} - \nabla_1 g^{(n)}(x^{(n)}_{t}, z^{(n)}_{t}; \xi^{(n)}_{t+1})) + \nabla_1 g^{(n)}(x^{(n)}_{t+1}, z^{(n)}_{t+1}; \xi^{(n)}_{t+1}) $\ ,  \\
		$ m^{(n)}_{x, t+1} = u^{(n)}_{1, t+1}  +\lambda (u^{(n)}_{2, t+1} - u^{(n)}_{3, t+1}) $ \ ,   \\
		
		\STATE Update the variance-reduced gradient $m^{(n)}_{y, t+1}$:  \\
		$ v^{(n)}_{1, t+1} =  (1-\beta_y\eta^2)(v^{(n)}_{1, t} - \nabla_2 f^{(n)}(x^{(n)}_{t}, y^{(n)}_{t}; \xi^{(n)}_{t+1})) + \nabla_2 f^{(n)}(x^{(n)}_{t+1}, y^{(n)}_{t+1}; \xi^{(n)}_{t+1})$ \ ,  \\
		$ v^{(n)}_{2, t+1} =  (1-\beta_y\eta^2)(v^{(n)}_{2, t} - \nabla_2 g^{(n)}(x^{(n)}_{t}, y^{(n)}_{t}; \xi^{(n)}_{t+1})) + \nabla_2 g^{(n)}(x^{(n)}_{t+1}, y^{(n)}_{t+1}; \xi^{(n)}_{t+1})$ \ ,  \\
		$m^{(n)}_{y, t+1} =  v^{(n)}_{1, t+1}  + \lambda  v^{(n)}_{2, t+1}$ \ ,  \\
		
		\STATE Update the variance-reduced gradient $m^{(n)}_{z, t+1}$:  \\
		$ w^{(n)}_{1, t+1} =  (1-\beta_z\eta^2)(w^{(n)}_{1, t} - \nabla_2 g^{(n)}(x^{(n)}_{t}, z^{(n)}_{t}; \xi^{(n)}_{t+1})) + \nabla_2 g^{(n)}(x^{(n)}_{t+1}, z^{(n)}_{t+1}; \xi^{(n)}_{t+1}) $ \ ,  \\
		$m^{(n)}_{z, t+1} =  \lambda w^{(n)}_{1, t+1}$ \ ,  \\
		
		\IF {$\text{mod}(t+1, p)==0$}
		\STATE ${x}_{t+1}^{(n)} = \bar{x}_{t+1}=  \frac{1}{N}\sum_{n'=1}^{N}{x}_{t+1}^{(n')}$\ ,  \ ${y}_{t+1}^{(n)}= \bar{y}_{t+1}=  \frac{1}{N}\sum_{n'=1}^{N}{y}_{t+1}^{(n')}$ \ , \  ${z}_{t+1}^{(n)}= \bar{z}_{t+1}=  \frac{1}{N}\sum_{n'=1}^{N}{z}_{t+1}^{(n')}$,   \\
		${u}_{1, t+1}^{(n)}=\bar{u}_{1, t+1}=  \frac{1}{N}\sum_{n'=1}^{N}{u}_{1,  t+1}^{(n')}$ \ ,   \ 	${u}_{2, t+1}^{(n)}=\bar{u}_{2, t+1}=  \frac{1}{N}\sum_{n'=1}^{N}{u}_{2,  t+1}^{(n')}$\ ,   \ 	${u}_{3, t+1}^{(n)}=\bar{u}_{3, t+1}=  \frac{1}{N}\sum_{n'=1}^{N}{u}_{3,  t+1}^{(n')}$ \ ,   \\
		
		${v}_{1, t+1}^{(n)}=\bar{v}_{1, t+1}=  \frac{1}{N}\sum_{n'=1}^{N}{v}_{1,  t+1}^{(n')}$ \ , \  ${v}_{2, t+1}^{(n)}=\bar{v}_{2, t+1}=  \frac{1}{N}\sum_{n'=1}^{N}{v}_{2,  t+1}^{(n')}$\ , \    
		${w}_{1, t+1}^{(n)}=\bar{w}_{1, t+1}=  \frac{1}{N}\sum_{n'=1}^{N}{w}_{1,  t+1}^{(n')}$ \ ,
		\ENDIF

		\ENDFOR
	\end{algorithmic}
\end{algorithm*}

\section{Algorithm Design}

\subsection{Problem Definition}
According to \cite{kwon2023fully}, Eq.~(\ref{eq:loss}) can be converted to a constrained optimization problem as follows:
\begin{align}
	& \min_{x\in \mathbb{R}^{d_x}, y\in \mathbb{R}^{d_y}}  f(x, y) \ , \notag \\
	& s.t. \quad g(x, y) - g(x, y^*(x)) \leq  0   \ .  
\end{align}
Then, by leveraging the penalty approach, this constrained optimization problem can be converted to an unconstrained minimax optimization problem as follows:
\begin{align}
	& \min_{x\in \mathbb{R}^{d_x}, y\in \mathbb{R}^{d_y}}\max_{z\in\mathbb{R}^{d_y}}\underbrace{f(x, y) + \lambda(g(x, y)- g(x, z))}_{\mathcal{L}_{\lambda}(x, y, z)} \ , 
\end{align}
where $\lambda>0$ is the penalty hyperparameter. Denoting
\begin{align}
	& 	\mathcal{L}(x) = f(x, y^*(x))  \ ,  \notag \\ 
	& 	\mathcal{L}^*_{\lambda}(x)  =f(x, y^*_{\lambda}(x)) + \lambda (g(x, y^*_{\lambda}(x)) - g(x, y^*(x)))  \  ,   \notag
\end{align}
where $y^*_{\lambda}(x) =\arg\min_{y\in \mathbb{R}^{d_y}}f(x, y) + \lambda(g(x, y) - g(x, y^*(x)))$, existing works \cite{kwon2023fully} have shown that $\mathcal{L}^*_{\lambda}(x)$ is a good approximation to $\mathcal{L}(x)$  when given an appropriate penalty hyperparameter $\lambda$. 




Based on this reformulation, we can solve Eq.~(\ref{eq:loss}) by optimizing the following minimax problem:
\begin{align} \label{eq:loss-first-order}
	& \min_{x, y}\max_{z}\frac{1}{N}\sum_{n=1}^{N}\underbrace{\left(f^{(n)}(x, y) + \lambda(g^{(n)}(x, y)- g^{(n)}(x, z))\right)}_{\mathcal{L}^{(n)}_{\lambda}(x, y, z)} \  . 
\end{align}
Obviously, the gradients $\nabla_x \mathcal{L}^{(n)}_{\lambda}(x, y, z)$, $\nabla_y \mathcal{L}^{(n)}_{\lambda}(x, y, z)$, and $\nabla_z \mathcal{L}^{(n)}_{\lambda}(x, y, z)$ on the $n$-th device  do not  depend on the second-order Jacobian and Hessian matrices. Thus, Eq.~(\ref{eq:loss-first-order}) can be solved efficiently. 

\subsection{Our Algorithm}
To solve Eq.~(\ref{eq:loss-first-order}), we develop a novel federated stochastic variance-reduced bilevel gradient descent algorithm with fully first-order oracles (FedSVRBGD-FO) in Algorithm~\ref{alg_sim}.  Specifically, in the $t$-th iteration, device $n$ computes the stochastic variance-reduced gradient for each component of $\nabla_x \mathcal{L}^{(n)}_{\lambda}(x, y, z)$ as follows:
\begin{align}
	& u^{(n)}_{1,  t} =  (1-\beta_x\eta^2)(u^{(n)}_{1, t-1} - \nabla_1 f^{(n)}(x^{(n)}_ {t-1}, y^{(n)}_ {t-1}; \xi^{(n)}_ {t})) \notag \\
	& \quad \quad \quad + \nabla_1 f^{(n)}(x^{(n)}_ {t}, y^{(n)}_ {t}; \xi^{(n)}_ {t}) \ , \notag \\
	& u^{(n)}_{2,  t} =  (1-\beta_x\eta^2)(u^{(n)}_{2, t} - \nabla_1 g^{(n)}(x^{(n)}_ {t-1}, y^{(n)}_ {t-1}; \xi^{(n)}_ {t})) \notag \\
	& \quad \quad \quad + \nabla_1 g^{(n)}(x^{(n)}_ {t}, y^{(n)}_ {t}; \xi^{(n)}_ {t}) \ ,  \notag \\
	& u^{(n)}_{3,  t} =  (1-\beta_x\eta^2)(u^{(n)}_{3, t} - \nabla_1 g^{(n)}(x^{(n)}_ {t-1}, z^{(n)}_ {t-1}; \xi^{(n)}_ {t}))\notag \\
	& \quad \quad \quad  + \nabla_1 g^{(n)}(x^{(n)}_ {t}, z^{(n)}_ {t}; \xi^{(n)}_ {t})  \ , 
\end{align}
where $\eta>0$ is the learning rate and $\beta_x>0$ is a constant hyperparameter, which satisfies $\beta_x\eta^2<1$. Then, our algorithm composes the variance-reduced gradient estimator for $\nabla_x \mathcal{L}^{(n)}_{\lambda}(x, y, z)$ as follows:
\begin{align}
	& m^{(n)}_{x, t} = u^{(n)}_{1, t}  +\lambda (u^{(n)}_{2, t} - u^{(n)}_{3, t})  \ . 
\end{align}
This gradient estimator is then used to update the local variable as follows:
\begin{align}
	& {x}^{(n)}_{t+1} = x^{(n)}_{t} -\alpha_x \eta m^{(n)}_{x, t} \ ,
\end{align}
where $\alpha_x>0$ is a constant hyperparameter.  Then, at every $p$ iterations, where $p>1$, the local variable ${x}^{(n)}_{t+1}$ and gradient estimators $u^{(n)}_{1,  t+1}$, $u^{(n)}_{2,  t+1}$, $u^{(n)}_{3,  t+1}$ are uploaded to the central server and then reset to the global ones, which is shown in Step 9 in Algorithm~\ref{alg_sim}.  The other two variables are updated in the same way. 

It is worth noting that the learning rate $\eta$ and the associated hyperparameters $\alpha_{x}$, $\alpha_{y}$, $\alpha_{z}$ in our Algorithm~\ref{alg_sim} are constant, rather than iteration-dependent as existing methods in the single-machine setting.

\section{Convergence Analysis}

\subsection{Assumption}
To establish the convergence rate of our Algorithm~\ref{alg_sim}, we introduce the following assumptions that are commonly used in existing work \cite{kwon2023fully,chen2023near}.

\begin{assumption} \label{assumption:f-smooth}
	For $n\in\{1, \cdots, N\}$, the upper-level function $f^{(n)}(\cdot, \cdot)$ is $L_{f}$-smooth in expectation where $L_{f}>0$ is a constant, and it is $C_f$-Lipschitz with respect to the second variable in expectation. 
\end{assumption}

\begin{assumption}\label{assumption:g-smooth}
	For $n\in\{1, \cdots, N\}$, the lower-level function $g^{(n)}(\cdot, \cdot)$ is $L_{g, 1}$-smooth  in expectation where $L_{g, 1}>0$ is a constant and $\nabla^2 g(\cdot, \cdot)$ is $L_{g, 2}$-Lipschitz  in expectation where $L_{g, 2}>0$ is a constant. 
\end{assumption}

\begin{assumption}\label{assumption:g-scvx}
	For $n\in\{1, \cdots, N\}$, the lower-level function $g^{(n)}(\cdot, \cdot)$ is $\mu$-strongly-convex with respect to the second variable. 
\end{assumption}

Based on Assumptions~\ref{assumption:f-smooth}-\ref{assumption:g-scvx}, \cite{kwon2023fully} shows that $\mathcal{L}_{\lambda}(x, y, z)$ is $\frac{\lambda\mu}{2}$-strongly-convex in $y$ when $\lambda>\frac{2L_{f}}{\mu}$.  


\begin{assumption}\label{assumption:variance}
	The stochastic gradients of the loss functions at the upper and lower levels have an upper bounded variance $\sigma^2$ where $\sigma>0$ is a constant.
\end{assumption}

\begin{assumption} \label{assumption:heterogeneous}
	The data distribution between workers is heterogeneous and the upper-level and lower-level gradients satisfy the following conditions:
	\begin{align}
		& \frac{1}{N}\sum_{n=1}^{N}\|\nabla  f^{(n)}(x, y) - \frac{1}{N}\sum_{n'=1}^{N}\nabla  f^{(n')}(x, y) \|^2 \leq \delta^2  \ , \notag \\
		& \frac{1}{N}\sum_{n=1}^{N}\|\nabla  g^{(n)}(x, y) - \frac{1}{N}\sum_{n'=1}^{N}\nabla  g^{(n')}(x, y) \|^2 \leq \delta^2 \  , 
	\end{align}
	where $\delta>0$ is a constant, $x\in \mathbb{R}^{d_x}$, and $y\in \mathbb{R}^{d_y}$. 
\end{assumption}

Following \cite{chen2023near}, we introduce the following definition.
\begin{definition}
	Given Assumptions~\ref{assumption:f-smooth}-\ref{assumption:g-scvx}, we denote $\ell=\max\{C_f, L_f, L_{g, 1}, L_{g, 2}\}$ and $\kappa=\ell/\mu$. 
\end{definition}

With the aforementioned assumptions and definitions, the following lemma shows how well $\mathcal{L}^*_{\lambda}(x)$ approximates $\mathcal{L}(x)$. 
\begin{lemma}\label{lemma:approximation}
	\cite{chen2023near,kwon2023fully}
	Given Assumptions~\ref{assumption:f-smooth}-\ref{assumption:g-scvx}, by setting $\lambda>\frac{2L_{f}}{\mu}$, then $\mathcal{L}^*_{\lambda}(x)$ satisfies the following conditions:
	\begin{itemize}
		\item $\nabla 	\mathcal{L}^*_{\lambda}(x)$ is $L$-Lipschitz, where $L=O(\ell\kappa^3)$;
		\item $\|\nabla \mathcal{L}^*_{\lambda}(x) - \nabla \mathcal{L}(x) \| = O(\ell \kappa^3/\lambda)$ for any $x\in \mathbb{R}^{d_x}$;
		\item $| \mathcal{L}^*_{\lambda}(x) -  \mathcal{L}(x) | = O(\ell \kappa^2/\lambda)$ for any $x\in \mathbb{R}^{d_x}$. 
	\end{itemize}

\end{lemma}

\subsection{Convergence Rate}
Based on Assumptions~\ref{assumption:f-smooth}-\ref{assumption:heterogeneous}, we established the convergence rate of our algorithm below, whose proof can be found in Appendix~\ref{sec:theorem}.

\begin{theorem}\label{theorem:convergence-rate}
	Given Assumptions~\ref{assumption:f-smooth}-\ref{assumption:heterogeneous}, by setting $ \lambda = O\left(\frac{\kappa^3}{\epsilon}\right)$, $\beta_x = O\left(\frac{1}{N}\right)$, $\beta_y = O\left(\frac{1}{N}\right)$, $ \beta_z = O\left(\frac{1}{N}\right)$, $\alpha_x = O\left(\frac{1}{\lambda \kappa^3}\right)$, $\alpha_y = O\left(\frac{1}{\lambda \kappa}\right)$, $\alpha_z = O\left(\frac{1}{\lambda \kappa}\right)$, $\eta = O\left(\frac{N\epsilon^2}{\kappa^4}\right)$, $p =  O\left(\frac{\kappa}{N\epsilon}\right)$, $B = O\left(\frac{\kappa^6}{N\epsilon^3}\right)$, and $T = O\left(\frac{\kappa^{10}}{N\epsilon^5}\right) $, Algorithm~\ref{alg_sim} can achieve the $\epsilon$-accuracy solution as follows:
	\begin{align}
		& \frac{1}{T}\sum_{t=0}^{T-1}  \mathbb{E}[\|\nabla \mathcal{L}(\bar{x}_{t})\|^2]  \leq O\left(\epsilon^2\right)  \ . 
	\end{align}
\end{theorem}

\begin{remark}
	The complexity of the iteration $T = O\left(\frac{\kappa^{10}}{N\epsilon^5}\right) $ indicates that our algorithm can achieve a linear speedup with respect to the number of devices $N$.  To the best of our knowledge, this is the first time that the first-order algorithm can achieve the linear speedup for federated bilevel optimization problems. Moreover, when $N=1$, i.e., the single-machine setting, our convergence rate can match that of the single-machine counterpart in \cite{kwon2023fully}.   Furthermore, 	the communication complexity of Algorithm~\ref{alg_sim} is $\frac{T}{p} = O\left(\frac{\kappa^{9}}{\epsilon^4}\right)$. 
\end{remark}

\begin{remark}
	The existing method  \cite{kwon2023fully} under the single-machine setting employs a time-dependent learning rate. For example, the learning rate in the $t$-th iteration for $x$ is of the order of $O(t^{-3/5})$.  On the contrary, our Algorithm~\ref{alg_sim}  employs a constant learning rate  $\eta$ that is independent of the current iteration. Moreover, \cite{kwon2023fully} employs a two-timescale learning rate. Specifically, the learning rate for $x$ is $O(t^{-3/5})$, while that for $z$ is $O(t^{-2/5})$. Then, $\lim_{t \rightarrow \infty } t^{-3/5}/ t^{-2/5}=0$. In contrast, the learning rate for three variables in Algorithm~\ref{alg_sim}, i.e., $\alpha_{x}\eta$, $\alpha_{y}\eta$, and $\alpha_{z}\eta$, are in the same timescale \footnote{In this paper, the same timescale refers to having the same order with respect to $\epsilon$ or the number of iterations.}. Therefore, our learning rate setting is easier to tune in practice. 
\end{remark}

\subsection{Proof Sketch}
To establish the convergence of Algorithm~\ref{alg_sim}, we develop a novel potential function as follows:
{
\small
\begin{align}
	& P_{t} = \mathbb{E}[\mathcal{L}_{\lambda}^{*}(\bar{x}_{t})] + c_1 \lambda \mathbb{E}[\|\bar{y}_{t} - y_{\lambda}^{*}(\bar{{x}}_{t})\|^2]  + c_2 \lambda \mathbb{E}[\|\bar{z}_{t} - y^{*}(\bar{{x}}_{t})\|^2] \notag  \\
	& \quad + c_3\mathbb{E}[\|\frac{1}{N}\sum_{n=1}^{N}\nabla_{1} f^{(n)}(x^{(n)}_{t}, y^{(n)}_{t})- \frac{1}{N}\sum_{n=1}^{N} u^{(n)}_{1, t} \|^2]  \notag  \\
	& \quad + c_4 \lambda^2 \mathbb{E}[\|\frac{1}{N}\sum_{n=1}^{N}\nabla_{1} g^{(n)}(x^{(n)}_{t}, y^{(n)}_{t})- \frac{1}{N}\sum_{n=1}^{N} u^{(n)}_{2, t} \|^2] \notag  \\
	& \quad + c_5\lambda^2\mathbb{E}[\|\frac{1}{N}\sum_{n=1}^{N}\nabla_{1} g^{(n)}(x^{(n)}_{t}, z^{(n)}_{t})- \frac{1}{N}\sum_{n=1}^{N} u^{(n)}_{3, t} \|^2] \notag  \\
	& \quad  + c_6\mathbb{E}[\|\frac{1}{N}\sum_{n=1}^{N} \nabla_{2} f^{(n)}({x}_{t}^{(n)}, {y}_{t}^{(n)}) -\frac{1}{N}\sum_{n=1}^{N} v^{(n)}_{1, t}  \|^2] \notag  \\
	& \quad + c_7\lambda^2 \mathbb{E}[\|\frac{1}{N}\sum_{n=1}^{N}  \nabla_{2} g^{(n)}({x}_{t}^{(n)}, {y}_{t}^{(n)})-\frac{1}{N}\sum_{n=1}^{N}     v^{(n)}_{2, t} \|^2] \notag \\
	& \quad  + c_8\lambda^2 \mathbb{E}[\| \frac{1}{N}\sum_{n=1}^{N} \nabla_{2} g^{(n)}({x}^{(n)}_{t}, {z}^{(n)}_{t}) - \frac{1}{N}\sum_{n=1}^{N}{w}^{(n)}_{1, t} \|^2]  \ , 
\end{align}
}
where $\{c_i\}_{i=1}^{8}$ are positive coefficients.  

A key design in our potential function is that $c_1$ and $c_2$ are associated with $\lambda$ while $c_4$, $c_5$, $c_7$, and $c_8$ are associated with $\lambda^2$. As such, \textbf{all coefficients $\{c_i\}_{i=1}^{8}$ are independent of the penalty parameter $\lambda$. Otherwise, $\{c_i\}_{i=1}^{8}$ will depend on $\lambda$, which will further affect the variance term and then result in a worse convergence rate}.  For example, as shown in Eq.~(\ref{eq:p-t+1-p-t}), the coefficient $\{c_i\}_{i=3}^{8}$ affects the terms regarding the variance in the last line. If they depended on the penalty parameter,  the learning rate $\eta$ has to be much smaller to control the variance term. For example, when $c_4=O(\lambda^2)$, $\eta$ should be as small as $O(\epsilon^{6})$ because of $\lambda = O(\frac{1}{\epsilon})$, leading to a much slower convergence rate. 

\begin{figure*}[h]
  \centering
  \hspace{-11pt}
    \subfigure[a9a]{
    \includegraphics[scale=0.41]{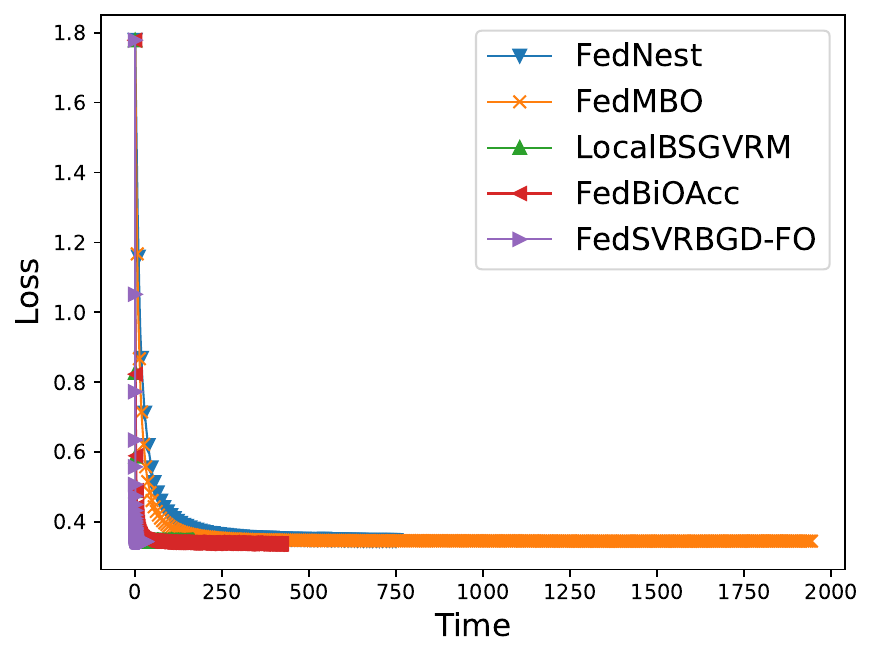}}
    \hspace{-10pt}
  \subfigure[w8a]{
    \includegraphics[scale=0.41]{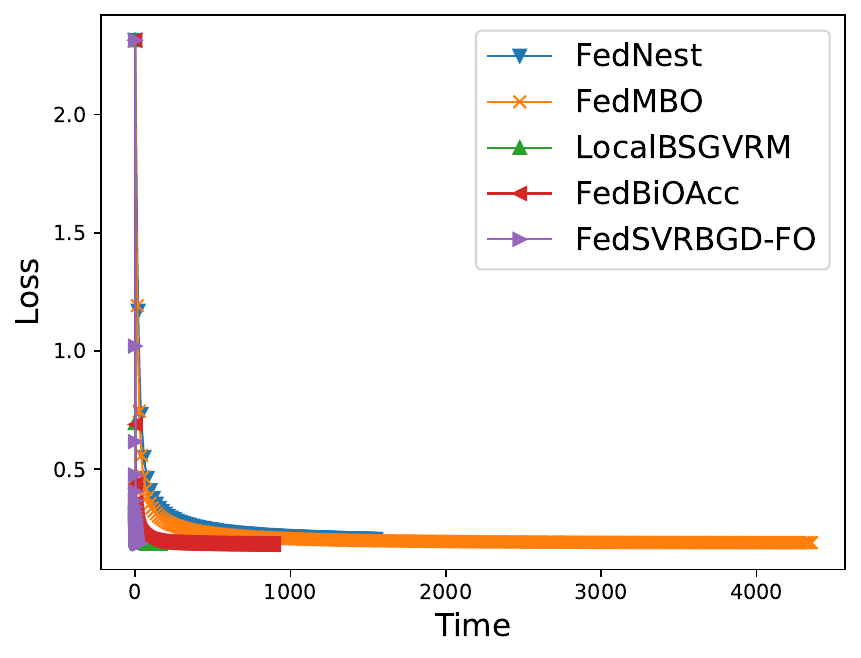}}
    \hspace{-10pt}
  \subfigure[covtype]{
    \includegraphics[scale=0.41]{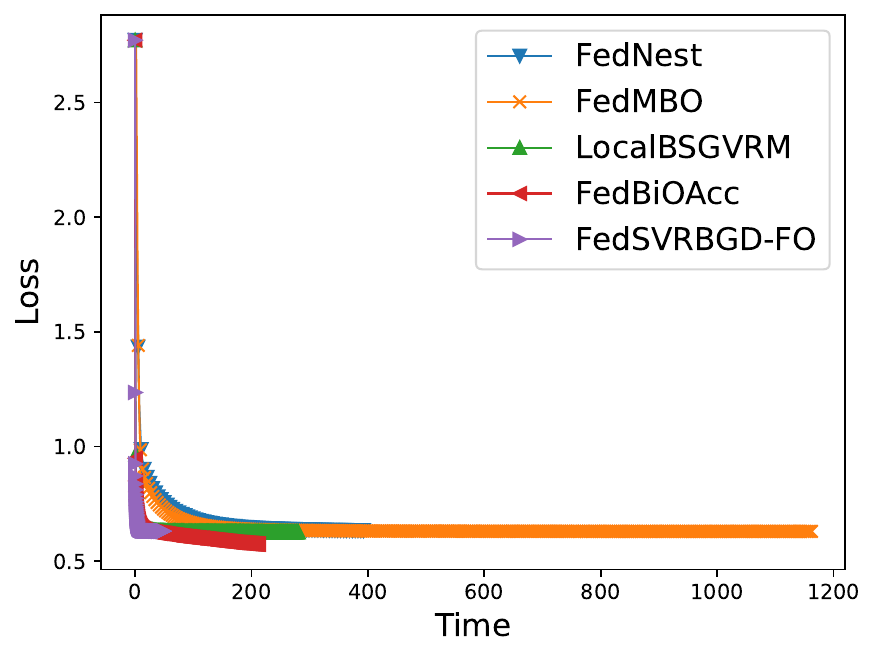}}
    \caption{The upper-level loss function value versus the running time (seconds). Communication period $p=4$.}
  \label{fig:all-datasets-running-time-hyper-para-opt}
\end{figure*}

Based on this potential function, we first establish the upper bound for each term of this potential function in the Appendix~\ref{sec:lemma} and then identify the coefficient $\{c_i\}_{i=1}^{8}$ in Appendix~\ref{sec:theorem} to eliminate all their associated terms in the upper bound of $P_{t+1} -P_{t}$, resulting in Eq.~(\ref{eq:remove-all-terms}). Then, we can establish the convergence rate of our Algorithm~\ref{alg_sim}.

Moreover, a key step for establishing the final convergence rate is to bound the consensus error in the following lemma, whose proof can be found in Appendix~\ref{sec:lemma}.
\begin{lemma} \label{lemma:consensus-error-main-text}
	Given Assumptions~\ref{assumption:f-smooth}-\ref{assumption:heterogeneous}, and 
    $\eta \leq \frac{1}{500 p \sqrt{L_{f}^2+ L_{g, 1}^2 \lambda^2} }$,  $\beta_x\leq \frac{{L_f^2+ L_{g, 1}^{2}  \lambda^2}}{N}$, 
     $\beta_y\leq \frac{{L_f^2+ L_{g, 1}^{2}  \lambda^2}}{N}$,  $\beta_z\leq \frac{{L_f^2+ L_{g, 1}^{2}  \lambda^2}}{N}$, 
    $\alpha_x\leq 10$,  $\alpha_y\leq 10$,  $\alpha_z\leq 10$, 
we have
\begin{align}
\small
	&  \frac{1}{T}\sum_{t=0}^{T-1}\frac{1}{N}\sum_{n=1}^{N}\Big(\mathbb{E}[\|u^{(n)}_{1, t} -\bar{u}_{1, t}\|^2]+ \lambda^2\mathbb{E}[\|u^{(n)}_{2, t} -\bar{u}_{2, t}\|^2] \notag \\
    & \quad +\lambda^2\mathbb{E}[\|u^{(n)}_{3, t} -\bar{u}_{3, t}\|^2] + \mathbb{E}[\|v^{(n)}_{1, t} -\bar{v}_{1, t}\|^2] \notag \\
    & \quad +\lambda^2 \mathbb{E}[\|v^{(n)}_{2, t} -\bar{v}_{2, t}\|^2] + \lambda^2\mathbb{E}[\|w^{(n)}_{1, t} -\bar{w}_{1, t}\|^2]\Big) \notag\\
	& \leq  576p^2\alpha_x^2\eta^2\Big(L_{f}^2+ L_{g, 1}^2  \lambda^2 \Big) \frac{1}{T}\sum_{t=0}^{T-1}\mathbb{E}[\| \bar{m}_{x, t} \|^2]  \notag \\
    & \quad + 576p^2\alpha_y^2\eta^2 \Big(L_{f}^2+ L_{g, 1}^2\lambda^2\Big)\frac{1}{T}\sum_{t=0}^{T-1} \mathbb{E}[\| \bar{m}_{y, t} \|^2] \notag \\
	& \quad  + 576 p^2\alpha_z^2\eta^2 \Big(L_{f}^2+ L_{g, 1}^2\lambda^2\Big)\frac{1}{T}\sum_{t=0}^{T-1}\mathbb{E}[\|   \bar{m}_{z, t}  \|^2]  \notag \\
    & \quad +576p^2\beta_x^2\eta^4( 1+  \lambda^2 )\delta^2   +576p^2\beta_y^2\eta^4(  1+  \lambda^2 )   \delta^2 \notag \\
	& \quad + 576p^2\beta_x^2 \eta^4(  1+ \lambda^2)\sigma^{2} + 576p^2\beta_y^2\eta^4(  1 + \lambda^2  )\sigma^{2} \notag \\
    & \quad + 576p^2\beta_z^2\eta^4 \lambda^2\delta^2 +576p^2\beta_z^2\eta^4\lambda^2\sigma^{2} \ . 
\end{align}
\end{lemma}
In this lemma, we reveal how the penalty hyperparameter $\lambda$ and the communication period $p$ affect the consensus error. With such a lemma, we can then establish the convergence rate of our Algorithm~\ref{alg_sim}.

\section{Experiments}
In this section, we evaluate the performance of our Algorithm~\ref{alg_sim} on the commonly used benchmark task: hyperparameter optimization and hyper-representation learning. 

\subsection{Hyperparameter Optimization}
In the  experiment, we apply  Algorithm~\ref{alg_sim} to the following hyperparameter optimization problem:

\begin{equation} \label{eq:exp-hyperparameter-opt}
\small
	\begin{aligned}
		& \min_{x\in \mathbb{R}^d} \frac{1}{N} \sum_{n=1}^{N}\frac{1}{m^{(n)}}\sum_{i=1}^{m^{(n)}}\ell(y^*(x)^Ta_{v, i}^{(n)}, b_{v,i}^{(n)}) \\
		& s.t. \  y^*(x) =  \arg\min_{y\in \mathbb{R}^{d\times c}}\frac{1}{N} \sum_{n=1}^{N} \frac{1}{m^{(n)}}\sum_{i=1}^{m^{(n)}}  \ell(y^Ta_{t, i}^{(n)}, b_{t, i}^{(n)})   \\
		& \quad \quad \quad\quad \quad \quad\quad \quad \quad + \frac{1}{cd}\sum_{p=1}^{c}\sum_{q=1}^{d} \exp(x_q)y_{pq}^2 \ ,
	\end{aligned}
\end{equation}
where the lower-level  problem is to learn  the parameter $y\in \mathbb{R}^{d\times c}$ of the logistic regression model with  the training set $\{(a_{t,i}^{(n)}, b_{t,i}^{(n)})\}_{i=1}^{m^{(n)}}$, the upper-level  problem learns the the regularization coefficient $x\in \mathbb{R}^d$ with  the validation set $\{(a_{v,i}^{(n)}, b_{v,i}^{(n)})\}_{i=1}^{m^{(n)}}$. 

In this experiment, we used three benchmark datasets: a9a, w8a,  covtype, which are obtained from LIBSVM datasets \footnote{\url{https://www.csie.ntu.edu.tw/~cjlin/libsvmtools/datasets/}}. Here, $10\%$ of the samples are randomly selected as the test set. For the remaining samples,  $70\%$ of them are randomly selected as the training set and the others are used as the validation set.  The training and validation sets are then randomly distributed to eight workers.  The batch size for each worker is set to $10$ in this experiment. To verify the performance of our Algorithm~\ref{alg_sim}, we compare it with four state-of-the-art second-order methods: FedNEST \cite{tarzanagh2022fednest}, FedMBO \cite{huang2023achieving}, LocalBSGVRM \cite{gao2022convergence}, and FedBiOAcc \cite{li2024communication}. In our experiment, we set the solution accuracy $\epsilon$ to $0.1$. Then, according to the theoretical results in \cite{tarzanagh2022fednest,huang2023achieving,gao2022convergence,li2024communication}, the learning rate of FedNEST and FedMBO is set to $\epsilon^2$, while that of LocalBSGVRM and FedBiOAcc is set to $\epsilon$. Regarding our method, according to Theorem~\ref{theorem:convergence-rate}, we set the learning rate $\eta$ to $\epsilon^2$, the coefficient $\alpha_x=\alpha_y=\alpha_z=\epsilon$, and the penalty $\lambda=5/\epsilon$. Moreover, for all methods using the momentum-based variance reduction technique, we set the the coefficient of the momentum to $0.1$. Then, we run all experiments on a workstation with 4 NVIDIA A5000 GPU cards, each of which accommodates two threads to simulate eight workers.  

In Figure~\ref{fig:all-datasets-running-time-hyper-para-opt}, we plot the upper-level loss function value versus the running time (seconds). Specifically, to make a fair comparison, we ran all methods to update $x$ for 16,000 iterations on all datasets.
The communication period is set to $4$ in this experiment. From Figure~\ref{fig:all-datasets-running-time-hyper-para-opt}, we can find that our algorithm requires much less time to converge than all baseline methods. This confirms the advantage of using the fully first-order gradient in practical applications.  

 \begin{wrapfigure}[12]{R}{1.7in}
	\vspace{-10pt}
	\centering
	\includegraphics[scale=0.25]{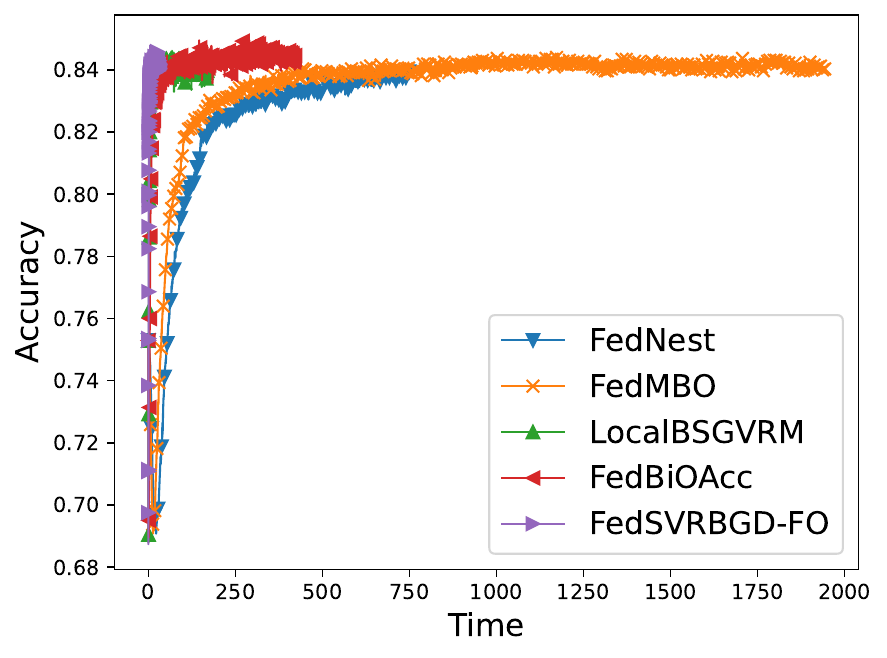}
    \caption{The test accuracy versus the running time (seconds) for a9a dataset. Communication period $p=4$.}
  \label{fig:acc-a9a}
\end{wrapfigure} 

In addition, in Figure~\ref{fig:acc-a9a}, we plot the accuracy in the test set versus the running time for the a9a dataset. In Figure~\ref{fig:acc-a9a}, we can still find that our algorithm uses much less time to achieve almost the same accuracy as the baseline methods, further confirming the efficacy of using fully first-order gradients for federated bilevel optimization. 


\begin{wrapfigure}[13]{R}{1.7in}
	\vspace{-10pt}
	\centering
	\includegraphics[scale=0.25]{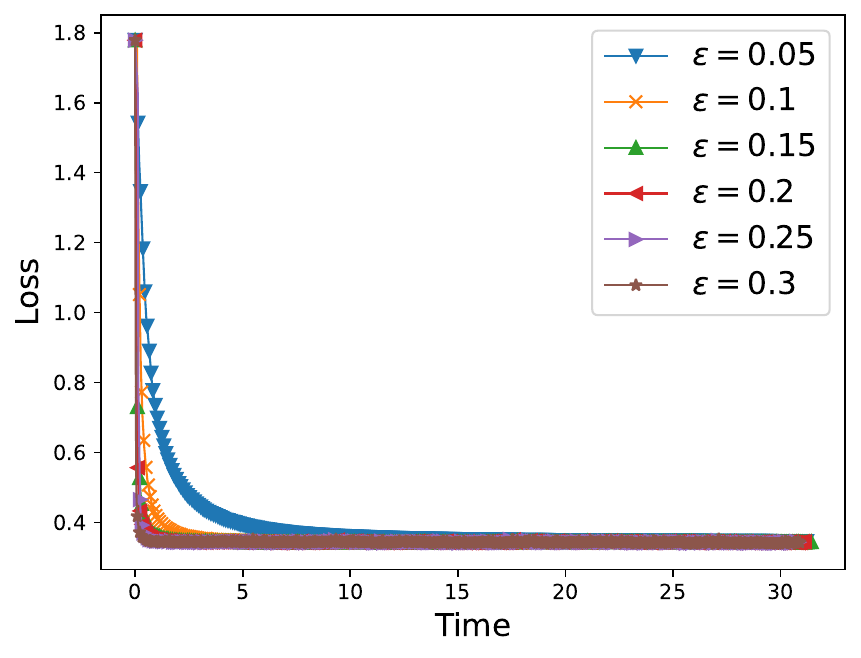}
    \caption{The upper-level loss function value when using different $\epsilon$ for our method on a9a dataset. Communication period $p=4$.}
  \label{fig:epsilon-a9a}
\end{wrapfigure} 

Because the learning rate and penalty parameter are set according to the accuracy of the solution $\epsilon$ as suggested by Theorem~\ref{theorem:convergence-rate}, we added an additional experiment to show the influence of $\epsilon$ on the convergence rate, which is shown in Figure~\ref{fig:epsilon-a9a}. It can be observed that a smaller $\epsilon$ leads to a slower convergence rate. The reason is that a smaller $\epsilon$ results in a smaller learning rate as shown in Theorem~\ref{theorem:convergence-rate}, slowing the convergence.

Furthermore, to verify the impact of the communication period on convergence performance, we performed an additional experiment with $p=16$. In Figure~\ref{fig:all-datasets-running-time-hyper-para-opt-16}, we plot the upper-level loss function value versus the running time in Figure~\ref{fig:all-datasets-running-time-hyper-para-opt-loss-16} and plot the test accuracy in Figure~\ref{fig:all-datasets-running-time-hyper-para-opt-acc-16}. Note that FedNest becomes more efficient than FedBiOAcc because it computes the second-order Hessian and Jacobian matrices in each communication round, rather than each iteration. Thus, it computes second-order information less frequently when the communication period becomes large. However, we can still find that our algorithm needs less time to converge than all baseline methods when the communication period is $16$.  

\begin{figure}[h]
  \centering
    \subfigure[Loss]{
    \label{fig:all-datasets-running-time-hyper-para-opt-loss-16}
    \includegraphics[scale=0.25]{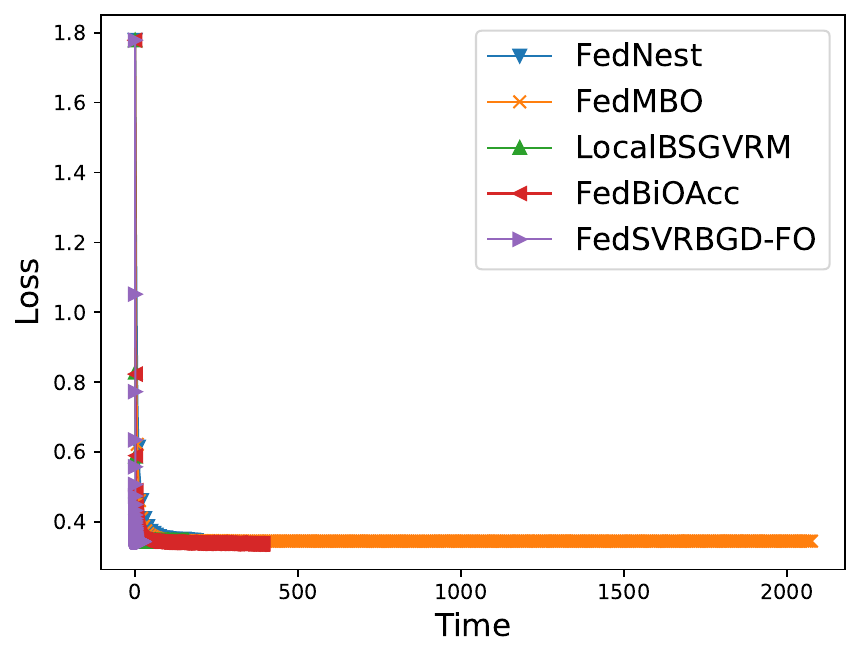}}
  \subfigure[Accuracy]{
  \label{fig:all-datasets-running-time-hyper-para-opt-acc-16}
    \includegraphics[scale=0.25]{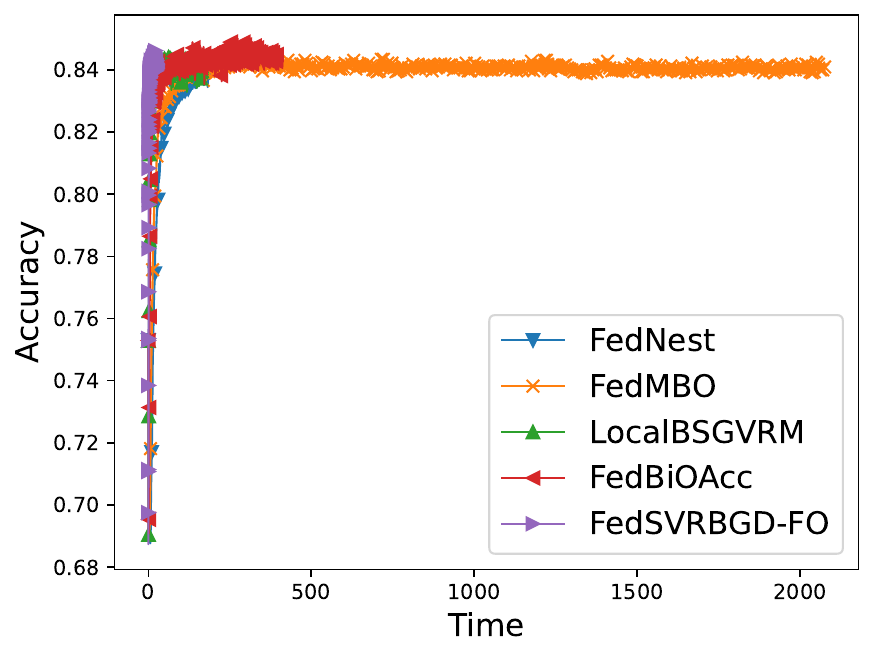}}
    \caption{The upper-level loss function value and the test accuracy versus the running time (seconds) for a9a dataset. Communication period $p=16$.}
  \label{fig:all-datasets-running-time-hyper-para-opt-16}
\end{figure}

\subsection{Hyper-representation Learning}
To further verify the performance of our algorithm, we apply our algorithm to the hyper-representation learning task as \cite{tarzanagh2022fednest}. Specifically, given a deep neural network, the upper-level problem learns the weight of hidden layers, while the lower-level problem learns the weight of the classifier. As such, the upper-level problem is nonconvex and the lower-level one is strongly convex when using the cross-entropy loss function with $\ell_2$-norm regularization for the classifier's weight. In our experiment, we used a fully connected two-layer neural network. Its dimensionality of the input, hidden, and output layer is 54, 30, and 7, respectively. The dataset is covtype, which has seven classes. The batch size is 100. The communication period $p=4$. All the other settings are the same as the first experiment.

In Figure~\ref{fig:hyper-representation-exp}, we plot the loss function value versus the running time (seconds) for the hyper-representation task. It can be seen that our algorithm converges much faster than all baseline methods in terms of running time, which further confirms the advantages of using the first-order gradient over the second-order ones.

\begin{figure}[h]
  \centering
    \includegraphics[scale=0.45]{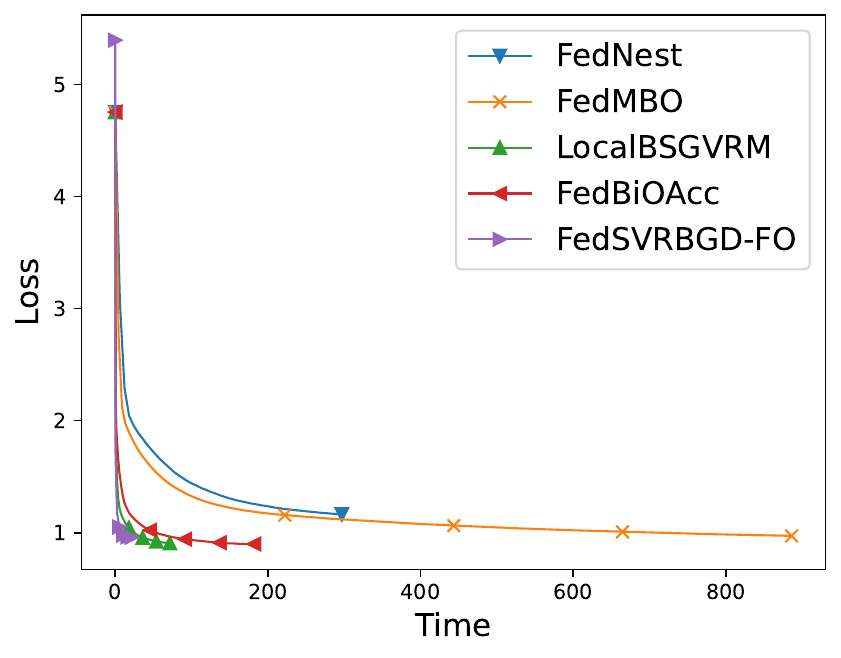}
    \caption{The upper-level loss function value versus the running time (seconds) for the hyper-representation task.}
  \label{fig:hyper-representation-exp}
\end{figure}





\section{Conclusions}
In this paper, we develop a novel federated stochastic bilevel optimization algorithm based on the fully first-order gradient. In particular, each work only needs to compute the first-order stochastic variance-reduced gradient to update the upper- and lower-level variables. This can significantly save running time because our algorithm does not need to compute the second-order Hessian and Jacobian matrices. Moreover, we developed a novel single-timescale constant learning rate to coordinate the update of different variables and presented a novel strategy to establish the convergence rate of our algorithm.  The extensive experimental results confirm the efficacy of our algorithm.

\newpage
\section*{Acknowledgments}
Y. Zhang and H. Gao were partially supported by U.S. NSF CAREER 2339545, NSF IIS 2416607, NSF CNS 2107014. 

\bibliographystyle{named}
\bibliography{ijcai25}

\onecolumn
\newpage
\appendix
\input{supp.tex}

\end{document}

%% file: supp.tex
\section{Appendix}

\subsection{Fundamental Lemmas} \label{sec:lemma}
\begin{lemma}
Given Assumptions~\ref{assumption:f-smooth}-\ref{assumption:heterogeneous}, by setting $\eta  \leq \frac{1}{2\alpha_{x} L}$,  we have
	\begin{align}
		& 	\mathbb{E}[\mathcal{L}_{\lambda}^{*}(\bar{x}_{t+1})] \leq 	\mathbb{E}[\mathcal{L}_{\lambda}^{*}(\bar{x}_{t})]  -\frac{\alpha_{x}\eta}{2} \mathbb{E}[\|\nabla \mathcal{L}_{\lambda}^{*}(\bar{x}_{t})\|^2] - 	\frac{\alpha_{x}\eta}{4} \mathbb{E}[\| \bar{m}_{x, t} \|^2] \notag \\
		& \quad + 	\frac{9\alpha_{x}\eta}{2}  (L_f^2+\lambda^2L_{g,1}^2)\mathbb{E}[\|{y}_{\lambda}^{*}(\bar{x}_{t})-\bar{y}_{t}\|^2] +\frac{9\alpha_{x}\eta}{2}  \lambda^2L_{g,1}^2\mathbb{E}[\|{y}^{*}(\bar{x}_{t})-\bar{z}_{t}\|^2] \notag \\
		& \quad + 	\frac{9\alpha_{x}\eta}{2}(L_f^2+\lambda^2L_{g,1}^2+\lambda^2L_{g,1}^2)\frac{1}{N}\sum_{n=1}^{N} \mathbb{E}[\| x^{(n)}_{t} - \bar{x}_{t}\|^2] + \frac{9\alpha_{x}\eta}{2}(L_f^2+\lambda^2L_{g,1}^2)\frac{1}{N}\sum_{n=1}^{N} \mathbb{E}[\| y^{(n)}_{t} - \bar{y}_{t}\|^2] \notag \\
		& \quad + \frac{9\alpha_{x}\eta}{2}\lambda^2L_{g,1}^2\frac{1}{N}\sum_{n=1}^{N} \mathbb{E}[\| z^{(n)}_{t} - \bar{z}_{t}\|^2]   +   	\frac{9\alpha_{x}\eta}{2}\mathbb{E}[\|\frac{1}{N}\sum_{n=1}^{N}\nabla_{1} f^{(n)}(x^{(n)}_{t}, y^{(n)}_{t})- \frac{1}{N}\sum_{n=1}^{N} u^{(n)}_{1, t} \|^2]  \notag \\
		& \quad  +   	\frac{9\alpha_{x}\eta\lambda^2}{2}\mathbb{E}[\|\frac{1}{N}\sum_{n=1}^{N}\nabla_{1} g^{(n)}(x^{(n)}_{t}, y^{(n)}_{t})- \frac{1}{N}\sum_{n=1}^{N} u^{(n)}_{2, t} \|^2]  \notag \\
		& \quad  +   	\frac{9\alpha_{x}\eta\lambda^2}{2}\mathbb{E}[\|\frac{1}{N}\sum_{n=1}^{N}\nabla_{1} g^{(n)}(x^{(n)}_{t}, z^{(n)}_{t})- \frac{1}{N}\sum_{n=1}^{N} u^{(n)}_{3, t} \|^2]  \ .
	\end{align}
\end{lemma}

\begin{proof}
	\begin{align}
		& 	\mathbb{E}[\mathcal{L}_{\lambda}^{*}(\bar{x}_{t+1})] \leq 		\mathbb{E}[\mathcal{L}_{\lambda}^{*}(\bar{x}_{t})] + 	\mathbb{E}[\langle \nabla \mathcal{L}_{\lambda}^{*}(\bar{x}_{t}), \bar{x}_{t+1} - \bar{x}_{t} \rangle] + \frac{L}{2} 	\mathbb{E}[\|\bar{x}_{t+1} - \bar{x}_{t}\|^2] \notag \\
		& = 	\mathbb{E}[\mathcal{L}_{\lambda}^{*}(\bar{x}_{t})]  -\alpha_{x}\eta	\mathbb{E}[\langle \nabla \mathcal{L}_{\lambda}^{*}(\bar{x}_{t}), \bar{m}_{x, t} \rangle] + \frac{\alpha^2_{x}\eta^2 L}{2} 	\mathbb{E}[\|\bar{m}_{x, t}\|^2] \notag \\
		& = 	\mathbb{E}[\mathcal{L}_{\lambda}^{*}(\bar{x}_{t})]  -\frac{\alpha_{x}\eta}{2} \mathbb{E}[\|\nabla \mathcal{L}_{\lambda}^{*}(\bar{x}_{t})\|^2] - 	\frac{\alpha_{x}\eta}{2} \mathbb{E}[\| \bar{m}_{x, t} \|^2]   + 	\frac{\alpha_{x}\eta}{2} \mathbb{E}[\|\nabla \mathcal{L}_{\lambda}^{*}(\bar{x}_{t}) -  \bar{m}_{x, t} \|^2] + \frac{\alpha^2_{x}\eta^2 L}{2} 	\mathbb{E}[\|\bar{m}_{x, t}\|^2] \notag \\
		& \leq 	\mathbb{E}[\mathcal{L}_{\lambda}^{*}(\bar{x}_{t})]  -\frac{\alpha_{x}\eta}{2} \mathbb{E}[\|\nabla \mathcal{L}_{\lambda}^{*}(\bar{x}_{t})\|^2] - 	\frac{\alpha_{x}\eta}{4} \mathbb{E}[\| \bar{m}_{x, t} \|^2]   + 	\frac{\alpha_{x}\eta}{2} \mathbb{E}[\|\nabla \mathcal{L}_{\lambda}^{*}(\bar{x}_{t}) -  \bar{m}_{x, t} \|^2] \notag \\
		& \leq 	\mathbb{E}[\mathcal{L}_{\lambda}^{*}(\bar{x}_{t})]  -\frac{\alpha_{x}\eta}{2} \mathbb{E}[\|\nabla \mathcal{L}_{\lambda}^{*}(\bar{x}_{t})\|^2] - 	\frac{\alpha_{x}\eta}{4} \mathbb{E}[\| \bar{m}_{x, t} \|^2] \notag \\
		& \quad + 	\frac{3\alpha_{x}\eta}{2} \mathbb{E}[\|\nabla \mathcal{L}_{\lambda}^{*}(\bar{x}_{t}) -(\nabla_{1}  f(\bar{x}_{t}, \bar{y}_{t}) + \lambda(\nabla_{1} g(\bar{x}_{t}, \bar{y}_{t})-\nabla_{1} g(\bar{x}_{t}, \bar{z}_{t}))) \|^2]  \notag \\
		& \quad + 	\frac{3\alpha_{x}\eta}{2}\mathbb{E}[\| (\nabla_{1}  f(\bar{x}_{t}, \bar{y}_{t}) + \lambda(\nabla_{1} g(\bar{x}_{t}, \bar{y}_{t})-\nabla_{1} g(\bar{x}_{t}, \bar{z}_{t}))) \notag \\
		& \quad \quad -\frac{1}{N}\sum_{n=1}^{N} (\nabla_{1}  f^{(n)}(x^{(n)}_{t}, y^{(n)}_{t}) + \lambda(\nabla_{1} g^{(n)}(x^{(n)}_{t}, y^{(n)}_{t})-\nabla_{1} g^{(n)}(x^{(n)}_{t}, z^{(n)}_{t})) )\|^2]  \notag \\
		& \quad  +   	\frac{3\alpha_{x}\eta}{2}\mathbb{E}[\|\frac{1}{N}\sum_{n=1}^{N} (\nabla_{1}  f^{(n)}(x^{(n)}_{t}, y^{(n)}_{t}) + \lambda(\nabla_{1} g^{(n)}(x^{(n)}_{t}, y^{(n)}_{t})-\nabla_{1} g^{(n)}(x^{(n)}_{t}, z^{(n)}_{t})) ) - \frac{1}{N}\sum_{n=1}^{N} {m}^{(n)}_{x, t} \|^2]  \notag \\
		& \leq 	\mathbb{E}[\mathcal{L}_{\lambda}^{*}(\bar{x}_{t})]  -\frac{\alpha_{x}\eta}{2} \mathbb{E}[\|\nabla \mathcal{L}_{\lambda}^{*}(\bar{x}_{t})\|^2] - 	\frac{\alpha_{x}\eta}{4} \mathbb{E}[\| \bar{m}_{x, t} \|^2] \notag \\
		& \quad + 	\frac{9\alpha_{x}\eta}{2}  (L_f^2+\lambda^2L_{g,1}^2)\mathbb{E}[\|{y}_{\lambda}^{*}(\bar{x}_{t})-\bar{y}_{t}\|^2] +\frac{9\alpha_{x}\eta}{2}  \lambda^2L_{g,1}^2\mathbb{E}[\|{y}^{*}(\bar{x}_{t})-\bar{z}_{t}\|^2] \notag \\
		& \quad + 	\frac{9\alpha_{x}\eta}{2}(L_f^2+\lambda^2L_{g,1}^2+\lambda^2L_{g,1}^2)\frac{1}{N}\sum_{n=1}^{N} \mathbb{E}[\| x^{(n)}_{t} - \bar{x}_{t}\|^2] + \frac{9\alpha_{x}\eta}{2}(L_f^2+\lambda^2L_{g,1}^2)\frac{1}{N}\sum_{n=1}^{N} \mathbb{E}[\| y^{(n)}_{t} - \bar{y}_{t}\|^2] \notag \\
		& \quad + \frac{9\alpha_{x}\eta}{2}\lambda^2L_{g,1}^2\frac{1}{N}\sum_{n=1}^{N} \mathbb{E}[\| z^{(n)}_{t} - \bar{z}_{t}\|^2]   +   	\frac{9\alpha_{x}\eta}{2}\mathbb{E}[\|\frac{1}{N}\sum_{n=1}^{N}\nabla_{1} f^{(n)}(x^{(n)}_{t}, y^{(n)}_{t})- \frac{1}{N}\sum_{n=1}^{N} u^{(n)}_{1, t} \|^2]  \notag \\
		& \quad  +   	\frac{9\alpha_{x}\eta\lambda^2}{2}\mathbb{E}[\|\frac{1}{N}\sum_{n=1}^{N}\nabla_{1} g^{(n)}(x^{(n)}_{t}, y^{(n)}_{t})- \frac{1}{N}\sum_{n=1}^{N} u^{(n)}_{2, t} \|^2]  \notag \\
		& \quad  +   	\frac{9\alpha_{x}\eta\lambda^2}{2}\mathbb{E}[\|\frac{1}{N}\sum_{n=1}^{N}\nabla_{1} g^{(n)}(x^{(n)}_{t}, z^{(n)}_{t})- \frac{1}{N}\sum_{n=1}^{N} u^{(n)}_{3, t} \|^2]   \ , 
	\end{align}
    where the fourth step follows from $\eta  \leq \frac{1}{2\alpha_{x} L}$, the last step follows from the following two inequalities:
	\begin{align}
		&\quad   \mathbb{E}[\|\nabla \mathcal{L}_{\lambda}^{*}(\bar{x}_{t}) - (\nabla_{1}  f(\bar{x}_{t}, \bar{y}_{t}) + \lambda(\nabla_{1} g(\bar{x}_{t}, \bar{y}_{t})-\nabla_{1} g(\bar{x}_{t}, \bar{z}_{t}))) \|^2]  \notag \\
		& = \mathbb{E}[\|\nabla_{1}  f(\bar{x}_{t}, {y}_{\lambda}^{*}(\bar{x}_{t})) + \lambda(\nabla_{1} g(\bar{x}_{t}, {y}_{\lambda}^{*}(\bar{x}_{t}))-\nabla_{1} g(\bar{x}_{t}, {y}^{*}(\bar{x}_{t}))) - (\nabla_{1}  f(\bar{x}_{t}, \bar{y}_{t}) + \lambda(\nabla_{1} g(\bar{x}_{t}, \bar{y}_{t})-\nabla_{1} g(\bar{x}_{t}, \bar{z}_{t})))\|^2]  \notag \\
		& \leq 3 (L_f^2+\lambda^2L_{g,1}^2)\mathbb{E}[\|{y}_{\lambda}^{*}(\bar{x}_{t})-\bar{y}_{t}\|^2] + 3\lambda^2L_{g,1}^2\mathbb{E}[\|{y}^{*}(\bar{x}_{t})-\bar{z}_{t}\|^2] \ , 
	\end{align}
    and 
	\begin{align}
		& \quad \mathbb{E}[\| \nabla \mathcal{L}_{\lambda}(\bar{x}_{t}, \bar{y}_{t}, \bar{z}_{t})   -\frac{1}{N}\sum_{n=1}^{N}\nabla \mathcal{L}^{(n)}_{\lambda}({x}^{(n)}_{t}, {y}^{(n)}_{t}, {z}^{(n)}_{t})   \|^2]  \notag \\
		& =\frac{1}{N}\sum_{n=1}^{N} \mathbb{E}[\|\nabla_{1}  f^{(n)}(\bar{x}_{t}, \bar{y}_{t}) + \lambda(\nabla_{1} g^{(n)}(\bar{x}_{t}, \bar{y}_{t})-\nabla_{1} g^{(n)}(\bar{x}_{t}, \bar{z}_{t})) \notag \\
		& \quad - (\nabla_{1}  f^{(n)}(x^{(n)}_{t}, y^{(n)}_{t}) + \lambda(\nabla_{1} g^{(n)}(x^{(n)}_{t}, y^{(n)}_{t})-\nabla_{1} g^{(n)}(x^{(n)}_{t}, z^{(n)}_{t})) )   \|^2]  \notag \\
		& \leq 3(L_f^2+\lambda^2L_{g,1}^2+\lambda^2L_{g,1}^2)\frac{1}{N}\sum_{n=1}^{N} \mathbb{E}[\| x^{(n)}_{t} - \bar{x}_{t}\|^2] + 3(L_f^2+\lambda^2L_{g,1}^2)\frac{1}{N}\sum_{n=1}^{N} \mathbb{E}[\| y^{(n)}_{t} - \bar{y}_{t}\|^2] \notag \\
		& \quad + 3\lambda^2L_{g,1}^2\frac{1}{N}\sum_{n=1}^{N} \mathbb{E}[\| z^{(n)}_{t} - \bar{z}_{t}\|^2] \ . 
	\end{align}

\end{proof}

\begin{lemma} \label{lemma_y_opt}
	Given Assumptions~\ref{assumption:f-smooth}-\ref{assumption:heterogeneous},  by setting $\alpha_y\leq \frac{1}{6(L_f+\lambda L_{g, 1})}$, we have
	\begin{align}
		&\quad  \mathbb{E}[\|\bar{y}_{t+1} - y_{\lambda}^{*}(\bar{{x}}_{t+1})\|^2]   \leq (1-\frac{\eta\alpha_y\mu\lambda}{8}) \mathbb{E}[\|\bar{y}_{t} - y_{\lambda}^{*}(\bar{{x}}_{t})\|^2] - \frac{3\eta\alpha^2_y}{4} \mathbb{E}[\|\bar{m}_{y, t}\|^2] + \frac{100\eta\alpha^2_{x} (L_f+\lambda L_{g, 1})^2}{3\alpha_y\lambda^3\mu^3}\mathbb{E}[\|\bar{m}_{x, t}\|^2] \notag \\
		& \quad  + \frac{100\eta\alpha_y}{3\lambda\mu} (L_f^2 + \lambda^2L_{g, 1}^2)\frac{1}{N}\sum_{n=1}^{N}\mathbb{E}[\|\bar{x}_{t}-{x}_{t}^{(n)}\|^2] + \frac{100\eta\alpha_y}{3\lambda\mu} (L_f^2 + \lambda^2L_{g, 1}^2)\frac{1}{N}\sum_{n=1}^{N}\mathbb{E}[\|\bar{y}_{t}-{y}_{t}^{(n)}\|^2] \notag \\
		& \quad  + \frac{100\eta\alpha_y}{3\lambda\mu} \mathbb{E}[\|\frac{1}{N}\sum_{n=1}^{N} \nabla_{2} f^{(n)}({x}_{t}^{(n)}, {y}_{t}^{(n)}) -\frac{1}{N}\sum_{n=1}^{N} v^{(n)}_{1, t}  \|^2] \notag \\
		& \quad  + \frac{100\eta\alpha_y}{3\lambda\mu} \lambda^2\mathbb{E}[\|\frac{1}{N}\sum_{n=1}^{N}  \nabla_{2} g^{(n)}({x}_{t}^{(n)}, {y}_{t}^{(n)})-\frac{1}{N}\sum_{n=1}^{N}     v^{(n)}_{2, t} \|^2] \ . 
	\end{align}

\end{lemma}

	
	

\begin{proof}
    According to our algorithm, it is easy to know that the updating of $y$ is to solve the following sub-problem:
    \begin{align}
        & \min_{y\in\mathbb{R}^{d_y}} \mathcal{L}_{\lambda}(x, y) =  f(x, y) + \lambda g(x, y) \ . 
    \end{align}
    In terms of Assumptions~\ref{assumption:f-smooth}-\ref{assumption:heterogeneous}, $\mathcal{L}_{\lambda}(x, y)$ is $\mu_{y}\triangleq\frac{\lambda\mu}{2}$ strongly-convex in $y$ when $\lambda>\frac{2L_f}{\mu}$, and it is $L_{y}\triangleq(L_f+\lambda L_{g, 1})$-smooth in $y$. 
	Then, we can get
	\begin{align}
		&\quad  \mathbb{E}[\|\bar{y}_{t+1} - y_{\lambda}^{*}(\bar{{x}}_{t+1})\|^2]\notag \\
		& \leq (1-\frac{\eta\alpha_y\mu_{y}}{4}) \mathbb{E}[\|\bar{y}_{t} - y_{\lambda}^{*}(\bar{{x}}_{t})\|^2] - \frac{3\eta\alpha^2_y}{4} \mathbb{E}[\|\bar{m}_{y, t}\|^2] + \frac{25\eta\alpha^2_{x} L^2_{y}}{6\alpha_y\mu_{y}^3}\mathbb{E}[\|\bar{m}_{x, t}\|^2]  + \frac{25\eta\alpha_y}{6\mu_{y}} \mathbb{E}[\|\nabla_{2} \mathcal{L}_{\lambda}(\bar{x}_{t}, \bar{y}_{t}) -\bar{m}_{y, t} \|^2] \notag \\
		& \leq (1-\frac{\eta\alpha_y\mu_{y}}{4}) \mathbb{E}[\|\bar{y}_{t} - y_{\lambda}^{*}(\bar{{x}}_{t})\|^2] - \frac{3\eta\alpha^2_y}{4} \mathbb{E}[\|\bar{m}_{y, t}\|^2] + \frac{25\eta\alpha^2_{x} L^2_{y}}{6\alpha_y\mu_{y}^3}\mathbb{E}[\|\bar{m}_{x, t}\|^2] \notag \\
		& \quad  + \frac{25\eta\alpha_y}{3\mu_{y}} \mathbb{E}[\| \nabla_{2} f(\bar{x}_{t}, \bar{y}_{t}) + \lambda \nabla_{2} g(\bar{x}_{t}, \bar{y}_{t}) -\frac{1}{N}\sum_{n=1}^{N} (\nabla_{2} f^{(n)}({x}_{t}^{(n)}, {y}_{t}^{(n)}) + \lambda \nabla_{2} g^{(n)}({x}_{t}^{(n)}, {y}_{t}^{(n)}))\|^2] \notag \\
		& \quad  + \frac{25\eta\alpha_y}{3\mu_{y}} \mathbb{E}[\|\frac{1}{N}\sum_{n=1}^{N} (\nabla_{2} f^{(n)}({x}_{t}^{(n)}, {y}_{t}^{(n)}) + \lambda \nabla_{2} g^{(n)}({x}_{t}^{(n)}, {y}_{t}^{(n)}))-\frac{1}{N}\sum_{n=1}^{N} {m}^{(n)}_{y, t} \|^2] \notag \\
		& \leq (1-\frac{\eta\alpha_y\mu_{y}}{4}) \mathbb{E}[\|\bar{y}_{t} - y_{\lambda}^{*}(\bar{{x}}_{t})\|^2] - \frac{3\eta\alpha^2_y}{4} \mathbb{E}[\|\bar{m}_{y, t}\|^2] + \frac{25\eta\alpha^2_{x} L^2_{y}}{6\alpha_y\mu_{y}^3}\mathbb{E}[\|\bar{m}_{x, t}\|^2] \notag \\
		& \quad  + \frac{50\eta\alpha_y}{3\mu_{y}} \mathbb{E}[\| \nabla_{2} f(\bar{x}_{t}, \bar{y}_{t}) -\frac{1}{N}\sum_{n=1}^{N} \nabla_{2} f^{(n)}({x}_{t}^{(n)}, {y}_{t}^{(n)}) \|^2] \notag \\
		& \quad  + \frac{50\eta\alpha_y}{3\mu_{y}}\lambda^2 \mathbb{E}[\|   \nabla_{2} g(\bar{x}_{t}, \bar{y}_{t}) -\frac{1}{N}\sum_{n=1}^{N}  \nabla_{2} g^{(n)}({x}_{t}^{(n)}, {y}_{t}^{(n)})\|^2] \notag \\
		& \quad  + \frac{50\eta\alpha_y}{3\mu_{y}} \mathbb{E}[\|\frac{1}{N}\sum_{n=1}^{N} \nabla_{2} f^{(n)}({x}_{t}^{(n)}, {y}_{t}^{(n)}) -\frac{1}{N}\sum_{n=1}^{N} v^{(n)}_{1, t}  \|^2] \notag \\
		& \quad  + \frac{50\eta\alpha_y}{3\mu_{y}} \lambda^2\mathbb{E}[\|\frac{1}{N}\sum_{n=1}^{N}  \nabla_{2} g^{(n)}({x}_{t}^{(n)}, {y}_{t}^{(n)})-\frac{1}{N}\sum_{n=1}^{N}     v^{(n)}_{2, t} \|^2] \notag \\
		& \leq (1-\frac{\eta\alpha_y\mu_{y}}{4}) \mathbb{E}[\|\bar{y}_{t} - y_{\lambda}^{*}(\bar{{x}}_{t})\|^2] - \frac{3\eta\alpha^2_y}{4} \mathbb{E}[\|\bar{m}_{y, t}\|^2] + \frac{25\eta\alpha^2_{x} L^2_{y}}{6\alpha_y\mu_{y}^3}\mathbb{E}[\|\bar{m}_{x, t}\|^2] \notag \\
		& \quad  + \frac{50\eta\alpha_y}{3\mu_{y}} (L_f^2 + \lambda^2L_{g, 1}^2)\frac{1}{N}\sum_{n=1}^{N}\mathbb{E}[\|\bar{x}_{t}-{x}_{t}^{(n)}\|^2] + \frac{50\eta\alpha_y}{3\mu_{y}} (L_f^2 + \lambda^2L_{g, 1}^2)\frac{1}{N}\sum_{n=1}^{N}\mathbb{E}[\|\bar{y}_{t}-{y}_{t}^{(n)}\|^2] \notag \\
		& \quad  + \frac{50\eta\alpha_y}{3\mu_{y}} \mathbb{E}[\|\frac{1}{N}\sum_{n=1}^{N} \nabla_{2} f^{(n)}({x}_{t}^{(n)}, {y}_{t}^{(n)}) -\frac{1}{N}\sum_{n=1}^{N} v^{(n)}_{1, t}  \|^2] \notag \\
		& \quad  + \frac{50\eta\alpha_y}{3\mu_{y}} \lambda^2\mathbb{E}[\|\frac{1}{N}\sum_{n=1}^{N}  \nabla_{2} g^{(n)}({x}_{t}^{(n)}, {y}_{t}^{(n)})-\frac{1}{N}\sum_{n=1}^{N}     v^{(n)}_{2, t} \|^2] \ , 
	\end{align}
    where the first step follows from  \cite{huang2020accelerated} by setting $\alpha_y\leq \frac{1}{6L_{y}}$.

\end{proof}

\begin{lemma}
Given Assumptions~\ref{assumption:f-smooth}-\ref{assumption:heterogeneous},  by setting $\alpha_z\leq \frac{1}{6\lambda L_{g,1}}$,  we have
	\begin{align}
		& \quad \mathbb{E}[\|\bar{z}_{t+1} - y^{*}(\bar{{x}}_{t+1})\|^2]  \leq (1-\frac{\eta\alpha_z\lambda\mu}{4}) \mathbb{E}[\|\bar{z}_{t} - y^{*}(\bar{{x}}_{t})\|^2] - \frac{3\eta\alpha^2_z}{4} \mathbb{E}[\|\bar{m}_{z, t}\|^2] + \frac{25\eta\alpha^2_{x}  L^2_{g,1}}{6\alpha_z\lambda \mu^3}\mathbb{E}[\|\bar{m}_{x, t}\|^2]  \notag \\
		& \quad + \frac{25\eta\alpha_z}{3\mu} \lambda L_{g, 1}^2\frac{1}{N}\sum_{n=1}^{N}\mathbb{E}[\| \bar{x}_{t}-{x}^{(n)}_{t} \|^2] + \frac{25\eta\alpha_z}{3\mu} \lambda L_{g, 1}^2\frac{1}{N}\sum_{n=1}^{N}\mathbb{E}[\| \bar{z}_{t}-{z}^{(n)}_{t} \|^2] \notag \\
		& \quad + \frac{25\eta\alpha_z}{3\mu} \lambda \mathbb{E}[\| \frac{1}{N}\sum_{n=1}^{N} \nabla_{2} g^{(n)}({x}^{(n)}_{t}, {z}^{(n)}_{t}) - \frac{1}{N}\sum_{n=1}^{N}{w}^{(n)}_{1, t} \|^2] \ .
	\end{align}
\end{lemma}

\begin{proof}
According to our algorithm, it is easy to know that the updating of $z$ is to solve the  sub-problem: $\min_{z\in\mathbb{R}^{d_y}} g_{\lambda}(x, z)\triangleq \lambda g(x, z)$. In terms of Assumptions~\ref{assumption:f-smooth}-\ref{assumption:heterogeneous}, $g_{\lambda}(x, z)$ is $\mu_{z} \triangleq \lambda\mu$ strongly-convex in $z$ and  $L_{z}\triangleq \lambda L_{g,1}$-smooth in $y$. Then, we can get
	\begin{align}
		& \quad \mathbb{E}[\|\bar{z}_{t+1} - y^{*}(\bar{{x}}_{t+1})\|^2]\notag \\
		& \leq (1-\frac{\eta\alpha_z\mu_{z}}{4}) \mathbb{E}[\|\bar{z}_{t} - y^{*}(\bar{{x}}_{t})\|^2] - \frac{3\eta\alpha^2_z}{4} \mathbb{E}[\|\bar{m}_{z, t}\|^2] + \frac{25\eta\alpha^2_{x} L^2_{z}}{6\alpha_z\mu_{z}^3}\mathbb{E}[\|\bar{m}_{x, t}\|^2]  + \frac{25\eta\alpha_z}{6\mu_{z}} \mathbb{E}[\|\lambda \nabla_{2} g(\bar{x}_{t}, \bar{z}_{t}) -\bar{m}_{z, t} \|^2] \notag \\
		& \leq (1-\frac{\eta\alpha_z\mu_{z}}{4}) \mathbb{E}[\|\bar{z}_{t} - y^{*}(\bar{{x}}_{t})\|^2] - \frac{3\eta\alpha^2_z}{4} \mathbb{E}[\|\bar{m}_{z, t}\|^2] + \frac{25\eta\alpha^2_{x} L^2_{z}}{6\alpha_z\mu_{z}^3}\mathbb{E}[\|\bar{m}_{x, t}\|^2]  \notag \\
		& \quad + \frac{25\eta\alpha_z}{6\mu_{z}} \lambda^2\mathbb{E}[\| \nabla_{2} g(\bar{x}_{t}, \bar{z}_{t}) - \bar{w}_{1, t} \|^2] \notag \\
		& \leq (1-\frac{\eta\alpha_z\mu_{z}}{4}) \mathbb{E}[\|\bar{z}_{t} - y^{*}(\bar{{x}}_{t})\|^2] - \frac{3\eta\alpha^2_z}{4} \mathbb{E}[\|\bar{m}_{z, t}\|^2] + \frac{25\eta\alpha^2_{x} L^2_{z}}{6\alpha_z\mu_{z}^3}\mathbb{E}[\|\bar{m}_{x, t}\|^2]  \notag \\
		& \quad + \frac{25\eta\alpha_z}{3\mu_{z}} \lambda^2\mathbb{E}[\| \nabla_{2} g(\bar{x}_{t}, \bar{z}_{t}) - \frac{1}{N}\sum_{n=1}^{N} \nabla_{2} g^{(n)}({x}^{(n)}_{t}, {z}^{(n)}_{t})  \|^2] \notag \\
		& \quad + \frac{25\eta\alpha_z}{3\mu_{z}} \lambda^2\mathbb{E}[\| \frac{1}{N}\sum_{n=1}^{N} \nabla_{2} g^{(n)}({x}^{(n)}_{t}, {z}^{(n)}_{t}) - \frac{1}{N}\sum_{n=1}^{N}{w}^{(n)}_{1, t} \|^2] \notag \\
		& \leq (1-\frac{\eta\alpha_z\mu_{z}}{4}) \mathbb{E}[\|\bar{z}_{t} - y^{*}(\bar{{x}}_{t})\|^2] - \frac{3\eta\alpha^2_z}{4} \mathbb{E}[\|\bar{m}_{z, t}\|^2] + \frac{25\eta\alpha^2_{x} L^2_{z}}{6\alpha_z\mu_{z}^3}\mathbb{E}[\|\bar{m}_{x, t}\|^2]  \notag \\
		& \quad + \frac{25\eta\alpha_z}{3\mu_{z}} \lambda^2L_{g, 1}^2\frac{1}{N}\sum_{n=1}^{N}\mathbb{E}[\| \bar{x}_{t}-{x}^{(n)}_{t} \|^2] + \frac{25\eta\alpha_z}{3\mu_{\lambda}} \lambda^2L_{g, 1}^2\frac{1}{N}\sum_{n=1}^{N}\mathbb{E}[\| \bar{z}_{t}-{z}^{(n)}_{t} \|^2] \notag \\
		& \quad + \frac{25\eta\alpha_z}{3\mu_{z}} \lambda^2\mathbb{E}[\| \frac{1}{N}\sum_{n=1}^{N} \nabla_{2} g^{(n)}({x}^{(n)}_{t}, {z}^{(n)}_{t}) - \frac{1}{N}\sum_{n=1}^{N}{w}^{(n)}_{1, t} \|^2] \ , 
	\end{align}
     where the first step follows from \cite{huang2020accelerated} by setting $\alpha_z\leq \frac{1}{6L_{z}}$. 
     
\end{proof}

\begin{lemma}\label{lemma:incremental-x-y-z}
    Given Assumptions~\ref{assumption:f-smooth}-\ref{assumption:heterogeneous}, we have the following inequalities:
    \begin{align} \label{eq:update-x}
        &  \frac{1}{N}\sum_{n=1}^{N}\mathbb{E}[\| x^{(n)}_{t+1} - x^{(n)}_{t} \|^2] \leq 6\alpha_x^2\eta^2 \frac{1}{N}\sum_{n=1}^{N}\mathbb{E}[\| u^{(n)}_{1, t} - \bar{u}_{1, t} \|^2] + 6\alpha_x^2\eta^2\lambda^2 \frac{1}{N}\sum_{n=1}^{N}\mathbb{E}[\| u^{(n)}_{2, t} - \bar{u}_{2, t} \|^2]  \notag \\
        & \quad + 6\alpha_x^2\eta^2\lambda^2 \frac{1}{N}\sum_{n=1}^{N}\mathbb{E}[\| u^{(n)}_{3, t} - \bar{u}_{3, t} \|^2] + 6\alpha_x^2\eta^2 \mathbb{E}[\| \bar{m}_{x, t} \|^2]\ , 
    \end{align}
    \begin{align}\label{eq:update-y}
        &  \frac{1}{N}\sum_{n=1}^{N}\mathbb{E}[\| y^{(n)}_{t+1} - y^{(n)}_{t} \|^2] \leq 4\alpha_y^2\eta^2 \frac{1}{N}\sum_{n=1}^{N}\mathbb{E}[\| v^{(n)}_{1, t} - \bar{v}_{1, t} \|^2] + 4\alpha_y^2\eta^2\lambda^2 \frac{1}{N}\sum_{n=1}^{N}\mathbb{E}[\| v^{(n)}_{2, t} - \bar{v}_{2, t} \|^2] \notag \\
        & \quad + 4\alpha_y^2\eta^2  \mathbb{E}[\| \bar{m}_{y, t} \|^2] \ , 
    \end{align}
    \begin{align}\label{eq:update-z}
        &  \frac{1}{N}\sum_{n=1}^{N}\mathbb{E}[\| z^{(n)}_{t+1} - z^{(n)}_{t} \|^2]  \leq  2\lambda^2\alpha_z^2\eta^2 \frac{1}{N}\sum_{n=1}^{N}\mathbb{E}[\|    w^{(n)}_{1, t} -  \bar{w}_{1, t}   \|^2] + 2\alpha_z^2\eta^2 \mathbb{E}[\|   \bar{m}_{z, t}  \|^2] \ .
    \end{align}
\end{lemma}

\begin{proof}
    Due to ${x}^{(n)}_{t+1} = x^{(n)}_{t} -\alpha_x \eta m^{(n)}_{x, t}$, we have
    \begin{align} 
        & \quad \frac{1}{N}\sum_{n=1}^{N}\mathbb{E}[\| x^{(n)}_{t+1} - x^{(n)}_{t} \|^2] \notag \\
        & = \frac{1}{N}\sum_{n=1}^{N}\mathbb{E}[\| \alpha_x \eta m^{(n)}_{x, t} \|^2] \notag \\
        & \leq 2\alpha_x^2\eta^2 \frac{1}{N}\sum_{n=1}^{N}\mathbb{E}[\| m^{(n)}_{x, t} - \bar{m}_{x, t} \|^2] + 2\alpha_x^2\eta^2 \mathbb{E}[\| \bar{m}_{x, t} \|^2] \notag \\
        & = 2\alpha_x^2\eta^2 \frac{1}{N}\sum_{n=1}^{N}\mathbb{E}[\| u^{(n)}_{1, t}  +\lambda (u^{(n)}_{2, t} - u^{(n)}_{3, t}) - \bar{u}_{1, t}  - \lambda (\bar{u}_{2, t} - \bar{u}_{3, t}) \|^2]  + 2\alpha_x^2\eta^2 \mathbb{E}[\| \bar{m}_{x, t} \|^2] \notag \\
        & \leq 6\alpha_x^2\eta^2 \frac{1}{N}\sum_{n=1}^{N}\mathbb{E}[\| u^{(n)}_{1, t} - \bar{u}_{1, t} \|^2] + 6\alpha_x^2\eta^2\lambda^2 \frac{1}{N}\sum_{n=1}^{N}\mathbb{E}[\| u^{(n)}_{2, t} - \bar{u}_{2, t} \|^2] + 6\alpha_x^2\eta^2\lambda^2 \frac{1}{N}\sum_{n=1}^{N}\mathbb{E}[\| u^{(n)}_{3, t} - \bar{u}_{3, t} \|^2] \notag \\
        & \quad + 6\alpha_x^2\eta^2 \mathbb{E}[\| \bar{m}_{x, t} \|^2] \ . 
    \end{align}
    Similarly, due to $y^{(n)}_{t+1} = y^{(n)}_{t} - \alpha_y\eta m^{(n)}_{y, t}$,  we have
    \begin{align}
        & \quad \frac{1}{N}\sum_{n=1}^{N}\mathbb{E}[\| y^{(n)}_{t+1} - y^{(n)}_{t} \|^2] \notag \\
        & = \frac{1}{N}\sum_{n=1}^{N}\mathbb{E}[\|  \alpha_y\eta m^{(n)}_{y, t}  \|^2] \notag \\
        & \leq 2\alpha_y^2\eta^2 \frac{1}{N}\sum_{n=1}^{N}\mathbb{E}[\| m^{(n)}_{y, t} - \bar{m}_{y, t} \|^2] + 2\alpha_y^2\eta^2 \mathbb{E}[\| \bar{m}_{y, t} \|^2] \notag \\
        & = 2\alpha_y^2\eta^2 \frac{1}{N}\sum_{n=1}^{N}\mathbb{E}[\| v^{(n)}_{1, t}  + \lambda  v^{(n)}_{2, t} - \bar{v}_{1, t}  - \lambda  \bar{v}_{2, t} \|^2] + 2\alpha_y^2\eta^2  \mathbb{E}[\| \bar{m}_{y, t} \|^2] \notag \\
        & \leq 4\alpha_y^2\eta^2 \frac{1}{N}\sum_{n=1}^{N}\mathbb{E}[\| v^{(n)}_{1, t} - \bar{v}_{1, t} \|^2] + 4\alpha_y^2\eta^2\lambda^2 \frac{1}{N}\sum_{n=1}^{N}\mathbb{E}[\| v^{(n)}_{2, t} - \bar{v}_{2, t} \|^2] + 4\alpha_y^2\eta^2  \mathbb{E}[\| \bar{m}_{y, t} \|^2] \ . 
    \end{align}
    Moreover, due to $z^{(n)}_{t+1} = z^{(n)}_{t} - \alpha_z\eta m^{(n)}_{z, t}$,  we have
    \begin{align}
        & \quad \frac{1}{N}\sum_{n=1}^{N}\mathbb{E}[\| z^{(n)}_{t+1} - z^{(n)}_{t} \|^2] \notag \\
        & = \frac{1}{N}\sum_{n=1}^{N}\mathbb{E}[\|  \alpha_z\eta m^{(n)}_{z, t}  \|^2] \notag \\
        & \leq 2\alpha_z^2\eta^2 \frac{1}{N}\sum_{n=1}^{N}\mathbb{E}[\| m^{(n)}_{z, t} - \bar{m}_{z, t} \|^2] + 2\alpha_z^2\eta^2 \mathbb{E}[\| \bar{m}_{z, t} \|^2] \notag \\
        & = 2\lambda^2\alpha_z^2\eta^2 \frac{1}{N}\sum_{n=1}^{N}\mathbb{E}[\|    w^{(n)}_{1, t} -  \bar{w}_{1, t}   \|^2] + 2\alpha_z^2\eta^2 \mathbb{E}[\|   \bar{m}_{z, t}  \|^2] \ .
    \end{align}
\end{proof}

\begin{lemma}
Given Assumptions~\ref{assumption:f-smooth}-\ref{assumption:heterogeneous}, by setting $\eta\leq \min\{\frac{1}{\sqrt{\beta_x}}, \frac{1}{\sqrt{\beta_y}}, \frac{1}{\sqrt{\beta_z}}\}$,  we have the following inequalities:
	\begin{align}
	&  \mathbb{E}[\|\frac{1}{N}\sum_{n=1}^{N}\nabla_{1} f^{(n)}(x^{(n)}_{t+1}, y^{(n)}_{t+1})- \frac{1}{N}\sum_{n=1}^{N} u^{(n)}_{1, t+1} \|^2]   \leq (1-\beta_x\eta^2) \mathbb{E}[\| \frac{1}{N}\sum_{n=1}^{N}u^{(n)}_{1, t} - \nabla_1 f^{(n)}(x^{(n)}_{t}, y^{(n)}_{t}) \|^2] + 2\beta_x^2\eta^4\sigma^2 \frac{1}{N}\notag \\
	& \quad +  \frac{12\alpha_x^2\eta^2L_{f}^{2}}{N} \frac{1}{N}\sum_{n=1}^{N}\mathbb{E}[\| u^{(n)}_{1, t} - \bar{u}_{1, t} \|^2] + \frac{12\alpha_x^2\eta^2\lambda^2 L_{f}^{2}}{N}\frac{1}{N}\sum_{n=1}^{N}\mathbb{E}[\| u^{(n)}_{2, t} - \bar{u}_{2, t} \|^2] + \frac{12\alpha_x^2\eta^2\lambda^2 L_{f}^{2}}{N}\frac{1}{N}\sum_{n=1}^{N}\mathbb{E}[\| u^{(n)}_{3, t} - \bar{u}_{3, t} \|^2] \notag \\
	& \quad +  \frac{8\alpha_y^2\eta^2 L_{f}^{2}}{N}\frac{1}{N}\sum_{n=1}^{N}\mathbb{E}[\| v^{(n)}_{1, t} - \bar{v}_{1, t} \|^2] + \frac{8\alpha_y^2\eta^2\lambda^2L_{f}^{2}}{N} \frac{1}{N}\sum_{n=1}^{N}\mathbb{E}[\| v^{(n)}_{2, t} - \bar{v}_{2, t} \|^2] \notag \\
	& \quad + \frac{12\alpha_x^2\eta^2L_{f}^{2}}{N}  \mathbb{E}[\| \bar{m}_{x, t} \|^2] +\frac{8\alpha_y^2\eta^2L_{f}^{2}}{N}   \mathbb{E}[\| \bar{m}_{y, t} \|^2] \  ,
\end{align}
\begin{align}
	& \mathbb{E}[\|\frac{1}{N}\sum_{n=1}^{N}\nabla_{1} g^{(n)}(x^{(n)}_{t+1}, y^{(n)}_{t+1})- \frac{1}{N}\sum_{n=1}^{N} u^{(n)}_{2, t+1} \|^2]  \leq (1-\beta_x\eta^2) \mathbb{E}[\| \frac{1}{N}\sum_{n=1}^{N}\nabla_{1} g^{(n)}(x^{(n)}_{t}, y^{(n)}_{t})- \frac{1}{N}\sum_{n=1}^{N} u^{(n)}_{2, t}\|^2]  + 2\beta_x^2\eta^4\sigma^2 \frac{1}{N} \notag \\
	& \quad +  \frac{12\alpha_x^2\eta^2L_{g, 1}^{2}}{N} \frac{1}{N}\sum_{n=1}^{N}\mathbb{E}[\| u^{(n)}_{1, t} - \bar{u}_{1, t} \|^2] + \frac{12\alpha_x^2\eta^2\lambda^2 L_{g, 1}^{2}}{N}\frac{1}{N}\sum_{n=1}^{N}\mathbb{E}[\| u^{(n)}_{2, t} - \bar{u}_{2, t} \|^2] + \frac{12\alpha_x^2\eta^2\lambda^2 L_{g, 1}^{2}}{N}\frac{1}{N}\sum_{n=1}^{N}\mathbb{E}[\| u^{(n)}_{3, t} - \bar{u}_{3, t} \|^2] \notag \\
	& \quad +  \frac{8\alpha_y^2\eta^2 L_{g, 1}^{2}}{N}\frac{1}{N}\sum_{n=1}^{N}\mathbb{E}[\| v^{(n)}_{1, t} - \bar{v}_{1, t} \|^2] + \frac{8\alpha_y^2\eta^2\lambda^2L_{g, 1}^{2}}{N} \frac{1}{N}\sum_{n=1}^{N}\mathbb{E}[\| v^{(n)}_{2, t} - \bar{v}_{2, t} \|^2] \notag \\
	& \quad + \frac{12\alpha_x^2\eta^2L_{g, 1}^{2}}{N}  \mathbb{E}[\| \bar{m}_{x, t} \|^2]  +\frac{8\alpha_y^2\eta^2L_{g, 1}^{2}}{N}   \mathbb{E}[\| \bar{m}_{y, t} \|^2]\ ,
\end{align}
\begin{align}
	& \mathbb{E}[\|\frac{1}{N}\sum_{n=1}^{N}\nabla_{1} g^{(n)}(x^{(n)}_{t+1}, z^{(n)}_{t+1})- \frac{1}{N}\sum_{n=1}^{N} u^{(n)}_{3, t+1} \|^2] \leq (1-\beta_x\eta^2) \mathbb{E}[\| \frac{1}{N}\sum_{n=1}^{N}\nabla_{1} g^{(n)}(x^{(n)}_{t}, z^{(n)}_{t})- \frac{1}{N}\sum_{n=1}^{N} u^{(n)}_{3, t}\|^2]  + 2\beta_x^2\eta^4\sigma^2 \frac{1}{N} \notag \\
	& \quad +  \frac{12\alpha_x^2\eta^2L_{g, 1}^{2}}{N} \frac{1}{N}\sum_{n=1}^{N}\mathbb{E}[\| u^{(n)}_{1, t} - \bar{u}_{1, t} \|^2] + \frac{12\alpha_x^2\eta^2\lambda^2 L_{g, 1}^{2}}{N}\frac{1}{N}\sum_{n=1}^{N}\mathbb{E}[\| u^{(n)}_{2, t} - \bar{u}_{2, t} \|^2] + \frac{12\alpha_x^2\eta^2\lambda^2 L_{g, 1}^{2}}{N}\frac{1}{N}\sum_{n=1}^{N}\mathbb{E}[\| u^{(n)}_{3, t} - \bar{u}_{3, t} \|^2] \notag \\
	& \quad + \frac{4\lambda^2\alpha_z^2\eta^2L_{g, 1}^2}{N} \frac{1}{N}\sum_{n=1}^{N}\mathbb{E}[\|    w^{(n)}_{1, t} -  \bar{w}_{1, t}   \|^2] + \frac{12\alpha_x^2\eta^2L_{g, 1}^{2}}{N}  \mathbb{E}[\| \bar{m}_{x, t} \|^2]+ \frac{4\alpha_z^2\eta^2 L_{g, 1}^2}{N}\mathbb{E}[\|   \bar{m}_{z, t}  \|^2] \ ,
\end{align}
\begin{align}
	& 	\mathbb{E}[\|\frac{1}{N}\sum_{n=1}^{N} \nabla_{2} f^{(n)}({x}_{t+1}^{(n)}, {y}_{t+1}^{(n)}) -\frac{1}{N}\sum_{n=1}^{N} v^{(n)}_{1, t+1}  \|^2]  \leq (1-\beta_y\eta^2) \mathbb{E}[\| \frac{1}{N}\sum_{n=1}^{N}\nabla_{2} f^{(n)}(x^{(n)}_{t}, y^{(n)}_{t})- \frac{1}{N}\sum_{n=1}^{N} v^{(n)}_{1, t}\|^2]  + 2\beta_y^2\eta^4\sigma^2 \frac{1}{N} \notag \\
	& \quad +  \frac{12\alpha_x^2\eta^2L_{f}^{2}}{N} \frac{1}{N}\sum_{n=1}^{N}\mathbb{E}[\| u^{(n)}_{1, t} - \bar{u}_{1, t} \|^2] + \frac{12\alpha_x^2\eta^2\lambda^2 L_{f}^{2}}{N}\frac{1}{N}\sum_{n=1}^{N}\mathbb{E}[\| u^{(n)}_{2, t} - \bar{u}_{2, t} \|^2] + \frac{12\alpha_x^2\eta^2\lambda^2 L_{f}^{2}}{N}\frac{1}{N}\sum_{n=1}^{N}\mathbb{E}[\| u^{(n)}_{3, t} - \bar{u}_{3, t} \|^2] \notag \\
	& \quad +  \frac{8\alpha_y^2\eta^2 L_{f}^{2}}{N}\frac{1}{N}\sum_{n=1}^{N}\mathbb{E}[\| v^{(n)}_{1, t} - \bar{v}_{1, t} \|^2] + \frac{8\alpha_y^2\eta^2\lambda^2L_{f}^{2}}{N} \frac{1}{N}\sum_{n=1}^{N}\mathbb{E}[\| v^{(n)}_{2, t} - \bar{v}_{2, t} \|^2] \notag \\
	& \quad + \frac{12\alpha_x^2\eta^2L_{f}^{2}}{N}  \mathbb{E}[\| \bar{m}_{x, t} \|^2] +\frac{8\alpha_y^2\eta^2L_{f}^{2}}{N}   \mathbb{E}[\| \bar{m}_{y, t} \|^2] \ ,
\end{align}
\begin{align}
	&\mathbb{E}[\|\frac{1}{N}\sum_{n=1}^{N}  \nabla_{2} g^{(n)}({x}_{t+1}^{(n)}, {y}_{t+1}^{(n)})-\frac{1}{N}\sum_{n=1}^{N}     v^{(n)}_{2, t+1} \|^2] \leq (1-\beta_y\eta^2) \mathbb{E}[\| \frac{1}{N}\sum_{n=1}^{N}\nabla_{2} g^{(n)}(x^{(n)}_{t}, y^{(n)}_{t})- \frac{1}{N}\sum_{n=1}^{N} v^{(n)}_{2, t}\|^2]  + 2\beta_y^2\eta^4\sigma^2 \frac{1}{N} \notag \\
	& \quad +  \frac{12\alpha_x^2\eta^2L_{g, 1}^{2}}{N} \frac{1}{N}\sum_{n=1}^{N}\mathbb{E}[\| u^{(n)}_{1, t} - \bar{u}_{1, t} \|^2] + \frac{12\alpha_x^2\eta^2\lambda^2 L_{g, 1}^{2}}{N}\frac{1}{N}\sum_{n=1}^{N}\mathbb{E}[\| u^{(n)}_{2, t} - \bar{u}_{2, t} \|^2] + \frac{12\alpha_x^2\eta^2\lambda^2 L_{g, 1}^{2}}{N}\frac{1}{N}\sum_{n=1}^{N}\mathbb{E}[\| u^{(n)}_{3, t} - \bar{u}_{3, t} \|^2] \notag \\
	& \quad +  \frac{8\alpha_y^2\eta^2 L_{g, 1}^{2}}{N}\frac{1}{N}\sum_{n=1}^{N}\mathbb{E}[\| v^{(n)}_{1, t} - \bar{v}_{1, t} \|^2] + \frac{8\alpha_y^2\eta^2\lambda^2L_{g, 1}^{2}}{N} \frac{1}{N}\sum_{n=1}^{N}\mathbb{E}[\| v^{(n)}_{2, t} - \bar{v}_{2, t} \|^2] \notag \\
	& \quad + \frac{12\alpha_x^2\eta^2L_{g, 1}^{2}}{N}  \mathbb{E}[\| \bar{m}_{x, t} \|^2] +\frac{8\alpha_y^2\eta^2L_{g, 1}^{2}}{N}  \mathbb{E}[\| \bar{m}_{y, t} \|^2] \ ,
\end{align}
\begin{align}
	& 	\mathbb{E}[\| \frac{1}{N}\sum_{n=1}^{N} \nabla_{2} g^{(n)}({x}^{(n)}_{t+1}, {z}^{(n)}_{t+1}) - \frac{1}{N}\sum_{n=1}^{N}{w}^{(n)}_{1, t+1} \|^2] \leq (1-\beta_z\eta^2) \mathbb{E}[\| \frac{1}{N}\sum_{n=1}^{N}\nabla_{2} g^{(n)}(x^{(n)}_{t}, z^{(n)}_{t})- \frac{1}{N}\sum_{n=1}^{N} w^{(n)}_{1, t}\|^2]  + 2\beta_z^2\eta^4\sigma^2 \frac{1}{N} \notag \\
	& \quad +  \frac{12\alpha_x^2\eta^2L_{g, 1}^{2}}{N} \frac{1}{N}\sum_{n=1}^{N}\mathbb{E}[\| u^{(n)}_{1, t} - \bar{u}_{1, t} \|^2] + \frac{12\alpha_x^2\eta^2\lambda^2 L_{g, 1}^{2}}{N}\frac{1}{N}\sum_{n=1}^{N}\mathbb{E}[\| u^{(n)}_{2, t} - \bar{u}_{2, t} \|^2] + \frac{12\alpha_x^2\eta^2\lambda^2 L_{g, 1}^{2}}{N}\frac{1}{N}\sum_{n=1}^{N}\mathbb{E}[\| u^{(n)}_{3, t} - \bar{u}_{3, t} \|^2] \notag \\
	& \quad + \frac{4\lambda^2\alpha_z^2\eta^2L_{g, 1}^2}{N} \frac{1}{N}\sum_{n=1}^{N}\mathbb{E}[\|    w^{(n)}_{1, t} -  \bar{w}_{1, t}   \|^2] + \frac{12\alpha_x^2\eta^2L_{g, 1}^{2}}{N}  \mathbb{E}[\| \bar{m}_{x, t} \|^2]+ \frac{4\alpha_z^2\eta^2 L_{g, 1}^2}{N}\mathbb{E}[\|   \bar{m}_{z, t}  \|^2] \  .
\end{align}
	
\end{lemma}


\begin{proof}

    \begin{align}\label{eq:global-grad-var-1}
		& \quad \mathbb{E}[\|\frac{1}{N}\sum_{n=1}^{N}\nabla_{1} f^{(n)}(x^{(n)}_{t+1}, y^{(n)}_{t+1})- \frac{1}{N}\sum_{n=1}^{N} u^{(n)}_{1, t+1} \|^2]   \notag \\ 
		& = \mathbb{E}[\| \frac{1}{N}\sum_{n=1}^{N} ((1-\beta_x\eta^2)(u^{(n)}_{1, t} - \nabla_1 f^{(n)}(x^{(n)}_{t}, y^{(n)}_{t}; \xi^{(n)}_{t+1})) + \nabla_1 f^{(n)}(x^{(n)}_{t+1}, y^{(n)}_{t+1}; \xi^{(n)}_{t+1}))  - \frac{1}{N}\sum_{n=1}^{N}\nabla_{1} f^{(n)}(x^{(n)}_{t+1}, y^{(n)}_{t+1}) \|^2] \notag \\
		& = \mathbb{E}[\| \frac{1}{N}\sum_{n=1}^{N} \Big((1-\beta_x\eta^2)(u^{(n)}_{1, t} - \nabla_1 f^{(n)}(x^{(n)}_{t}, y^{(n)}_{t}))  \notag \\
		& \quad + (1-\beta_x\eta^2)(\nabla_1 f^{(n)}(x^{(n)}_{t}, y^{(n)}_{t}) - \nabla_1 f^{(n)}(x^{(n)}_{t}, y^{(n)}_{t}; \xi^{(n)}_{t+1})) + \nabla_1 f^{(n)}(x^{(n)}_{t+1}, y^{(n)}_{t+1}; \xi^{(n)}_{t+1})\Big) \notag \\
		& \quad - \frac{1}{N}\sum_{n=1}^{N}\nabla_{1} f^{(n)}(x^{(n)}_{t+1}, y^{(n)}_{t+1}) \|^2] \notag \\
		& = \mathbb{E}[\| \frac{1}{N}\sum_{n=1}^{N} \Big((1-\beta_x\eta^2)(u^{(n)}_{1, t} - \nabla_1 f^{(n)}(x^{(n)}_{t}, y^{(n)}_{t})\Big) \notag \\
		& \quad + \frac{1}{N}\sum_{n=1}^{N} \Big(\nabla_1 f^{(n)}(x^{(n)}_{t}, y^{(n)}_{t}) - \nabla_1 f^{(n)}(x^{(n)}_{t}, y^{(n)}_{t}; \xi^{(n)}_{t+1}) + \nabla_1 f^{(n)}(x^{(n)}_{t+1}, y^{(n)}_{t+1}; \xi^{(n)}_{t+1}) - \nabla_{1} f^{(n)}(x^{(n)}_{t+1}, y^{(n)}_{t+1}) \Big)  \notag \\
		& \quad + \frac{1}{N}\sum_{n=1}^{N} \beta_x\eta^2(\nabla_1 f^{(n)}(x^{(n)}_{t}, y^{(n)}_{t}; \xi^{(n)}_{t+1}) - \nabla_1 f^{(n)}(x^{(n)}_{t}, y^{(n)}_{t}))  \|^2] \notag \\
		& \leq \mathbb{E}[\| \frac{1}{N}\sum_{n=1}^{N} (1-\beta_x\eta^2)(u^{(n)}_{1, t} - \nabla_1 f^{(n)}(x^{(n)}_{t}, y^{(n)}_{t}) \|^2] \notag \\
		& \quad + 2 \mathbb{E}[\| \frac{1}{N}\sum_{n=1}^{N} (\nabla_1 f^{(n)}(x^{(n)}_{t}, y^{(n)}_{t}) - \nabla_1 f^{(n)}(x^{(n)}_{t}, y^{(n)}_{t}; \xi^{(n)}_{t+1}) + \nabla_1 f^{(n)}(x^{(n)}_{t+1}, y^{(n)}_{t+1}; \xi^{(n)}_{t+1}) - \nabla_{1} f^{(n)}(x^{(n)}_{t+1}, y^{(n)}_{t+1})) \|^2] \notag \\
		& \quad + 2 \mathbb{E}[\| \frac{1}{N}\sum_{n=1}^{N} \beta_x\eta^2(\nabla_1 f^{(n)}(x^{(n)}_{t}, y^{(n)}_{t}; \xi^{(n)}_{t+1}) - \nabla_1 f^{(n)}(x^{(n)}_{t}, y^{(n)}_{t}))  \|^2] \notag \\
		& \leq (1-\beta_x\eta^2)^{2} \mathbb{E}[\| \frac{1}{N}\sum_{n=1}^{N}u^{(n)}_{1, t} - \nabla_1 f^{(n)}(x^{(n)}_{t}, y^{(n)}_{t}) \|^2] \notag \\
		& \quad + 2 \frac{1}{N^2}\sum_{n=1}^{N} \mathbb{E}[\| \nabla_1 f^{(n)}(x^{(n)}_{t+1}, y^{(n)}_{t+1}; \xi^{(n)}_{t+1}) - \nabla_1 f^{(n)}(x^{(n)}_{t}, y^{(n)}_{t}; \xi^{(n)}_{t+1}) \|^2] \notag \\
		& \quad + 2\beta_x^2\eta^4 \frac{1}{N^2}\sum_{n=1}^{N}\mathbb{E}[\| \nabla_1 f^{(n)}(x^{(n)}_{t}, y^{(n)}_{t}; \xi^{(n)}_{t+1}) - \nabla_1 f^{(n)}(x^{(n)}_{t}, y^{(n)}_{t})  \|^2] \notag \\
		& \leq (1-\beta_x\eta^2) \mathbb{E}[\| \frac{1}{N}\sum_{n=1}^{N}u^{(n)}_{1, t} - \nabla_1 f^{(n)}(x^{(n)}_{t}, y^{(n)}_{t}) \|^2] \notag \\
		& \quad + 2L_{f}^{2} \frac{1}{N^2}\sum_{n=1}^{N}\mathbb{E}[\| x^{(n)}_{t+1} - x^{(n)}_{t} \|^2] + 2L_{f}^{2} \frac{1}{N^2}\sum_{n=1}^{N}\mathbb{E}[\| y^{(n)}_{t+1} - y^{(n)}_{t} \|^2] + 2\beta_x^2\eta^4\sigma^2 \frac{1}{N} \notag \\
		& \leq (1-\beta_x\eta^2) \mathbb{E}[\| \frac{1}{N}\sum_{n=1}^{N}u^{(n)}_{1, t} - \nabla_1 f^{(n)}(x^{(n)}_{t}, y^{(n)}_{t}) \|^2] + 2\beta_x^2\eta^4\sigma^2 \frac{1}{N}\notag \\
		& \quad +  \frac{12\alpha_x^2\eta^2L_{f}^{2}}{N} \frac{1}{N}\sum_{n=1}^{N}\mathbb{E}[\| u^{(n)}_{1, t} - \bar{u}_{1, t} \|^2] + \frac{12\alpha_x^2\eta^2\lambda^2 L_{f}^{2}}{N}\frac{1}{N}\sum_{n=1}^{N}\mathbb{E}[\| u^{(n)}_{2, t} - \bar{u}_{2, t} \|^2] + \frac{12\alpha_x^2\eta^2\lambda^2 L_{f}^{2}}{N}\frac{1}{N}\sum_{n=1}^{N}\mathbb{E}[\| u^{(n)}_{3, t} - \bar{u}_{3, t} \|^2] \notag \\
		& \quad +  \frac{8\alpha_y^2\eta^2 L_{f}^{2}}{N}\frac{1}{N}\sum_{n=1}^{N}\mathbb{E}[\| v^{(n)}_{1, t} - \bar{v}_{1, t} \|^2] + \frac{8\alpha_y^2\eta^2\lambda^2L_{f}^{2}}{N} \frac{1}{N}\sum_{n=1}^{N}\mathbb{E}[\| v^{(n)}_{2, t} - \bar{v}_{2, t} \|^2] \notag \\
		& \quad + \frac{12\alpha_x^2\eta^2L_{f}^{2}}{N}  \mathbb{E}[\| \bar{m}_{x, t} \|^2] +\frac{8\alpha_y^2\eta^2L_{f}^{2}}{N}  \mathbb{E}[\| \bar{m}_{y, t} \|^2] \  , 
	\end{align}
    where the first inequality follows from $\mathbb{E}[\nabla_1 f^{(n)}(x^{(n)}_{t}, y^{(n)}_{t}) - \nabla_1 f^{(n)}(x^{(n)}_{t}, y^{(n)}_{t}; \xi^{(n)}_{t+1}) + \nabla_1 f^{(n)}(x^{(n)}_{t+1}, y^{(n)}_{t+1}; \xi^{(n)}_{t+1}) - \nabla_{1} f^{(n)}(x^{(n)}_{t+1}, y^{(n)}_{t+1})] + \beta_x\eta^2(\nabla_1 f^{(n)}(x^{(n)}_{t}, y^{(n)}_{t}; \xi^{(n)}_{t+1}) - \nabla_1 f^{(n)}(x^{(n)}_{t}, y^{(n)}_{t}))=0$, the second inequality follows from the fact that the sampling operation on different workers is independent, and the last inequality follows from Lemma~\ref{lemma:incremental-x-y-z}.

Following the proof of Eq.~(\ref{eq:global-grad-var-1}), we can prove the other inequalities. Therefore, we omit the detailed proof to save space.

\end{proof}

\begin{lemma} \label{lemma:grad-diff}
Given Assumptions~\ref{assumption:f-smooth}-\ref{assumption:heterogeneous}, we have the following inequalities:
 \begin{align}
		& \quad \frac{1}{N}\sum_{n=1}^{N}\mathbb{E}[\|  \nabla_1 f^{(n)}(x^{(n)}_{t}, y^{(n)}_{t}; \xi^{(n)}_{t+1}) - \frac{1}{N}\sum_{m=1}^{N} \nabla_1 f^{(m)}(x^{(m)}_{t}, y^{(m)}_{t}; \xi^{(m)}_{t+1}) \|^2] \notag \\
		& \leq 18L_{f}^{2} \frac{1}{N}\sum_{n=1}^{N}\mathbb{E}[\| x^{(n)}_{t} - \bar{x}_{t} \|^2] + 18L_{f}^{2} \frac{1}{N}\sum_{n=1}^{N}\mathbb{E}[\| y^{(n)}_{t} - \bar{y}_{t} \|^2] + 9\delta^2 + 6 \sigma^{2} \ ,
	\end{align}
     \begin{align}
		& \quad \frac{1}{N}\sum_{n=1}^{N}\mathbb{E}[\|  \nabla_2 f^{(n)}(x^{(n)}_{t}, y^{(n)}_{t}; \xi^{(n)}_{t+1}) - \frac{1}{N}\sum_{m=1}^{N} \nabla_2 f^{(m)}(x^{(m)}_{t}, y^{(m)}_{t}; \xi^{(m)}_{t+1}) \|^2] \notag \\
		& \leq 18L_{f}^{2} \frac{1}{N}\sum_{n=1}^{N}\mathbb{E}[\| x^{(n)}_{t} - \bar{x}_{t} \|^2] + 18L_{f}^{2} \frac{1}{N}\sum_{n=1}^{N}\mathbb{E}[\| y^{(n)}_{t} - \bar{y}_{t} \|^2] + 9\delta^2 + 6 \sigma^{2} \ ,
	\end{align}
    \begin{align}
		& \quad \frac{1}{N}\sum_{n=1}^{N}\mathbb{E}[\|  \nabla_1 g^{(n)}(x^{(n)}_{t}, y^{(n)}_{t}; \xi^{(n)}_{t+1}) - \frac{1}{N}\sum_{m=1}^{N} \nabla_1 g^{(m)}(x^{(m)}_{t}, y^{(m)}_{t}; \xi^{(m)}_{t+1}) \|^2] \notag \\
		& \leq 18L_{g, 1}^{2} \frac{1}{N}\sum_{n=1}^{N}\mathbb{E}[\| x^{(n)}_{t} - \bar{x}_{t} \|^2] + 18L_{g, 1}^{2} \frac{1}{N}\sum_{n=1}^{N}\mathbb{E}[\| y^{(n)}_{t} - \bar{y}_{t} \|^2] + 9\delta^2 + 6 \sigma^{2} \ ,
	\end{align}
     \begin{align}
		& \quad \frac{1}{N}\sum_{n=1}^{N}\mathbb{E}[\|  \nabla_2 g^{(n)}(x^{(n)}_{t}, y^{(n)}_{t}; \xi^{(n)}_{t+1}) - \frac{1}{N}\sum_{m=1}^{N} \nabla_2 g^{(m)}(x^{(m)}_{t}, y^{(m)}_{t}; \xi^{(m)}_{t+1}) \|^2] \notag \\
		& \leq 18L_{g, 1}^{2} \frac{1}{N}\sum_{n=1}^{N}\mathbb{E}[\| x^{(n)}_{t} - \bar{x}_{t} \|^2] + 18L_{g, 1}^{2} \frac{1}{N}\sum_{n=1}^{N}\mathbb{E}[\| y^{(n)}_{t} - \bar{y}_{t} \|^2] + 9\delta^2 + 6 \sigma^{2} \ ,
	\end{align}
    \begin{align}
		& \quad \frac{1}{N}\sum_{n=1}^{N}\mathbb{E}[\|  \nabla_2 g^{(n)}(x^{(n)}_{t}, z^{(n)}_{t}; \xi^{(n)}_{t+1}) - \frac{1}{N}\sum_{m=1}^{N} \nabla_2 g^{(m)}(x^{(m)}_{t}, z^{(m)}_{t}; \xi^{(m)}_{t+1}) \|^2] \notag \\
		& \leq 18L_{g, 1}^{2} \frac{1}{N}\sum_{n=1}^{N}\mathbb{E}[\| x^{(n)}_{t} - \bar{x}_{t} \|^2] + 18L_{g, 1}^{2} \frac{1}{N}\sum_{n=1}^{N}\mathbb{E}[\| z^{(n)}_{t} - \bar{z}_{t} \|^2] + 9\delta^2 + 6 \sigma^{2} \ .
	\end{align}
    
\end{lemma}

\begin{proof}
     \begin{align}
		& \quad \frac{1}{N}\sum_{n=1}^{N}\mathbb{E}[\|  \nabla_1 f^{(n)}(x^{(n)}_{t}, y^{(n)}_{t}; \xi^{(n)}_{t+1}) - \frac{1}{N}\sum_{m=1}^{N} \nabla_1 f^{(m)}(x^{(m)}_{t}, y^{(m)}_{t}; \xi^{(m)}_{t+1}) \|^2] \notag \\
		& = \frac{1}{N}\sum_{n=1}^{N}\mathbb{E}[\|  \nabla_1 f^{(n)}(x^{(n)}_{t}, y^{(n)}_{t}; \xi^{(n)}_{t+1}) - \nabla_1 f^{(n)}(x^{(n)}_{t}, y^{(n)}_{t}) \notag \\
		& \quad+ \nabla_1 f^{(n)}(x^{(n)}_{t}, y^{(n)}_{t}) - \frac{1}{N}\sum_{m=1}^{N} \nabla_1 f^{(m)}(x^{(m)}_{t}, y^{(m)}_{t}) \notag \\
		& \quad+ \frac{1}{N}\sum_{m=1}^{N} \nabla_1 f^{(m)}(x^{(m)}_{t}, y^{(m)}_{t}) - \frac{1}{N}\sum_{m=1}^{N} \nabla_1 f^{(m)}(x^{(m)}_{t}, y^{(m)}_{t}; \xi^{(m)}_{t+1}) \|^2] \notag \\
		& \leq 3 \frac{1}{N}\sum_{n=1}^{N}\mathbb{E}[\|  \nabla_1 f^{(n)}(x^{(n)}_{t}, y^{(n)}_{t}) - \frac{1}{N}\sum_{m=1}^{N} \nabla_1 f^{(m)}(x^{(m)}_{t}, y^{(m)}_{t}) \|^2] + 6 \sigma^{2} \notag \\
		& = 3 \frac{1}{N}\sum_{n=1}^{N}\mathbb{E}[\|  \nabla_1 f^{(n)}(x^{(n)}_{t}, y^{(n)}_{t}) - \nabla_1 f^{(n)}(\bar{x}_{t}, \bar{y}_{t}) \notag \\
		& \quad+ \nabla_1 f^{(n)}(\bar{x}_{t}, \bar{y}_{t}) - \frac{1}{N}\sum_{m=1}^{N} \nabla_1 f^{(m)}(\bar{x}^{(m)}_{t}, \bar{y}^{(m)}_{t}) \notag \\
		& \quad+ \frac{1}{N}\sum_{m=1}^{N} \nabla_1 f^{(m)}(\bar{x}^{(m)}_{t}, \bar{y}^{(m)}_{t}) - \frac{1}{N}\sum_{m=1}^{N} \nabla_1 f^{(m)}(x^{(m)}_{t}, y^{(m)}_{t}) \|^2] + 6 \sigma^{2} \notag \\
		& \leq 9 \frac{1}{N}\sum_{n=1}^{N}\mathbb{E}[\|  \nabla_1 f^{(n)}(x^{(n)}_{t}, y^{(n)}_{t}) - \nabla_1 f^{(n)}(\bar{x}_{t}, \bar{y}_{t}) \|^2] \notag \\
		&\quad  + 9 \frac{1}{N}\sum_{n=1}^{N}\mathbb{E}[\| \nabla_1 f^{(n)}(\bar{x}_{t}, \bar{y}_{t}) - \frac{1}{N}\sum_{m=1}^{N} \nabla_1 f^{(m)}(\bar{x}^{(m)}_{t}, \bar{y}^{(m)}_{t}) \|^2] \notag \\
		& \quad + 9 \frac{1}{N}\sum_{n=1}^{N}\mathbb{E}[\| \frac{1}{N}\sum_{m=1}^{N} \nabla_1 f^{(m)}(\bar{x}^{(m)}_{t}, \bar{y}^{(m)}_{t}) - \frac{1}{N}\sum_{m=1}^{N} \nabla_1 f^{(m)}(x^{(m)}_{t}, y^{(m)}_{t}) \|^2] \notag \\
		& \leq 18L_{f}^{2} \frac{1}{N}\sum_{n=1}^{N}\mathbb{E}[\| x^{(n)}_{t} - \bar{x}_{t} \|^2] + 18L_{f}^{2} \frac{1}{N}\sum_{n=1}^{N}\mathbb{E}[\| y^{(n)}_{t} - \bar{y}_{t} \|^2] + 9\delta^2 + 6 \sigma^{2} \ ,
	\end{align}
    where the last step follows from Assumptions~\ref{assumption:f-smooth},~\ref{assumption:variance},~\ref{assumption:heterogeneous}. 

    The other inequalities can be proved in the same way. Therefore, we omit their detailed proof. 
\end{proof}

\begin{lemma}\label{lemma:consensus-error-momentum-each-step}
Given Assumptions~\ref{assumption:f-smooth}-\ref{assumption:heterogeneous}, by setting $\eta\leq \min\{\frac{1}{\sqrt{\beta_x}}, \frac{1}{\sqrt{\beta_y}}, \frac{1}{\sqrt{\beta_z}}\}$,  we have the following inequalities: 
	\begin{align}
		&  \frac{1}{N}\sum_{n=1}^{N}\mathbb{E}[\|u^{(n)}_{1, t+1} -\bar{u}_{1, t+1}\|^2]  \leq (1+\frac{1}{p}+24pL_{f}^2\alpha_x^2\eta^2)\frac{1}{N}\sum_{n=1}^{N}\mathbb{E}[\|u^{(n)}_{1, t} - \bar{u}_{1, t}\|^2]  +   24pL_{f}^2\alpha_x^2\eta^2\lambda^2 \frac{1}{N}\sum_{n=1}^{N}\mathbb{E}[\| u^{(n)}_{2, t} - \bar{u}_{2, t} \|^2]  \notag \\
    	& \quad + 24pL_{f}^2\alpha_x^2\eta^2\lambda^2 \frac{1}{N}\sum_{n=1}^{N}\mathbb{E}[\| u^{(n)}_{3, t} - \bar{u}_{3, t} \|^2] + 16pL_{f}^2\alpha_y^2\eta^2 \frac{1}{N}\sum_{n=1}^{N}\mathbb{E}[\| v^{(n)}_{1, t} - \bar{v}_{1, t} \|^2] + 16pL_{f}^2\alpha_y^2\eta^2\lambda^2 \frac{1}{N}\sum_{n=1}^{N}\mathbb{E}[\| v^{(n)}_{2, t} - \bar{v}_{2, t} \|^2]  \notag \\
    	& \quad+ 72p\beta_x^2\eta^4 L_{f}^{2} \frac{1}{N}\sum_{n=1}^{N}\mathbb{E}[\| x^{(n)}_{t} - \bar{x}_{t} \|^2] + 72p\beta_x^2\eta^4 L_{f}^{2} \frac{1}{N}\sum_{n=1}^{N}\mathbb{E}[\| y^{(n)}_{t} - \bar{y}_{t} \|^2] \notag \\
    	& \quad + 24pL_{f}^2\alpha_x^2\eta^2 \mathbb{E}[\| \bar{m}_{x, t} \|^2]   + 16pL_{f}^2\alpha_y^2\eta^2 \mathbb{E}[\| \bar{m}_{y, t} \|^2]  + 36p\beta_x^2\eta^4 \delta^2 + 24p\beta_x^2\eta^4  \sigma^{2} \ ,
	\end{align}
	\begin{align}
		&  \frac{1}{N}\sum_{n=1}^{N}\mathbb{E}[\|u^{(n)}_{2, t+1} -\bar{u}_{2, t+1}\|^2]  \leq (1+\frac{1}{p}+24pL_{g, 1}^2\alpha_x^2\eta^2\lambda^2)\frac{1}{N}\sum_{n=1}^{N}\mathbb{E}[\|u^{(n)}_{2, t} - \bar{u}_{2, t}\|^2]  + 24pL_{g, 1}^2\alpha_x^2\eta^2 \frac{1}{N}\sum_{n=1}^{N}\mathbb{E}[\| u^{(n)}_{1, t} - \bar{u}_{1, t} \|^2]  \notag \\
	& \quad + 24pL_{g, 1}^2\alpha_x^2\eta^2\lambda^2 \frac{1}{N}\sum_{n=1}^{N}\mathbb{E}[\| u^{(n)}_{3, t} - \bar{u}_{3, t} \|^2] + 16pL_{g, 1}^2\alpha_y^2\eta^2 \frac{1}{N}\sum_{n=1}^{N}\mathbb{E}[\| v^{(n)}_{1, t} - \bar{v}_{1, t} \|^2] + 16pL_{g, 1}^2\alpha_y^2\eta^2\lambda^2 \frac{1}{N}\sum_{n=1}^{N}\mathbb{E}[\| v^{(n)}_{2, t} - \bar{v}_{2, t} \|^2]  \notag \\
	& \quad+ 72p\beta_x^2\eta^4 L_{g, 1}^{2} \frac{1}{N}\sum_{n=1}^{N}\mathbb{E}[\| x^{(n)}_{t} - \bar{x}_{t} \|^2] + 72p\beta_x^2\eta^4 L_{g, 1}^{2} \frac{1}{N}\sum_{n=1}^{N}\mathbb{E}[\| y^{(n)}_{t} - \bar{y}_{t} \|^2] \notag \\
	& \quad + 24pL_{g, 1}^2\alpha_x^2\eta^2 \mathbb{E}[\| \bar{m}_{x, t} \|^2]  + 16pL_{g, 1}^2\alpha_y^2\eta^2 \mathbb{E}[\| \bar{m}_{y, t} \|^2]  + 36p\beta_x^2\eta^4 \delta^2 + 24p\beta_x^2\eta^4  \sigma^{2} \ ,
	\end{align}
	\begin{align}
		&  \frac{1}{N}\sum_{n=1}^{N}\mathbb{E}[\|u^{(n)}_{3, t+1} -\bar{u}_{3, t+1}\|^2]\leq (1+\frac{1}{p}+ 24pL_{g, 1}^2\alpha_x^2\eta^2\lambda^2)\frac{1}{N}\sum_{n=1}^{N}\mathbb{E}[\|u^{(n)}_{3, t} - \bar{u}_{3, t}\|^2]+ 24pL_{g, 1}^2\alpha_x^2\eta^2 \frac{1}{N}\sum_{n=1}^{N}\mathbb{E}[\| u^{(n)}_{1, t} - \bar{u}_{1, t} \|^2]  \notag \\
	& \quad  + 24pL_{g, 1}^2\alpha_x^2\eta^2\lambda^2 \frac{1}{N}\sum_{n=1}^{N}\mathbb{E}[\| u^{(n)}_{2, t} - \bar{u}_{2, t} \|^2]  + 8p L_{g, 1}^2\lambda^2\alpha_z^2\eta^2 \frac{1}{N}\sum_{n=1}^{N}\mathbb{E}[\|    w^{(n)}_{1, t} -  \bar{w}_{1, t}   \|^2] \notag \\
	& \quad + 72p\beta_x^2\eta^4L_{g, 1}^{2} \frac{1}{N}\sum_{n=1}^{N}\mathbb{E}[\| x^{(n)}_{t} - \bar{x}_{t} \|^2] + 72p\beta_x^2\eta^4L_{g, 1}^{2} \frac{1}{N}\sum_{n=1}^{N}\mathbb{E}[\| z^{(n)}_{t} - \bar{z}_{t} \|^2] \notag \\
    	& \quad + 24pL_{g, 1}^2\alpha_x^2\eta^2 \mathbb{E}[\| \bar{m}_{x, t} \|^2]+ 8p L_{g, 1}^2\alpha_z^2\eta^2 \mathbb{E}[\|   \bar{m}_{z, t}  \|^2]  + 36p\beta_x^2\eta^4\delta^2 +24p\beta_x^2\eta^4\sigma^{2} \ , 
	\end{align}
	\begin{align}
	&  \frac{1}{N}\sum_{n=1}^{N}\mathbb{E}[\|v^{(n)}_{1, t+1} -\bar{v}_{1, t+1}\|^2]  \leq (1+\frac{1}{p}+16pL_{f}^2\alpha_y^2\eta^2 )\frac{1}{N}\sum_{n=1}^{N}\mathbb{E}[\|v^{(n)}_{1, t} - \bar{v}_{1, t}\|^2]  + 24pL_{f}^2\alpha_x^2\eta^2 \frac{1}{N}\sum_{n=1}^{N}\mathbb{E}[\| u^{(n)}_{1, t} - \bar{u}_{1, t} \|^2]   \notag \\
	& \quad + 24pL_{f}^2\alpha_x^2\eta^2\lambda^2 \frac{1}{N}\sum_{n=1}^{N}\mathbb{E}[\| u^{(n)}_{2, t} - \bar{u}_{2, t} \|^2]+ 24pL_{f}^2\alpha_x^2\eta^2\lambda^2 \frac{1}{N}\sum_{n=1}^{N}\mathbb{E}[\| u^{(n)}_{3, t} - \bar{u}_{3, t} \|^2] + 16pL_{f}^2\alpha_y^2\eta^2\lambda^2 \frac{1}{N}\sum_{n=1}^{N}\mathbb{E}[\| v^{(n)}_{2, t} - \bar{v}_{2, t} \|^2]  \notag \\
	& \quad+ 72p\beta_y^2\eta^4 L_{f}^{2} \frac{1}{N}\sum_{n=1}^{N}\mathbb{E}[\| x^{(n)}_{t} - \bar{x}_{t} \|^2] + 72p\beta_y^2\eta^4 L_{f}^{2} \frac{1}{N}\sum_{n=1}^{N}\mathbb{E}[\| y^{(n)}_{t} - \bar{y}_{t} \|^2] \notag \\
	& \quad + 24pL_{f}^2\alpha_x^2\eta^2 \mathbb{E}[\| \bar{m}_{x, t} \|^2]  + 16pL_{f}^2\alpha_y^2\eta^2 \mathbb{E}[\| \bar{m}_{y, t} \|^2]    + 36p\beta_y^2\eta^4 \delta^2 + 24p\beta_y^2\eta^4  \sigma^{2} \ ,
\end{align}
\begin{align}
	&  \frac{1}{N}\sum_{n=1}^{N}\mathbb{E}[\|v^{(n)}_{2, t+1} -\bar{v}_{2, t+1}\|^2]  \leq (1+\frac{1}{p}+16pL_{g, 1}^2\alpha_y^2\eta^2\lambda^2)\frac{1}{N}\sum_{n=1}^{N}\mathbb{E}[\|v^{(n)}_{2, t} - \bar{v}_{2, t}\|^2]  + 24pL_{g, 1}^2\alpha_x^2\eta^2 \frac{1}{N}\sum_{n=1}^{N}\mathbb{E}[\| u^{(n)}_{1, t} - \bar{u}_{1, t} \|^2]  \notag \\
	& \quad + 24pL_{g, 1}^2\alpha_x^2\eta^2\lambda^2 \frac{1}{N}\sum_{n=1}^{N}\mathbb{E}[\| u^{(n)}_{2, t} - \bar{u}_{2, t} \|^2] + 24pL_{g, 1}^2\alpha_x^2\eta^2\lambda^2 \frac{1}{N}\sum_{n=1}^{N}\mathbb{E}[\| u^{(n)}_{3, t} - \bar{u}_{3, t} \|^2] + 16pL_{g, 1}^2\alpha_y^2\eta^2 \frac{1}{N}\sum_{n=1}^{N}\mathbb{E}[\| v^{(n)}_{1, t} - \bar{v}_{1, t} \|^2]  \notag \\
	& \quad+ 72p\beta_y^2\eta^4 L_{g, 1}^{2} \frac{1}{N}\sum_{n=1}^{N}\mathbb{E}[\| x^{(n)}_{t} - \bar{x}_{t} \|^2] + 72p\beta_y^2\eta^4 L_{g, 1}^{2} \frac{1}{N}\sum_{n=1}^{N}\mathbb{E}[\| y^{(n)}_{t} - \bar{y}_{t} \|^2] \notag \\
	& \quad + 24pL_{g, 1}^2\alpha_x^2\eta^2\mathbb{E}[\| \bar{m}_{x, t} \|^2] + 16pL_{g, 1}^2\alpha_y^2\eta^2 \mathbb{E}[\| \bar{m}_{y, t} \|^2]  + 36p\beta_y^2\eta^4 \delta^2 + 24p\beta_y^2\eta^4  \sigma^{2} \ ,
\end{align}
	\begin{align}
		&  \frac{1}{N}\sum_{n=1}^{N}\mathbb{E}[\|w^{(n)}_{1, t+1} -\bar{w}_{1, t+1}\|^2] \leq (1+\frac{1}{p}+ 8p L_{g, 1}^2\lambda^2\alpha_z^2\eta^2)\frac{1}{N}\sum_{n=1}^{N}\mathbb{E}[\|w^{(n)}_{1, t} - \bar{w}_{1, t}\|^2]+ 24pL_{g, 1}^2\alpha_x^2\eta^2 \frac{1}{N}\sum_{n=1}^{N}\mathbb{E}[\| u^{(n)}_{1, t} - \bar{u}_{1, t} \|^2]  \notag \\
	& \quad + 24pL_{g, 1}^2\alpha_x^2\eta^2\lambda^2 \frac{1}{N}\sum_{n=1}^{N}\mathbb{E}[\| u^{(n)}_{2, t} - \bar{u}_{2, t} \|^2]  + 24pL_{g, 1}^2\alpha_x^2\eta^2\lambda^2 \frac{1}{N}\sum_{n=1}^{N}\mathbb{E}[\| u^{(n)}_{3, t} - \bar{u}_{3, t} \|^2]  \notag \\
	& \quad + 72p\beta_z^2\eta^4L_{g, 1}^{2} \frac{1}{N}\sum_{n=1}^{N}\mathbb{E}[\| x^{(n)}_{t} - \bar{x}_{t} \|^2] + 72p\beta_z^2\eta^4L_{g, 1}^{2} \frac{1}{N}\sum_{n=1}^{N}\mathbb{E}[\| z^{(n)}_{t} - \bar{z}_{t} \|^2] \notag \\
	& \quad + 24pL_{g, 1}^2\alpha_x^2\eta^2 \mathbb{E}[\| \bar{m}_{x, t} \|^2]+ 8p L_{g, 1}^2\alpha_z^2\eta^2 \mathbb{E}[\|   \bar{m}_{z, t}  \|^2]+ 36p\beta_z^2\eta^4\delta^2 +24p\beta_z^2\eta^4\sigma^{2} \ .
	\end{align}
	
\end{lemma}

\begin{proof}
 \begin{align}
	& \quad \frac{1}{N}\sum_{n=1}^{N}\mathbb{E}[\|u^{(n)}_{1, t+1} -\bar{u}_{1, t+1}\|^2] \notag \\
	& = \frac{1}{N}\sum_{n=1}^{N}\mathbb{E}[\|\Big((1-\beta_x\eta^2)(u^{(n)}_{1, t} - \nabla_1 f^{(n)}(x^{(n)}_{t}, y^{(n)}_{t}; \xi^{(n)}_{t+1})) + \nabla_1 f^{(n)}(x^{(n)}_{t+1}, y^{(n)}_{t+1}; \xi^{(n)}_{t+1})\Big) \notag \\
	& \quad- \frac{1}{N}\sum_{m=1}^{N}\Big((1-\beta_x\eta^2)(u^{(m)}_{1, t} - \nabla_1 f^{(m)}(x^{(m)}_{t}, y^{(m)}_{t}; \xi^{(m)}_{t+1})) + \nabla_1 f^{(m)}(x^{(m)}_{t+1}, y^{(m)}_{t+1}; \xi^{(m)}_{t+1})\Big) \|^2] \notag \\
	& = \frac{1}{N}\sum_{n=1}^{N}\mathbb{E}[\|(1-\beta_x\eta^2)(u^{(n)}_{1, t} - \frac{1}{N}\sum_{m=1}^{N} u^{(m)}_{1, t}) \notag \\
	& \quad- \nabla_1 f^{(n)}(x^{(n)}_{t}, y^{(n)}_{t}; \xi^{(n)}_{t+1}) + \nabla_1 f^{(n)}(x^{(n)}_{t+1}, y^{(n)}_{t+1}; \xi^{(n)}_{t+1}) + \beta_x\eta^2 \nabla_1 f^{(n)}(x^{(n)}_{t}, y^{(n)}_{t}; \xi^{(n)}_{t+1}) \notag \\
	& \quad- \frac{1}{N}\sum_{m=1}^{N}\Big( - \nabla_1 f^{(m)}(x^{(m)}_{t}, y^{(m)}_{t}; \xi^{(m)}_{t+1}) + \nabla_1 f^{(m)}(x^{(m)}_{t+1}, y^{(m)}_{t+1}; \xi^{(m)}_{t+1}) + \beta_x\eta^2 \nabla_1 f^{(m)}(x^{(m)}_{t}, y^{(m)}_{t}; \xi^{(m)}_{t+1})\Big) \|^2] \notag \\
	& \leq (1-\beta_x\eta^2)^2(1+c)\frac{1}{N}\sum_{n=1}^{N}\mathbb{E}[\|u^{(n)}_{1, t} - \frac{1}{N}\sum_{m=1}^{N} u^{(m)}_{1, t}\|^2] \notag \\
	& \quad+ (1+c^{-1})\frac{1}{N}\sum_{n=1}^{N}\mathbb{E}[\|- \nabla_1 f^{(n)}(x^{(n)}_{t}, y^{(n)}_{t}; \xi^{(n)}_{t+1}) + \nabla_1 f^{(n)}(x^{(n)}_{t+1}, y^{(n)}_{t+1}; \xi^{(n)}_{t+1}) + \beta_x\eta^2 \nabla_1 f^{(n)}(x^{(n)}_{t}, y^{(n)}_{t}; \xi^{(n)}_{t+1}) \notag \\
	& \quad- \frac{1}{N}\sum_{m=1}^{N}\Big( - \nabla_1 f^{(m)}(x^{(m)}_{t}, y^{(m)}_{t}; \xi^{(m)}_{t+1}) + \nabla_1 f^{(m)}(x^{(m)}_{t+1}, y^{(m)}_{t+1}; \xi^{(m)}_{t+1}) + \beta_x\eta^2 \nabla_1 f^{(m)}(x^{(m)}_{t}, y^{(m)}_{t}; \xi^{(m)}_{t+1})\Big) \|^2] \notag \\
	& \leq (1-\beta_x\eta^2)^2(1+c)\frac{1}{N}\sum_{n=1}^{N}\mathbb{E}[\|u^{(n)}_{1, t} - \frac{1}{N}\sum_{m=1}^{N} u^{(m)}_{1, t}\|^2] \notag \\
	& \quad+ 2(1+c^{-1})\frac{1}{N}\sum_{n=1}^{N}\mathbb{E}[\|- \nabla_1 f^{(n)}(x^{(n)}_{t}, y^{(n)}_{t}; \xi^{(n)}_{t+1}) + \nabla_1 f^{(n)}(x^{(n)}_{t+1}, y^{(n)}_{t+1}; \xi^{(n)}_{t+1}) \notag \\
	& \quad- \frac{1}{N}\sum_{m=1}^{N}\Big( - \nabla_1 f^{(m)}(x^{(m)}_{t}, y^{(m)}_{t}; \xi^{(m)}_{t+1}) + \nabla_1 f^{(m)}(x^{(m)}_{t+1}, y^{(m)}_{t+1}; \xi^{(m)}_{t+1}) \Big)\|^2] \notag \\
	& \quad+ 2(1+c^{-1}) \frac{1}{N}\sum_{n=1}^{N}\mathbb{E}[\| \beta_x\eta^2 \nabla_1 f^{(n)}(x^{(n)}_{t}, y^{(n)}_{t}; \xi^{(n)}_{t+1}) - \frac{1}{N}\sum_{m=1}^{N}\beta_x\eta^2 \nabla_1 f^{(m)}(x^{(m)}_{t}, y^{(m)}_{t}; \xi^{(m)}_{t+1}) \|^2] \notag \\
	& \leq (1-\beta_x\eta^2)^2(1+c)\frac{1}{N}\sum_{n=1}^{N}\mathbb{E}[\|u^{(n)}_{1, t} - \frac{1}{N}\sum_{m=1}^{N} u^{(m)}_{1, t}\|^2] \notag \\
	& \quad+ 2(1+c^{-1})\frac{1}{N}\sum_{n=1}^{N}\mathbb{E}[\|- \nabla_1 f^{(n)}(x^{(n)}_{t}, y^{(n)}_{t}; \xi^{(n)}_{t+1}) + \nabla_1 f^{(n)}(x^{(n)}_{t+1}, y^{(n)}_{t+1}; \xi^{(n)}_{t+1}) \|^2] \notag \\
	& \quad+ 2(1+c^{-1})\beta_x^2\eta^4 \frac{1}{N}\sum_{n=1}^{N}\mathbb{E}[\|  \nabla_1 f^{(n)}(x^{(n)}_{t}, y^{(n)}_{t}; \xi^{(n)}_{t+1}) - \frac{1}{N}\sum_{m=1}^{N} \nabla_1 f^{(m)}(x^{(m)}_{t}, y^{(m)}_{t}; \xi^{(m)}_{t+1}) \|^2] \notag \\
	& \leq (1-\beta_x\eta^2)^2(1+c)\frac{1}{N}\sum_{n=1}^{N}\mathbb{E}[\|u^{(n)}_{1, t} - \frac{1}{N}\sum_{m=1}^{N} u^{(m)}_{1, t}\|^2] \notag \\
	& \quad+ 2(1+c^{-1})L_{f}^2\frac{1}{N}\sum_{n=1}^{N}\mathbb{E}[\| x^{(n)}_{t} - x^{(n)}_{t+1} \|^2 + 2(1+c^{-1})L_{f}^2\frac{1}{N}\sum_{n=1}^{N}\mathbb{E}[\| y^{(n)}_{t} - y^{(n)}_{t+1} \|^2 \notag \\
	& \quad+ 2(1+c^{-1})\beta_x^2\eta^4 \frac{1}{N}\sum_{n=1}^{N}\mathbb{E}[\|  \nabla_1 f^{(n)}(x^{(n)}_{t}, y^{(n)}_{t}; \xi^{(n)}_{t+1}) - \frac{1}{N}\sum_{m=1}^{N} \nabla_1 f^{(m)}(x^{(m)}_{t}, y^{(m)}_{t}; \xi^{(m)}_{t+1}) \|^2] \notag \\
	& \leq (1-\beta_x\eta^2)^2(1+c)\frac{1}{N}\sum_{n=1}^{N}\mathbb{E}[\|u^{(n)}_{1, t} - \bar{u}_{1, t}\|^2] \notag \\
	& \quad+ 12(1+c^{-1})L_{f}^2\alpha_x^2\eta^2 \frac{1}{N}\sum_{n=1}^{N}\mathbb{E}[\| u^{(n)}_{1, t} - \bar{u}_{1, t} \|^2] + 12(1+c^{-1})L_{f}^2\alpha_x^2\eta^2\lambda^2 \frac{1}{N}\sum_{n=1}^{N}\mathbb{E}[\| u^{(n)}_{2, t} - \bar{u}_{2, t} \|^2]  \notag \\
	& \quad + 12(1+c^{-1})L_{f}^2\alpha_x^2\eta^2\lambda^2 \frac{1}{N}\sum_{n=1}^{N}\mathbb{E}[\| u^{(n)}_{3, t} - \bar{u}_{3, t} \|^2] + 12(1+c^{-1})L_{f}^2\alpha_x^2\eta^2 \mathbb{E}[\| \bar{m}_{x, t} \|^2]  \notag \\
	& \quad + 8(1+c^{-1})L_{f}^2\alpha_y^2\eta^2 \frac{1}{N}\sum_{n=1}^{N}\mathbb{E}[\| v^{(n)}_{1, t} - \bar{v}_{1, t} \|^2] + 8(1+c^{-1})L_{f}^2\alpha_y^2\eta^2\lambda^2 \frac{1}{N}\sum_{n=1}^{N}\mathbb{E}[\| v^{(n)}_{2, t} - \bar{v}_{2, t} \|^2] \notag \\
	& \quad + 8(1+c^{-1})L_{f}^2\alpha_y^2\eta^2 \mathbb{E}[\| \bar{m}_{y, t} \|^2]   \notag \\
	& \quad+ 36(1+c^{-1})\beta_x^2\eta^4 L_{f}^{2} \frac{1}{N}\sum_{n=1}^{N}\mathbb{E}[\| x^{(n)}_{t} - \bar{x}_{t} \|^2] + 36(1+c^{-1})\beta_x^2\eta^4 L_{f}^{2} \frac{1}{N}\sum_{n=1}^{N}\mathbb{E}[\| y^{(n)}_{t} - \bar{y}_{t} \|^2] \notag \\
	& \quad + 18(1+c^{-1})\beta_x^2\eta^4 \delta^2 + 12(1+c^{-1})\beta_x^2\eta^4  \sigma^{2} \notag \\
	& \leq (1+\frac{1}{p})\frac{1}{N}\sum_{n=1}^{N}\mathbb{E}[\|u^{(n)}_{1, t} - \bar{u}_{1, t}\|^2]  + 24pL_{f}^2\alpha_x^2\eta^2 \frac{1}{N}\sum_{n=1}^{N}\mathbb{E}[\| u^{(n)}_{1, t} - \bar{u}_{1, t} \|^2] + 24pL_{f}^2\alpha_x^2\eta^2\lambda^2 \frac{1}{N}\sum_{n=1}^{N}\mathbb{E}[\| u^{(n)}_{2, t} - \bar{u}_{2, t} \|^2]  \notag \\
	& \quad + 24pL_{f}^2\alpha_x^2\eta^2\lambda^2 \frac{1}{N}\sum_{n=1}^{N}\mathbb{E}[\| u^{(n)}_{3, t} - \bar{u}_{3, t} \|^2] + 16pL_{f}^2\alpha_y^2\eta^2 \frac{1}{N}\sum_{n=1}^{N}\mathbb{E}[\| v^{(n)}_{1, t} - \bar{v}_{1, t} \|^2] + 16pL_{f}^2\alpha_y^2\eta^2\lambda^2 \frac{1}{N}\sum_{n=1}^{N}\mathbb{E}[\| v^{(n)}_{2, t} - \bar{v}_{2, t} \|^2]  \notag \\
	& \quad+ 72p\beta_x^2\eta^4 L_{f}^{2} \frac{1}{N}\sum_{n=1}^{N}\mathbb{E}[\| x^{(n)}_{t} - \bar{x}_{t} \|^2] + 72p\beta_x^2\eta^4 L_{f}^{2} \frac{1}{N}\sum_{n=1}^{N}\mathbb{E}[\| y^{(n)}_{t} - \bar{y}_{t} \|^2] \notag \\
		& \quad + 24pL_{f}^2\alpha_x^2\eta^2 \mathbb{E}[\| \bar{m}_{x, t} \|^2]   + 16pL_{f}^2\alpha_y^2\eta^2 \mathbb{E}[\| \bar{m}_{y, t} \|^2]    + 36p\beta_x^2\eta^4 \delta^2 + 24p\beta_x^2\eta^4  \sigma^{2} \ ,
\end{align}
where the second to last step follows from Lemma~\ref{lemma:incremental-x-y-z} and Lemma~\ref{lemma:grad-diff}, the last step follows from $\eta\leq\frac{1}{\sqrt{\beta_x}}$, $c=\frac{1}{p}$ and $p>1$. 

The other inequalities can be proved in the same way. Therefore, we omit their detailed proof.

\end{proof}

\begin{lemma} \label{lemma:consensus-error-variables}
    Given Assumptions~\ref{assumption:f-smooth}-\ref{assumption:heterogeneous}, by denoting $s_{t}=\lfloor t/p\rfloor$, we have the following inequalities:
    \begin{align}
		  \frac{1}{N}\sum_{n=1}^{N}\mathbb{E}[\| x^{(n)}_{t} - \bar{x}_{t} \|^2] 
		& \leq 3 p\alpha_x^2\eta^2 \sum_{t'=s_{t} p}^{t-1} \frac{1}{N}\sum_{n=1}^{N}\mathbb{E}[\| u^{(n)}_{1, t'} -\bar{u}_{1, t'} \|^2] + 3 p\alpha_x^2\eta^2\lambda^2 \sum_{t'=s_{t} p}^{t-1} \frac{1}{N}\sum_{n=1}^{N}\mathbb{E}[\| u^{(n)}_{2, t'} -\bar{u}_{2, t'} \|^2] \notag \\
        & \quad + 3 p\alpha_x^2\eta^2\lambda^2 \sum_{t'=s_{t} p}^{t-1} \frac{1}{N}\sum_{n=1}^{N}\mathbb{E}[\| u^{(n)}_{3, t'} -\bar{u}_{3, t'} \|^2] \ , 
	\end{align}
    \begin{align}
		  \frac{1}{N}\sum_{n=1}^{N}\mathbb{E}[\| y^{(n)}_{t} - \bar{y}_{t} \|^2]  \leq 2p\alpha_y^2\eta^2 \sum_{t'=s_{t} p}^{t-1} \frac{1}{N}\sum_{n=1}^{N}\mathbb{E}[\| v^{(n)}_{1, t'} -\bar{v}_{1, t'} \|^2] + 2 p\alpha_y^2\eta^2\lambda^2 \sum_{t'=s_{t} p}^{t-1} \frac{1}{N}\sum_{n=1}^{N}\mathbb{E}[\| v^{(n)}_{2, t'} -\bar{v}_{2, t'} \|^2] \ , 
	\end{align}
    \begin{align}
		 \frac{1}{N}\sum_{n=1}^{N}\mathbb{E}[\| z^{(n)}_{t} - \bar{z}_{t} \|^2]  \leq  p\lambda^2\alpha_z^2\eta^2 \sum_{t'=s_{t} p}^{t-1} \frac{1}{N}\sum_{n=1}^{N}\mathbb{E}[\|   w^{(n)}_{1, t'}   -  \bar{w}_{1, t'}\|^2] \ . 
	\end{align}
\end{lemma}

\begin{proof}
By denoting $s_{t}=\lfloor t/p\rfloor$, we have 
    \begin{align}
		& \quad \frac{1}{N}\sum_{n=1}^{N}\mathbb{E}[\| x^{(n)}_{t} - \bar{x}_{t} \|^2] \notag \\
		& = \frac{1}{N}\sum_{n=1}^{N}\mathbb{E}[\| x^{(n)}_{s_{t} p} - \alpha_x\eta\sum_{t'=s_{t} p}^{t-1}m^{(n)}_{x, t'} - \bar{x}_{s_{t} p} + \alpha_x\eta\sum_{t'=s_{t} p}^{t-1}\bar{m}_{x, t'} \|^2] \notag \\
		& = \frac{1}{N}\sum_{n=1}^{N}\mathbb{E}[\|  - \alpha_x\eta\sum_{t'=s_{t} p}^{t-1}m^{(n)}_{x, t'}  + \alpha_x\eta\sum_{t'=s_{t} p}^{t-1}\bar{m}_{x, t'} \|^2] \notag \\
		& \leq p\alpha_x^2\eta^2 \sum_{t'=s_{t} p}^{t-1} \frac{1}{N}\sum_{n=1}^{N}\mathbb{E}[\|  - m^{(n)}_{x, t'}  + \bar{m}_{x, t'} \|^2] \notag \\
		& = p\alpha_x^2\eta^2 \sum_{t'=s_{t} p}^{t-1} \frac{1}{N}\sum_{n=1}^{N}\mathbb{E}[\|  u^{(n)}_{1, t'}  +\lambda (u^{(n)}_{2, t'} - u^{(n)}_{3, t'})  -\bar{u}_{1, t'}  - \lambda (\bar{u}_{2, t'} -\bar{u}_{3, t'}) \|^2] \notag \\
		& \leq 3 p\alpha_x^2\eta^2 \sum_{t'=s_{t} p}^{t-1} \frac{1}{N}\sum_{n=1}^{N}\mathbb{E}[\| u^{(n)}_{1, t'} -\bar{u}_{1, t'} \|^2] + 3 p\alpha_x^2\eta^2\lambda^2 \sum_{t'=s_{t} p}^{t-1} \frac{1}{N}\sum_{n=1}^{N}\mathbb{E}[\| u^{(n)}_{2, t'} -\bar{u}_{2, t'} \|^2] \notag \\
        & \quad + 3 p\alpha_x^2\eta^2\lambda^2 \sum_{t'=s_{t} p}^{t-1} \frac{1}{N}\sum_{n=1}^{N}\mathbb{E}[\| u^{(n)}_{3, t'} -\bar{u}_{3, t'} \|^2] \ . 
	\end{align}
    The other two inequalities can be proved in the same way. Therefore, we omit their detailed proof.

\end{proof}

\begin{lemma}
	(Restatement of Lemma~\ref{lemma:consensus-error-main-text}) Given Assumptions~\ref{assumption:f-smooth}-\ref{assumption:heterogeneous}, and 
    \begin{align}
     & \eta \leq \frac{1}{500 p \sqrt{L_{f}^2+ L_{g, 1}^2 \lambda^2} } \  ,  \beta_x\leq \frac{{L_f^2+ L_{g, 1}^{2}  \lambda^2}}{N} \ , \beta_y\leq \frac{{L_f^2+ L_{g, 1}^{2}  \lambda^2}}{N} \ ,  \beta_z\leq \frac{{L_f^2+ L_{g, 1}^{2}  \lambda^2}}{N} \ , \notag \\
    & \alpha_x\leq 10 \ ,  \alpha_y\leq 10 \ ,  \alpha_z\leq 10  \ , 
\end{align}
we have
\begin{align}
	& \quad \frac{1}{T}\sum_{t=0}^{T-1}\frac{1}{N}\sum_{n=1}^{N}\mathbb{E}[\|u^{(n)}_{1, t} -\bar{u}_{1, t}\|^2]  +  \lambda^2\frac{1}{T}\sum_{t=0}^{T-1} \frac{1}{N}\sum_{n=1}^{N}\mathbb{E}[\|u^{(n)}_{2, t} -\bar{u}_{2, t}\|^2]  + \lambda^2 \frac{1}{T}\sum_{t=0}^{T-1}\frac{1}{N}\sum_{n=1}^{N}\mathbb{E}[\|u^{(n)}_{3, t} -\bar{u}_{3, t}\|^2]\notag \\
	& \quad + \frac{1}{T}\sum_{t=0}^{T-1}\frac{1}{N}\sum_{n=1}^{N}\mathbb{E}[\|v^{(n)}_{1, t} -\bar{v}_{1, t}\|^2]  + \lambda^2\frac{1}{T}\sum_{t=0}^{T-1}\frac{1}{N}\sum_{n=1}^{N}\mathbb{E}[\|v^{(n)}_{2, t} -\bar{v}_{2, t}\|^2]  + \lambda^2\frac{1}{T}\sum_{t=0}^{T-1}\frac{1}{N}\sum_{n=1}^{N}\mathbb{E}[\|w^{(n)}_{1, t} -\bar{w}_{1, t}\|^2]\notag \\
	& \leq  576p^2\alpha_x^2\eta^2\Big(L_{f}^2+ L_{g, 1}^2  \lambda^2 \Big) \frac{1}{T}\sum_{t=0}^{T-1}\mathbb{E}[\| \bar{m}_{x, t} \|^2]   + 576p^2\alpha_y^2\eta^2 \Big(L_{f}^2+ L_{g, 1}^2\lambda^2\Big)\frac{1}{T}\sum_{t=0}^{T-1} \mathbb{E}[\| \bar{m}_{y, t} \|^2] \notag \\
	& \quad  + 576 p^2\alpha_z^2\eta^2 \Big(L_{f}^2+ L_{g, 1}^2\lambda^2\Big)\frac{1}{T}\sum_{t=0}^{T-1}\mathbb{E}[\|   \bar{m}_{z, t}  \|^2]   +576p^2\beta_x^2\eta^4( 1+  \lambda^2 )\delta^2   +576p^2\beta_y^2\eta^4(  1+  \lambda^2 )   \delta^2 + 576p^2\beta_z^2\eta^4 \lambda^2\delta^2 \notag \\
	& \quad + 576p^2\beta_x^2 \eta^4(  1+ \lambda^2)\sigma^{2} + 576p^2\beta_y^2\eta^4(  1 + \lambda^2  )\sigma^{2}+576p^2\beta_z^2\eta^4\lambda^2\sigma^{2} \ . 
\end{align}
\end{lemma}

\begin{proof}
\begin{align}
	& \quad \frac{1}{N}\sum_{n=1}^{N}\mathbb{E}[\|u^{(n)}_{1, t+1} -\bar{u}_{1, t+1}\|^2]  +  \lambda^2 \frac{1}{N}\sum_{n=1}^{N}\mathbb{E}[\|u^{(n)}_{2, t+1} -\bar{u}_{2, t+1}\|^2]  + \lambda^2 \frac{1}{N}\sum_{n=1}^{N}\mathbb{E}[\|u^{(n)}_{3, t+1} -\bar{u}_{3, t+1}\|^2]\notag \\
	& \quad + \frac{1}{N}\sum_{n=1}^{N}\mathbb{E}[\|v^{(n)}_{1, t+1} -\bar{v}_{1, t+1}\|^2]  + \lambda^2\frac{1}{N}\sum_{n=1}^{N}\mathbb{E}[\|v^{(n)}_{2, t+1} -\bar{v}_{2, t+1}\|^2]  + \lambda^2\frac{1}{N}\sum_{n=1}^{N}\mathbb{E}[\|w^{(n)}_{1, t+1} -\bar{w}_{1, t+1}\|^2]\notag \\
& \leq \Big(1+\frac{1}{p}+48pL_{f}^2\alpha_x^2\eta^2+ 96pL_{g, 1}^2\alpha_x^2\eta^2  \lambda^2\Big)\frac{1}{N}\sum_{n=1}^{N}\mathbb{E}[\|u^{(n)}_{1, t} - \bar{u}_{1, t}\|^2]  \notag \\
& \quad +  \Big(1+\frac{1}{p}+48pL_{f}^2\alpha_x^2\eta^2+96pL_{g, 1}^2\alpha_x^2\eta^2\lambda^2 \Big)\lambda^2 \frac{1}{N}\sum_{n=1}^{N}\mathbb{E}[\| u^{(n)}_{2, t} - \bar{u}_{2, t} \|^2]  \notag \\
& \quad + \Big(1+\frac{1}{p}+48pL_{f}^2\alpha_x^2\eta^2+ 96pL_{g, 1}^2\alpha_x^2\eta^2\lambda^2\Big) \lambda^2\frac{1}{N}\sum_{n=1}^{N} \mathbb{E}[\| u^{(n)}_{3, t} - \bar{u}_{3, t} \|^2] \notag \\
& \quad + \Big(1+\frac{1}{p}+32pL_{f}^2\alpha_y^2\eta^2+ 32pL_{g, 1}^2\alpha_y^2\eta^2  \lambda^2 \Big) \frac{1}{N}\sum_{n=1}^{N}\mathbb{E}[\| v^{(n)}_{1, t} - \bar{v}_{1, t} \|^2]  \notag \\
& \quad + \Big( 1+\frac{1}{p}+32pL_{f}^2\alpha_y^2\eta^2+32pL_{g, 1}^2\alpha_y^2\eta^2\lambda^2 \Big)\lambda^2 \frac{1}{N}\sum_{n=1}^{N}\mathbb{E}[\| v^{(n)}_{2, t} - \bar{v}_{2, t} \|^2]  \notag \\
& \quad    + \Big(1+\frac{1}{p}+ 16p L_{g, 1}^2\lambda^2\alpha_z^2\eta^2 \Big)\lambda^2\frac{1}{N}\sum_{n=1}^{N}\mathbb{E}[\|    w^{(n)}_{1, t} -  \bar{w}_{1, t}   \|^2] \notag \\
& \quad+\Big( 72p\beta_x^2\eta^4 L_{f}^{2} + 72p\beta_y^2\eta^4 L_{f}^{2}+ 144p\beta_x^2\eta^4 L_{g, 1}^{2}  \lambda^2 + 72p\beta_y^2\eta^4 L_{g, 1}^{2} \lambda^2+ 72p\beta_z^2\eta^4L_{g, 1}^{2} \lambda^2\Big)\frac{1}{N}\sum_{n=1}^{N}\mathbb{E}[\| x^{(n)}_{t} - \bar{x}_{t} \|^2]\notag \\
& \quad  +\Big( 72p\beta_x^2\eta^4 L_{f}^{2} + 72p\beta_y^2\eta^4 L_{f}^{2}+ 72p\beta_x^2\eta^4 L_{g, 1}^{2}  \lambda^2+ 72p\beta_y^2\eta^4 L_{g, 1}^{2}  \lambda^2\Big) \frac{1}{N}\sum_{n=1}^{N}\mathbb{E}[\| y^{(n)}_{t} - \bar{y}_{t} \|^2] \notag \\
& \quad   + \Big(72p\beta_x^2\eta^4L_{g, 1}^{2} \lambda^2+ 72p\beta_z^2\eta^4L_{g, 1}^{2} \lambda^2\Big)\frac{1}{N}\sum_{n=1}^{N}\mathbb{E}[\| z^{(n)}_{t} - \bar{z}_{t} \|^2] \notag \\
& \quad + \Big(48pL_{f}^2\alpha_x^2\eta^2+ 96pL_{g, 1}^2\alpha_x^2\eta^2  \lambda^2 \Big) \mathbb{E}[\| \bar{m}_{x, t} \|^2] + \Big(32pL_{f}^2\alpha_y^2\eta^2+ 32pL_{g, 1}^2\alpha_y^2\eta^2 \lambda^2\Big) \mathbb{E}[\| \bar{m}_{y, t} \|^2] \notag \\
& \quad   + 16p L_{g, 1}^2\lambda^2\alpha_z^2\eta^2 \mathbb{E}[\|   \bar{m}_{z, t}  \|^2]   + 36p\beta_x^2\eta^4 \delta^2+ 72p \lambda^2\beta_x^2\eta^4 \delta^2  + 36p\beta_y^2\eta^4 \delta^2 + 36p \lambda^2\beta_y^2\eta^4 \delta^2  + 36p\lambda^2\beta_z^2\eta^4\delta^2\notag \\
& \quad + 24p\beta_x^2\eta^4  \sigma^{2} + 48p \lambda^2\beta_x^2\eta^4  \sigma^{2} + 24p\beta_y^2\eta^4  \sigma^{2} + 24p \lambda^2\beta_y^2\eta^4  \sigma^{2} +24p\lambda^2\beta_z^2\eta^4\sigma^{2} \notag \\
& \leq \Big(1+\frac{26}{25p}\Big)\frac{1}{N}\sum_{n=1}^{N}\Big(\mathbb{E}[\|u^{(n)}_{1, t} - \bar{u}_{1, t}\|^2]+\lambda^2\mathbb{E}[\| u^{(n)}_{2, t} - \bar{u}_{2, t} \|^2]+ \lambda^2\mathbb{E}[\| u^{(n)}_{3, t} - \bar{u}_{3, t} \|^2]\notag \\
& \quad \quad +\mathbb{E}[\| v^{(n)}_{1, t} - \bar{v}_{1, t} \|^2]+\lambda^2\mathbb{E}[\| v^{(n)}_{2, t} - \bar{v}_{2, t} \|^2] +\lambda^2\mathbb{E}[\|    w^{(n)}_{1, t} -  \bar{w}_{1, t}   \|^2]\Big) \notag \\
& \quad+\Big( 72p\beta_x^2\eta^4 L_{f}^{2} + 72p\beta_y^2\eta^4 L_{f}^{2}+ 144p\beta_x^2\eta^4 L_{g, 1}^{2}  \lambda^2 + 72p\beta_y^2\eta^4 L_{g, 1}^{2} \lambda^2+ 72p\beta_z^2\eta^4L_{g, 1}^{2} \lambda^2\Big)\frac{1}{N}\sum_{n=1}^{N}\mathbb{E}[\| x^{(n)}_{t} - \bar{x}_{t} \|^2]\notag \\
& \quad  +\Big( 72p\beta_x^2\eta^4 L_{f}^{2} + 72p\beta_y^2\eta^4 L_{f}^{2}+ 72p\beta_x^2\eta^4 L_{g, 1}^{2}  \lambda^2+ 72p\beta_y^2\eta^4 L_{g, 1}^{2}  \lambda^2\Big) \frac{1}{N}\sum_{n=1}^{N}\mathbb{E}[\| y^{(n)}_{t} - \bar{y}_{t} \|^2] \notag \\
& \quad   + \Big(72p\beta_x^2\eta^4L_{g, 1}^{2} \lambda^2+ 72p\beta_z^2\eta^4L_{g, 1}^{2} \lambda^2\Big)\frac{1}{N}\sum_{n=1}^{N}\mathbb{E}[\| z^{(n)}_{t} - \bar{z}_{t} \|^2] \notag \\
& \quad + \Big(48pL_{f}^2\alpha_x^2\eta^2+ 96pL_{g, 1}^2\alpha_x^2\eta^2  \lambda^2 \Big) \mathbb{E}[\| \bar{m}_{x, t} \|^2]+ \Big(32pL_{f}^2\alpha_y^2\eta^2+ 32pL_{g, 1}^2\alpha_y^2\eta^2 \lambda^2\Big) \mathbb{E}[\| \bar{m}_{y, t} \|^2] \notag \\
& \quad   + 16p L_{g, 1}^2\lambda^2\alpha_z^2\eta^2 \mathbb{E}[\|   \bar{m}_{z, t}  \|^2]  + 36p\beta_x^2\eta^4 \delta^2+ 72p \lambda^2\beta_x^2\eta^4 \delta^2  + 36p\beta_y^2\eta^4 \delta^2 + 36p \lambda^2\beta_y^2\eta^4 \delta^2  + 36p\lambda^2\beta_z^2\eta^4\delta^2\notag \\
& \quad + 24p\beta_x^2\eta^4  \sigma^{2} + 48p \lambda^2\beta_x^2\eta^4  \sigma^{2} + 24p\beta_y^2\eta^4  \sigma^{2} + 24p \lambda^2\beta_y^2\eta^4  \sigma^{2} +24p\lambda^2\beta_z^2\eta^4\sigma^{2} \notag \\
& \leq \Big(1+\frac{26}{25p}\Big)\frac{1}{N}\sum_{n=1}^{N}\Big(\mathbb{E}[\|u^{(n)}_{1, t} - \bar{u}_{1, t}\|^2]+\lambda^2\mathbb{E}[\| u^{(n)}_{2, t} - \bar{u}_{2, t} \|^2]+ \lambda^2\mathbb{E}[\| u^{(n)}_{3, t} - \bar{u}_{3, t} \|^2]\notag \\
& \quad \quad +\mathbb{E}[\| v^{(n)}_{1, t} - \bar{v}_{1, t} \|^2]+\lambda^2\mathbb{E}[\| v^{(n)}_{2, t} - \bar{v}_{2, t} \|^2] +\lambda^2\mathbb{E}[\|    w^{(n)}_{1, t} -  \bar{w}_{1, t}   \|^2]\Big) \notag \\
& \quad+3 p\alpha_x^2\eta^2\Big( 72p\beta_x^2\eta^4 L_{f}^{2} + 72p\beta_y^2\eta^4 L_{f}^{2}+ 144p\beta_x^2\eta^4 L_{g, 1}^{2}  \lambda^2 \notag \\
& \quad \quad + 72p\beta_y^2\eta^4 L_{g, 1}^{2} \lambda^2+ 72p\beta_z^2\eta^4L_{g, 1}^{2} \lambda^2\Big) \sum_{t'=s_{t} p}^{t-1} \frac{1}{N}\sum_{n=1}^{N}\mathbb{E}[\| u^{(n)}_{1, t'} -\bar{u}_{1, t'} \|^2] \notag \\
& \quad + 3 p\alpha_x^2\eta^2\lambda^2\Big( 72p\beta_x^2\eta^4 L_{f}^{2} + 72p\beta_y^2\eta^4 L_{f}^{2}+ 144p\beta_x^2\eta^4 L_{g, 1}^{2}  \lambda^2 \notag \\
& \quad \quad + 72p\beta_y^2\eta^4 L_{g, 1}^{2} \lambda^2+ 72p\beta_z^2\eta^4L_{g, 1}^{2} \lambda^2\Big) \sum_{t'=s_{t} p}^{t-1} \frac{1}{N}\sum_{n=1}^{N}\mathbb{E}[\| u^{(n)}_{2, t'} -\bar{u}_{2, t'} \|^2] \notag \\
& \quad + 3 p\alpha_x^2\eta^2\lambda^2\Big( 72p\beta_x^2\eta^4 L_{f}^{2} + 72p\beta_y^2\eta^4 L_{f}^{2}+ 144p\beta_x^2\eta^4 L_{g, 1}^{2}  \lambda^2 \notag \\
& \quad \quad + 72p\beta_y^2\eta^4 L_{g, 1}^{2} \lambda^2+ 72p\beta_z^2\eta^4L_{g, 1}^{2} \lambda^2\Big) \sum_{t'=s_{t} p}^{t-1} \frac{1}{N}\sum_{n=1}^{N}\mathbb{E}[\| u^{(n)}_{3, t'} -\bar{u}_{3, t'} \|^2]  \notag \\
& \quad  +2p\alpha_y^2\eta^2 \Big( 72p\beta_x^2\eta^4 L_{f}^{2} + 72p\beta_y^2\eta^4 L_{f}^{2}+ 72p\beta_x^2\eta^4 L_{g, 1}^{2}  \lambda^2+ 72p\beta_y^2\eta^4 L_{g, 1}^{2}  \lambda^2\Big) \sum_{t'=s_{t} p}^{t-1} \frac{1}{N}\sum_{n=1}^{N}\mathbb{E}[\| v^{(n)}_{1, t'} -\bar{v}_{1, t'} \|^2] \notag \\
& \quad + 2 p\alpha_y^2\eta^2\lambda^2\Big( 72p\beta_x^2\eta^4 L_{f}^{2} + 72p\beta_y^2\eta^4 L_{f}^{2}+ 72p\beta_x^2\eta^4 L_{g, 1}^{2}  \lambda^2+ 72p\beta_y^2\eta^4 L_{g, 1}^{2}  \lambda^2\Big)  \sum_{t'=s_{t} p}^{t-1} \frac{1}{N}\sum_{n=1}^{N}\mathbb{E}[\| v^{(n)}_{2, t'} -\bar{v}_{2, t'} \|^2] \notag \\
& \quad   + \Big(72p\beta_x^2\eta^4L_{g, 1}^{2} \lambda^2+ 72p\beta_z^2\eta^4L_{g, 1}^{2} \lambda^2\Big)p\lambda^2\alpha_z^2\eta^2 \sum_{t'=s_{t} p}^{t-1} \frac{1}{N}\sum_{n=1}^{N}\mathbb{E}[\|   w^{(n)}_{1, t'}   -  \bar{w}_{1, t'}\|^2]\notag \\
& \quad + \Big(48pL_{f}^2\alpha_x^2\eta^2+ 96pL_{g, 1}^2\alpha_x^2\eta^2  \lambda^2 \Big) \mathbb{E}[\| \bar{m}_{x, t} \|^2]   + \Big(32pL_{f}^2\alpha_y^2\eta^2+ 32pL_{g, 1}^2\alpha_y^2\eta^2 \lambda^2\Big) \mathbb{E}[\| \bar{m}_{y, t} \|^2] \notag \\
& \quad  + 16p L_{g, 1}^2\lambda^2\alpha_z^2\eta^2 \mathbb{E}[\|   \bar{m}_{z, t}  \|^2]  + 36p\beta_x^2\eta^4 \delta^2+ 72p \lambda^2\beta_x^2\eta^4 \delta^2  + 36p\beta_y^2\eta^4 \delta^2 + 36p \lambda^2\beta_y^2\eta^4 \delta^2  + 36p\lambda^2\beta_z^2\eta^4\delta^2\notag \\
& \quad + 24p\beta_x^2\eta^4  \sigma^{2} + 48p \lambda^2\beta_x^2\eta^4  \sigma^{2} + 24p\beta_y^2\eta^4  \sigma^{2} + 24p \lambda^2\beta_y^2\eta^4  \sigma^{2} +24p\lambda^2\beta_z^2\eta^4\sigma^{2} \notag \\
& \leq 3 p\alpha_x^2\eta^2\Big( 72p\beta_x^2\eta^4 L_{f}^{2} + 72p\beta_y^2\eta^4 L_{f}^{2}+ 144p\beta_x^2\eta^4 L_{g, 1}^{2}  \lambda^2\notag \\
& \quad \quad  + 72p\beta_y^2\eta^4 L_{g, 1}^{2} \lambda^2+ 72p\beta_z^2\eta^4L_{g, 1}^{2} \lambda^2\Big) \sum_{i=s_{t} p}^{t}\left(1+\frac{26}{25p}\right)^{t-i}\sum_{t'=s_{t} p}^{i-1} \frac{1}{N}\sum_{n=1}^{N}\mathbb{E}[\| u^{(n)}_{1, t'} -\bar{u}_{1, t'} \|^2] \notag \\
& \quad + 3 p\alpha_x^2\eta^2\lambda^2\Big( 72p\beta_x^2\eta^4 L_{f}^{2} + 72p\beta_y^2\eta^4 L_{f}^{2}+ 144p\beta_x^2\eta^4 L_{g, 1}^{2}  \lambda^2 \notag \\
& \quad\quad + 72p\beta_y^2\eta^4 L_{g, 1}^{2} \lambda^2+ 72p\beta_z^2\eta^4L_{g, 1}^{2} \lambda^2\Big) \sum_{i=s_{t} p}^{t}\left(1+\frac{26}{25p}\right)^{t-i}\sum_{t'=s_{t} p}^{i-1} \frac{1}{N}\sum_{n=1}^{N}\mathbb{E}[\| u^{(n)}_{2, t'} -\bar{u}_{2, t'} \|^2] \notag \\
& \quad + 3 p\alpha_x^2\eta^2\lambda^2\Big( 72p\beta_x^2\eta^4 L_{f}^{2} + 72p\beta_y^2\eta^4 L_{f}^{2}+ 144p\beta_x^2\eta^4 L_{g, 1}^{2}  \lambda^2 \notag \\
& \quad \quad + 72p\beta_y^2\eta^4 L_{g, 1}^{2} \lambda^2+ 72p\beta_z^2\eta^4L_{g, 1}^{2} \lambda^2\Big) \sum_{i=s_{t} p}^{t}\left(1+\frac{26}{25p}\right)^{t-i}\sum_{t'=s_{t} p}^{i-1} \frac{1}{N}\sum_{n=1}^{N}\mathbb{E}[\| u^{(n)}_{3, t'} -\bar{u}_{3, t'} \|^2]  \notag \\
& \quad  +2p\alpha_y^2\eta^2 \Big( 72p\beta_x^2\eta^4 L_{f}^{2} + 72p\beta_y^2\eta^4 L_{f}^{2}+ 72p\beta_x^2\eta^4 L_{g, 1}^{2}  \lambda^2\notag \\
& \quad \quad + 72p\beta_y^2\eta^4 L_{g, 1}^{2}  \lambda^2\Big) \sum_{i=s_{t} p}^{t}\left(1+\frac{26}{25p}\right)^{t-i}\sum_{t'=s_{t} p}^{i-1} \frac{1}{N}\sum_{n=1}^{N}\mathbb{E}[\| v^{(n)}_{1, t'} -\bar{v}_{1, t'} \|^2] \notag \\
& \quad + 2 p\alpha_y^2\eta^2\lambda^2\Big( 72p\beta_x^2\eta^4 L_{f}^{2} + 72p\beta_y^2\eta^4 L_{f}^{2}+ 72p\beta_x^2\eta^4 L_{g, 1}^{2}  \lambda^2\notag \\
& \quad \quad + 72p\beta_y^2\eta^4 L_{g, 1}^{2}  \lambda^2\Big) \sum_{i=s_{t} p}^{t}\left(1+\frac{26}{25p}\right)^{t-i}\sum_{t'=s_{t} p}^{i-1} \frac{1}{N}\sum_{n=1}^{N}\mathbb{E}[\| v^{(n)}_{2, t'} -\bar{v}_{2, t'} \|^2]  \notag \\
& \quad   + \Big(72p\beta_x^2\eta^4L_{g, 1}^{2} \lambda^2+ 72p\beta_z^2\eta^4L_{g, 1}^{2} \lambda^2\Big)p\lambda^2\alpha_z^2\eta^2 \sum_{i=s_{t} p}^{t}\left(1+\frac{26}{25p}\right)^{t-i}\sum_{t'=s_{t} p}^{i-1} \frac{1}{N}\sum_{n=1}^{N}\mathbb{E}[\|   w^{(n)}_{1, t'}   -  \bar{w}_{1, t'}\|^2]\notag \\
& \quad + \Big(48pL_{f}^2\alpha_x^2\eta^2+ 96pL_{g, 1}^2\alpha_x^2\eta^2  \lambda^2 \Big) \sum_{i=s_{t} p}^{t}\left(1+\frac{26}{25p}\right)^{t-i}\mathbb{E}[\| \bar{m}_{x, i} \|^2]  \notag \\
& \quad + \Big(32pL_{f}^2\alpha_y^2\eta^2+ 32pL_{g, 1}^2\alpha_y^2\eta^2 \lambda^2\Big)\sum_{i=s_{t} p}^{t}\left(1+\frac{26}{25p}\right)^{t-i} \mathbb{E}[\| \bar{m}_{y, i} \|^2] \notag \\
& \quad + 16p L_{g, 1}^2\lambda^2\alpha_z^2\eta^2 \sum_{i=s_{t} p}^{t}\left(1+\frac{26}{25p}\right)^{t-i}\mathbb{E}[\|   \bar{m}_{z, i}  \|^2]  \notag \\
& \quad + (36p\beta_x^2\eta^4 \delta^2+ 72p \lambda^2\beta_x^2\eta^4 \delta^2  + 36p\beta_y^2\eta^4 \delta^2 + 36p \lambda^2\beta_y^2\eta^4 \delta^2  + 36p\lambda^2\beta_z^2\eta^4\delta^2)\sum_{i=s_{t} p}^{t}\left(1+\frac{26}{25p}\right)^{t-i}\notag \\
& \quad + (24p\beta_x^2\eta^4  \sigma^{2} + 48p \lambda^2\beta_x^2\eta^4  \sigma^{2} + 24p\beta_y^2\eta^4  \sigma^{2} + 24p \lambda^2\beta_y^2\eta^4  \sigma^{2} +24p\lambda^2\beta_z^2\eta^4\sigma^{2})\sum_{i=s_{t} p}^{t}\left(1+\frac{26}{25p}\right)^{t-i}\notag \\
& \leq 9 p^3\alpha_x^2\eta^2\Big( 72p\beta_x^2\eta^4 L_{f}^{2} + 72p\beta_y^2\eta^4 L_{f}^{2}+ 144p\beta_x^2\eta^4 L_{g, 1}^{2}  \lambda^2 \notag \\
& \quad \quad + 72p\beta_y^2\eta^4 L_{g, 1}^{2} \lambda^2+ 72p\beta_z^2\eta^4L_{g, 1}^{2} \lambda^2\Big)\sum_{t'=s_{t} p}^{t} \frac{1}{N}\sum_{n=1}^{N}\mathbb{E}[\| u^{(n)}_{1, t'} -\bar{u}_{1, t'} \|^2] \notag \\
& \quad + 9 p^2\alpha_x^2\eta^2\lambda^2\Big( 72p\beta_x^2\eta^4 L_{f}^{2} + 72p\beta_y^2\eta^4 L_{f}^{2}+ 144p\beta_x^2\eta^4 L_{g, 1}^{2}  \lambda^2 \notag \\
& \quad \quad + 72p\beta_y^2\eta^4 L_{g, 1}^{2} \lambda^2+ 72p\beta_z^2\eta^4L_{g, 1}^{2} \lambda^2\Big) \sum_{t'=s_{t} p}^{t} \frac{1}{N}\sum_{n=1}^{N}\mathbb{E}[\| u^{(n)}_{2, t'} -\bar{u}_{2, t'} \|^2] \notag \\
& \quad + 9 p^2\alpha_x^2\eta^2\lambda^2\Big( 72p\beta_x^2\eta^4 L_{f}^{2} + 72p\beta_y^2\eta^4 L_{f}^{2}+ 144p\beta_x^2\eta^4 L_{g, 1}^{2}  \lambda^2 \notag \\
& \quad \quad + 72p\beta_y^2\eta^4 L_{g, 1}^{2} \lambda^2+ 72p\beta_z^2\eta^4L_{g, 1}^{2} \lambda^2\Big)\sum_{t'=s_{t} p}^{t} \frac{1}{N}\sum_{n=1}^{N}\mathbb{E}[\| u^{(n)}_{3, t'} -\bar{u}_{3, t'} \|^2]  \notag \\
& \quad  +6p^2\alpha_y^2\eta^2 \Big( 72p\beta_x^2\eta^4 L_{f}^{2} + 72p\beta_y^2\eta^4 L_{f}^{2}+ 72p\beta_x^2\eta^4 L_{g, 1}^{2}  \lambda^2+ 72p\beta_y^2\eta^4 L_{g, 1}^{2}  \lambda^2\Big)\sum_{t'=s_{t} p}^{t} \frac{1}{N}\sum_{n=1}^{N}\mathbb{E}[\| v^{(n)}_{1, t'} -\bar{v}_{1, t'} \|^2] \notag \\
& \quad + 6 p^2\alpha_y^2\eta^2\lambda^2\Big( 72p\beta_x^2\eta^4 L_{f}^{2} + 72p\beta_y^2\eta^4 L_{f}^{2}+ 72p\beta_x^2\eta^4 L_{g, 1}^{2}  \lambda^2+ 72p\beta_y^2\eta^4 L_{g, 1}^{2}  \lambda^2\Big)\sum_{t'=s_{t} p}^{t} \frac{1}{N}\sum_{n=1}^{N}\mathbb{E}[\| v^{(n)}_{2, t'} -\bar{v}_{2, t'} \|^2] \notag \\
& \quad   +3p^2\lambda^2\alpha_z^2\eta^2 \Big(72p\beta_x^2\eta^4L_{g, 1}^{2} \lambda^2+ 72p\beta_z^2\eta^4L_{g, 1}^{2} \lambda^2\Big) \sum_{t'=s_{t} p}^{t} \frac{1}{N}\sum_{n=1}^{N}\mathbb{E}[\|   w^{(n)}_{1, t'}   -  \bar{w}_{1, t'}\|^2]\notag \\
& \quad + 3\Big(48pL_{f}^2\alpha_x^2\eta^2+ 96pL_{g, 1}^2\alpha_x^2\eta^2  \lambda^2 \Big) \sum_{t'=s_{t} p}^{t}\mathbb{E}[\| \bar{m}_{x, t'} \|^2]    + 3\Big(32pL_{f}^2\alpha_y^2\eta^2+ 32pL_{g, 1}^2\alpha_y^2\eta^2 \lambda^2\Big)\sum_{t'=s_{t} p}^{t} \mathbb{E}[\| \bar{m}_{y, t'} \|^2] \notag \\
& \quad    + 3\Big(16p L_{g, 1}^2\lambda^2\alpha_z^2\eta^2 \Big)\sum_{t'=s_{t} p}^{t}\mathbb{E}[\|   \bar{m}_{z, t'}  \|^2]  \notag \\
& \quad +\sum_{i=s_{t} p}^{t}3 (36p\beta_x^2\eta^4 \delta^2+ 72p \lambda^2\beta_x^2\eta^4 \delta^2  + 36p\beta_y^2\eta^4 \delta^2 + 36p \lambda^2\beta_y^2\eta^4 \delta^2  + 36p\lambda^2\beta_z^2\eta^4\delta^2)\notag \\
& \quad + \sum_{i=s_{t} p}^{t}3(24p\beta_x^2\eta^4  \sigma^{2} + 48p \lambda^2\beta_x^2\eta^4  \sigma^{2} + 24p\beta_y^2\eta^4  \sigma^{2} + 24p \lambda^2\beta_y^2\eta^4  \sigma^{2} +24p\lambda^2\beta_z^2\eta^4\sigma^{2})\ , 
\end{align}
where the first step follows from Lemma~\ref{lemma:consensus-error-momentum-each-step}, the second step follows from $\alpha_x\leq 10$, $\alpha_y<10$, $\alpha_z<10$,  $ \eta\leq \frac{1}{500 p \sqrt{L_{f}^2+ L_{g, 1}^2 \lambda^2} }$, $ \eta\leq  \frac{1}{500 p \sqrt{L_{f}^2+L_{g, 1}^2\lambda^2}}$, and $\eta\leq  \frac{1}{500 p   \sqrt{L_{f}^2+ L_{g, 1}^2 \lambda^2}}$, the third step follows from Lemma~\ref{lemma:consensus-error-variables}, the fourth step follows from periodic communication at the $s_{t} p$ iteration, the last step follows from $\left(1+\frac{26}{25p}\right)^{t-i}\leq \left(1+\frac{26}{25p}\right)^{p}\leq 3$.


As a result, we have
\begin{align}
	& \quad \frac{1}{T}\sum_{t=0}^{T-1}\frac{1}{N}\sum_{n=1}^{N}\mathbb{E}[\|u^{(n)}_{1, t} -\bar{u}_{1, t}\|^2]  +  \lambda^2\frac{1}{T}\sum_{t=0}^{T-1} \frac{1}{N}\sum_{n=1}^{N}\mathbb{E}[\|u^{(n)}_{2, t} -\bar{u}_{2, t}\|^2]  + \lambda^2 \frac{1}{T}\sum_{t=0}^{T-1}\frac{1}{N}\sum_{n=1}^{N}\mathbb{E}[\|u^{(n)}_{3, t} -\bar{u}_{3, t}\|^2]\notag \\
	& \quad + \frac{1}{T}\sum_{t=0}^{T-1}\frac{1}{N}\sum_{n=1}^{N}\mathbb{E}[\|v^{(n)}_{1, t} -\bar{v}_{1, t}\|^2]  + \lambda^2\frac{1}{T}\sum_{t=0}^{T-1}\frac{1}{N}\sum_{n=1}^{N}\mathbb{E}[\|v^{(n)}_{2, t} -\bar{v}_{2, t}\|^2]  + \lambda^2\frac{1}{T}\sum_{t=0}^{T-1}\frac{1}{N}\sum_{n=1}^{N}\mathbb{E}[\|w^{(n)}_{1, t} -\bar{w}_{1, t}\|^2]\notag \\
	& \leq 9 p^2\alpha_x^2\eta^2\Big( 72p\beta_x^2\eta^4 L_{f}^{2} + 72p\beta_y^2\eta^4 L_{f}^{2}+ 144p\beta_x^2\eta^4 L_{g, 1}^{2}  \lambda^2 \notag \\
	& \quad \quad + 72p\beta_y^2\eta^4 L_{g, 1}^{2} \lambda^2+ 72p\beta_z^2\eta^4L_{g, 1}^{2} \lambda^2\Big)\frac{1}{T}\sum_{t=0}^{T-1}\sum_{t'=s_{t} p}^{t} \frac{1}{N}\sum_{n=1}^{N}\mathbb{E}[\| u^{(n)}_{1, t'} -\bar{u}_{1, t'} \|^2] \notag \\
	& \quad + 9 p^2\alpha_x^2\eta^2\lambda^2\Big( 72p\beta_x^2\eta^4 L_{f}^{2} + 72p\beta_y^2\eta^4 L_{f}^{2}+ 144p\beta_x^2\eta^4 L_{g, 1}^{2}  \lambda^2\notag \\
	& \quad \quad  + 72p\beta_y^2\eta^4 L_{g, 1}^{2} \lambda^2+ 72p\beta_z^2\eta^4L_{g, 1}^{2} \lambda^2\Big) \frac{1}{T}\sum_{t=0}^{T-1}\sum_{t'=s_{t} p}^{t} \frac{1}{N}\sum_{n=1}^{N}\mathbb{E}[\| u^{(n)}_{2, t'} -\bar{u}_{2, t'} \|^2] \notag \\
	& \quad + 9 p^2\alpha_x^2\eta^2\lambda^2\Big( 72p\beta_x^2\eta^4 L_{f}^{2} + 72p\beta_y^2\eta^4 L_{f}^{2}+ 144p\beta_x^2\eta^4 L_{g, 1}^{2}  \lambda^2 \notag \\
	& \quad \quad + 72p\beta_y^2\eta^4 L_{g, 1}^{2} \lambda^2+ 72p\beta_z^2\eta^4L_{g, 1}^{2} \lambda^2\Big)\frac{1}{T}\sum_{t=0}^{T-1}\sum_{t'=s_{t} p}^{t} \frac{1}{N}\sum_{n=1}^{N}\mathbb{E}[\| u^{(n)}_{3, t'} -\bar{u}_{3, t'} \|^2]  \notag \\
	& \quad  +6p^2\alpha_y^2\eta^2 \Big( 72p\beta_x^2\eta^4 L_{f}^{2} + 72p\beta_y^2\eta^4 L_{f}^{2}+ 72p\beta_x^2\eta^4 L_{g, 1}^{2}  \lambda^2+ 72p\beta_y^2\eta^4 L_{g, 1}^{2}  \lambda^2\Big)\frac{1}{T}\sum_{t=0}^{T-1}\sum_{t'=s_{t} p}^{t} \frac{1}{N}\sum_{n=1}^{N}\mathbb{E}[\| v^{(n)}_{1, t'} -\bar{v}_{1, t'} \|^2] \notag \\
	& \quad + 6 p^2\alpha_y^2\eta^2\lambda^2\Big( 72p\beta_x^2\eta^4 L_{f}^{2} + 72p\beta_y^2\eta^4 L_{f}^{2}+ 72p\beta_x^2\eta^4 L_{g, 1}^{2}  \lambda^2+ 72p\beta_y^2\eta^4 L_{g, 1}^{2}  \lambda^2\Big)\frac{1}{T}\sum_{t=0}^{T-1}\sum_{t'=s_{t} p}^{t} \frac{1}{N}\sum_{n=1}^{N}\mathbb{E}[\| v^{(n)}_{2, t'} -\bar{v}_{2, t'} \|^2]  \notag \\
	& \quad   +3p^2\lambda^2\alpha_z^2\eta^2 \Big(72p\beta_x^2\eta^4L_{g, 1}^{2} \lambda^2+ 72p\beta_z^2\eta^4L_{g, 1}^{2} \lambda^2\Big) \frac{1}{T}\sum_{t=0}^{T-1}\sum_{t'=s_{t} p}^{t} \frac{1}{N}\sum_{n=1}^{N}\mathbb{E}[\|   w^{(n)}_{1, t'}   -  \bar{w}_{1, t'}\|^2]\notag \\
	& \quad + 3\Big(48pL_{f}^2\alpha_x^2\eta^2+ 96pL_{g, 1}^2\alpha_x^2\eta^2  \lambda^2 \Big) \frac{1}{T}\sum_{t=0}^{T-1}\sum_{t'=s_{t} p}^{t}\mathbb{E}[\| \bar{m}_{x, t'} \|^2] \notag \\
	& \quad + 3\Big(32pL_{f}^2\alpha_y^2\eta^2+ 32pL_{g, 1}^2\alpha_y^2\eta^2 \lambda^2\Big)\frac{1}{T}\sum_{t=0}^{T-1}\sum_{t'=s_{t} p}^{t} \mathbb{E}[\| \bar{m}_{y, t'} \|^2] \notag \\
	& \quad + 3\Big(16p L_{g, 1}^2\lambda^2\alpha_z^2\eta^2 \Big)\frac{1}{T}\sum_{t=0}^{T-1}\sum_{t'=s_{t} p}^{t}\mathbb{E}[\|   \bar{m}_{z, t'}  \|^2]  \notag \\
	& \quad +\frac{1}{T}\sum_{t=0}^{T-1}\sum_{i=s_{t} p}^{t}3 (36p\beta_x^2\eta^4 \delta^2+ 72p \lambda^2\beta_x^2\eta^4 \delta^2  + 36p\beta_y^2\eta^4 \delta^2 + 36p \lambda^2\beta_y^2\eta^4 \delta^2  + 36p\lambda^2\beta_z^2\eta^4\delta^2)\notag \\
	& \quad + \frac{1}{T}\sum_{t=0}^{T-1}\sum_{i=s_{t} p}^{t}3(24p\beta_x^2\eta^4  \sigma^{2} + 48p \lambda^2\beta_x^2\eta^4  \sigma^{2} + 24p\beta_y^2\eta^4  \sigma^{2} + 24p \lambda^2\beta_y^2\eta^4  \sigma^{2} +24p\lambda^2\beta_z^2\eta^4\sigma^{2})\notag \\
	& \leq 9 p^3\alpha_x^2\eta^2\Big( 72p\beta_x^2\eta^4 L_{f}^{2} + 72p\beta_y^2\eta^4 L_{f}^{2}+ 144p\beta_x^2\eta^4 L_{g, 1}^{2}  \lambda^2 \notag \\
	& \quad \quad + 72p\beta_y^2\eta^4 L_{g, 1}^{2} \lambda^2+ 72p\beta_z^2\eta^4L_{g, 1}^{2} \lambda^2\Big)\frac{1}{T}\sum_{t=0}^{T-1}\frac{1}{N}\sum_{n=1}^{N}\mathbb{E}[\| u^{(n)}_{1, t} -\bar{u}_{1, t} \|^2] \notag \\
	& \quad + 9 p^3\alpha_x^2\eta^2\lambda^2\Big( 72p\beta_x^2\eta^4 L_{f}^{2} + 72p\beta_y^2\eta^4 L_{f}^{2}+ 144p\beta_x^2\eta^4 L_{g, 1}^{2}  \lambda^2 \notag \\
	& \quad \quad + 72p\beta_y^2\eta^4 L_{g, 1}^{2} \lambda^2+ 72p\beta_z^2\eta^4L_{g, 1}^{2} \lambda^2\Big) \frac{1}{T}\sum_{t=0}^{T-1} \frac{1}{N}\sum_{n=1}^{N}\mathbb{E}[\| u^{(n)}_{2, t} -\bar{u}_{2, t} \|^2] \notag \\
	& \quad + 9 p^3\alpha_x^2\eta^2\lambda^2\Big( 72p\beta_x^2\eta^4 L_{f}^{2} + 72p\beta_y^2\eta^4 L_{f}^{2}+ 144p\beta_x^2\eta^4 L_{g, 1}^{2}  \lambda^2 \notag \\
	& \quad \quad + 72p\beta_y^2\eta^4 L_{g, 1}^{2} \lambda^2+ 72p\beta_z^2\eta^4L_{g, 1}^{2} \lambda^2\Big)\frac{1}{T}\sum_{t=0}^{T-1}\frac{1}{N}\sum_{n=1}^{N}\mathbb{E}[\| u^{(n)}_{3, t} -\bar{u}_{3, t} \|^2]  \notag \\
	& \quad  +6p^3\alpha_y^2\eta^2 \Big( 72p\beta_x^2\eta^4 L_{f}^{2} + 72p\beta_y^2\eta^4 L_{f}^{2}+ 72p\beta_x^2\eta^4 L_{g, 1}^{2}  \lambda^2+ 72p\beta_y^2\eta^4 L_{g, 1}^{2}  \lambda^2\Big)\frac{1}{T}\sum_{t=0}^{T-1} \frac{1}{N}\sum_{n=1}^{N}\mathbb{E}[\| v^{(n)}_{1, t} -\bar{v}_{1, t} \|^2] \notag \\
	& \quad + 6 p^3\alpha_y^2\eta^2\lambda^2\Big( 72p\beta_x^2\eta^4 L_{f}^{2} + 72p\beta_y^2\eta^4 L_{f}^{2}+ 72p\beta_x^2\eta^4 L_{g, 1}^{2}  \lambda^2+ 72p\beta_y^2\eta^4 L_{g, 1}^{2}  \lambda^2\Big)\frac{1}{T}\sum_{t=0}^{T-1} \frac{1}{N}\sum_{n=1}^{N}\mathbb{E}[\| v^{(n)}_{2, t} -\bar{v}_{2, t} \|^2]  \notag \\
	& \quad   +3p^3\lambda^2\alpha_z^2\eta^2 \Big(72p\beta_x^2\eta^4L_{g, 1}^{2} \lambda^2+ 72p\beta_z^2\eta^4L_{g, 1}^{2} \lambda^2\Big) \frac{1}{T}\sum_{t=0}^{T-1} \frac{1}{N}\sum_{n=1}^{N}\mathbb{E}[\|   w^{(n)}_{1, t}   -  \bar{w}_{1, t}\|^2]\notag \\
	& \quad + 3p\Big(48pL_{f}^2\alpha_x^2\eta^2+ 96pL_{g, 1}^2\alpha_x^2\eta^2  \lambda^2 \Big) \frac{1}{T}\sum_{t=0}^{T-1}\mathbb{E}[\| \bar{m}_{x, t} \|^2]  + 3p\Big(32pL_{f}^2\alpha_y^2\eta^2+ 32pL_{g, 1}^2\alpha_y^2\eta^2 \lambda^2\Big)\frac{1}{T}\sum_{t=0}^{T-1} \mathbb{E}[\| \bar{m}_{y, t} \|^2] \notag \\
	& \quad  + 3p\Big(16p L_{g, 1}^2\lambda^2\alpha_z^2\eta^2 \Big)\frac{1}{T}\sum_{t=0}^{T-1}\mathbb{E}[\|   \bar{m}_{z, t}  \|^2]  \notag \\
	& \quad +3p (36p\beta_x^2\eta^4 \delta^2+ 72p \lambda^2\beta_x^2\eta^4 \delta^2  + 36p\beta_y^2\eta^4 \delta^2 + 36p \lambda^2\beta_y^2\eta^4 \delta^2  + 36p\lambda^2\beta_z^2\eta^4\delta^2)\notag \\
	& \quad + 3p(24p\beta_x^2\eta^4  \sigma^{2} + 48p \lambda^2\beta_x^2\eta^4  \sigma^{2} + 24p\beta_y^2\eta^4  \sigma^{2} + 24p \lambda^2\beta_y^2\eta^4  \sigma^{2} +24p\lambda^2\beta_z^2\eta^4\sigma^{2})\ . 
\end{align}

Then, by setting $\alpha_x\leq 10$, we enforce
\begin{align}
    & \quad 1- 9 p^3\alpha_x^2\eta^2\Big( 72p\beta_x^2\eta^4 L_{f}^{2} + 144p\beta_x^2\eta^4 L_{g, 1}^{2}  \lambda^2+ 72p\beta_y^2\eta^4 L_{f}^{2}+ 72p\beta_y^2\eta^4 L_{g, 1}^{2} \lambda^2+ 72p\beta_z^2\eta^4L_{g, 1}^{2} \lambda^2\Big)  \notag \\
	& \geq  1- 1296 p^3\alpha_x^2\eta^2\Big( p\beta_x^2\eta^4 L_{f}^{2} + p\beta_x^2\eta^4 L_{g, 1}^{2}  \lambda^2+ p\beta_y^2\eta^4 L_{f}^{2}+ p\beta_y^2\eta^4 L_{g, 1}^{2} \lambda^2+ p\beta_z^2\eta^4L_{g, 1}^{2} \lambda^2\Big) \notag \\
    & \geq  1- 129600 p^3\eta^2\Big( p\beta_x^2\eta^4 L_{f}^{2} + p\beta_x^2\eta^4 L_{g, 1}^{2}  \lambda^2+ p\beta_y^2\eta^4 L_{f}^{2}+ p\beta_y^2\eta^4 L_{g, 1}^{2} \lambda^2+ p\beta_z^2\eta^4L_{g, 1}^{2} \lambda^2\Big) \notag \\
	& \geq \frac{1}{2} \ . 
\end{align}
This can be done by setting $\beta_x\leq \frac{{L_f^2+ L_{g, 1}^{2}  \lambda^2}}{N}$, $\beta_y\leq \frac{{L_f^2+ L_{g, 1}^{2}  \lambda^2}}{N}$, $\beta_z\leq \frac{{L_f^2+ L_{g, 1}^{2}  \lambda^2}}{N}$, and $ \eta\leq \frac{1}{500 p \sqrt{L_{f}^2+ L_{g, 1}^2 \lambda^2} }$ such that 
\begin{align}
    & \quad 129600 p^4\beta_x^2\eta^6(  L_{f}^{2} +  L_{g, 1}^{2}  \lambda^2) \notag \\
	& \leq 129600 p^4(  L_{f}^{2} +  L_{g, 1}^{2}  \lambda^2) \frac{1}{500^6 p^6 (L_{f}^2+ L_{g, 1}^2 \lambda^2)^3 } \frac{(L_f^2+ L_{g, 1}^{2}  \lambda^2)^2}{N^2} \notag\\
	& =  \frac{1296}{500^6 N^2p^2  } \notag\\
	& \leq   \frac{1296}{500^6   } \notag\\
	& \leq \frac{1}{6} \ , 
\end{align}
\begin{align}
    & \quad  129600 p^4\beta_y^2\eta^6( L_{f}^{2}+  L_{g, 1}^{2} \lambda^2) \notag \\
    &  \leq 129600 p^4(  L_{f}^{2} +  L_{g, 1}^{2}  \lambda^2) \frac{1}{500^6 p^6 (L_{f}^2+ L_{g, 1}^2 \lambda^2)^3 } \frac{(L_f^2+ L_{g, 1}^{2}  \lambda^2)^2}{N^2} \notag \\
    & \leq \frac{1}{6}   \ , 
\end{align}
and
\begin{align}
    & \quad 129600 p^4\beta_z^2\eta^6L_{g, 1}^{2} \lambda^2 \notag \\
    & \leq 129600 p^4\beta_z^2\eta^6(L_{f, 1}^2+L_{g, 1}^{2} \lambda^2)\notag  \\
    & \leq \frac{1}{6} \ . 
\end{align}

Similarly, by setting $\alpha_y\leq 10$, we enforce
\begin{align}
    & 1-  6p^3\alpha_y^2\eta^2 \Big( 72p\beta_x^2\eta^4 L_{f}^{2} + 72p\beta_y^2\eta^4 L_{f}^{2}+ 72p\beta_x^2\eta^4 L_{g, 1}^{2}  \lambda^2+ 72p\beta_y^2\eta^4 L_{g, 1}^{2}  \lambda^2\Big) \notag\\
    & \geq 1-  1296p^4\alpha_y^2\eta^6 \Big( \beta_x^2 L_{f}^{2} + \beta_x^2 L_{g, 1}^{2}  \lambda^2 + \beta_y^2 L_{f}^{2}+ \beta_y^2 L_{g, 1}^{2}  \lambda^2\Big) \notag \\
    & \geq 1-  129600p^4 \eta^6 \Big( \beta_x^2 L_{f}^{2} + \beta_x^2 L_{g, 1}^{2}  \lambda^2 + \beta_y^2 L_{f}^{2}+ \beta_y^2 L_{g, 1}^{2}  \lambda^2\Big) \notag \\
    & \geq \frac{1}{2} \ .
\end{align}
This can be done by setting $\beta_x\leq \frac{{L_f^2+ L_{g, 1}^{2}  \lambda^2}}{N}$, $\beta_y\leq \frac{{L_f^2+ L_{g, 1}^{2}  \lambda^2}}{N}$,  and $ \eta\leq \frac{1}{500 p \sqrt{L_{f}^2+ L_{g, 1}^2 \lambda^2} }$ such that 
\begin{align}
    & \quad  129600 p^4\beta_x^2\eta^6 (  L_{f}^{2} +  L_{g, 1}^{2}  \lambda^2 ) \notag \\
    & \leq 1296p^4\alpha_y^2\eta^6 (  L_{f}^{2} +  L_{g, 1}^{2}  \lambda^2 ) \frac{({L_f^2+ L_{g, 1}^{2}  \lambda^2})^2}{N^2}\left(\frac{1}{500 p \sqrt{L_{f}^2+ L_{g, 1}^2 \lambda^2} }\right)^6 \notag \\
    & \leq \frac{1}{6} \ , 
\end{align}
and 
\begin{align}
    & \quad  129600p^4\beta_y^2\eta^6 (  L_{f}^{2} +  L_{g, 1}^{2}  \lambda^2 )  \leq \frac{1}{6} \  .
\end{align}

Moreover, by setting $\alpha_z\leq 10$,  we enforce
\begin{align}
    & \quad 1- 3p^3\lambda^2\alpha_z^2\eta^2 \Big(72p\beta_x^2\eta^4L_{g, 1}^{2} \lambda^2+ 72p\beta_z^2\eta^4L_{g, 1}^{2} \lambda^2\Big) \notag \\
    & \geq 1- 1296 p^4\lambda^2\alpha_z^2\eta^6 \Big(\beta_x^2(L_{f, 1}^2+L_{g, 1}^{2} \lambda^2)+ \beta_z^2(L_{f, 1}^2+L_{g, 1}^{2} \lambda^2)\Big) \notag \\
    & \geq 1- 129600 p^4\lambda^2\eta^6 \Big(\beta_x^2(L_{f, 1}^2+L_{g, 1}^{2} \lambda^2)+ \beta_z^2(L_{f, 1}^2+L_{g, 1}^{2} \lambda^2)\Big) \notag \\
    & \geq \frac{1}{2} \ . 
\end{align}
This can be done by setting $\beta_x\leq \frac{{L_f^2+ L_{g, 1}^{2}  \lambda^2}}{N}$, $\beta_z\leq \frac{{L_f^2+ L_{g, 1}^{2}  \lambda^2}}{N}$,  and $ \eta\leq \frac{1}{500 p \sqrt{L_{f}^2+ L_{g, 1}^2 \lambda^2} }$ such that 
\begin{align}
    & 129600 p^4\lambda^2\eta^6 \beta_x^2(L_{f, 1}^2+L_{g, 1}^{2} \lambda^2) \leq \frac{1}{6} \ , \notag \\
    & 129600p^4\lambda^2\eta^6 \beta_z^2(L_{f, 1}^2+L_{g, 1}^{2} \lambda^2)  \leq \frac{1}{6} \  . 
\end{align}

In summary, by setting
\begin{align}
    & \eta \leq \frac{1}{500 p \sqrt{L_{f}^2+ L_{g, 1}^2 \lambda^2} }\ , \notag \\
    & \beta_x\leq \frac{{L_f^2+ L_{g, 1}^{2}  \lambda^2}}{N} \ , \beta_y\leq \frac{{L_f^2+ L_{g, 1}^{2}  \lambda^2}}{N} \ ,  \beta_z\leq \frac{{L_f^2+ L_{g, 1}^{2}  \lambda^2}}{N} \ , \notag \\
    & \alpha_x\leq 10 \ ,  \alpha_y\leq 10 \ ,  \alpha_z\leq 10  \ , 
\end{align}
we have
\begin{align}
	& \quad \frac{1}{T}\sum_{t=0}^{T-1}\frac{1}{N}\sum_{n=1}^{N}\mathbb{E}[\|u^{(n)}_{1, t} -\bar{u}_{1, t}\|^2]  +  \lambda^2\frac{1}{T}\sum_{t=0}^{T-1} \frac{1}{N}\sum_{n=1}^{N}\mathbb{E}[\|u^{(n)}_{2, t} -\bar{u}_{2, t}\|^2]  + \lambda^2 \frac{1}{T}\sum_{t=0}^{T-1}\frac{1}{N}\sum_{n=1}^{N}\mathbb{E}[\|u^{(n)}_{3, t} -\bar{u}_{3, t}\|^2]\notag \\
	& \quad + \frac{1}{T}\sum_{t=0}^{T-1}\frac{1}{N}\sum_{n=1}^{N}\mathbb{E}[\|v^{(n)}_{1, t} -\bar{v}_{1, t}\|^2]  + \lambda^2\frac{1}{T}\sum_{t=0}^{T-1}\frac{1}{N}\sum_{n=1}^{N}\mathbb{E}[\|v^{(n)}_{2, t} -\bar{v}_{2, t}\|^2]  + \lambda^2\frac{1}{T}\sum_{t=0}^{T-1}\frac{1}{N}\sum_{n=1}^{N}\mathbb{E}[\|w^{(n)}_{1, t} -\bar{w}_{1, t}\|^2]\notag \\
	& \leq  6p^2\Big(48L_{f}^2\alpha_x^2\eta^2+ 96L_{g, 1}^2\alpha_x^2\eta^2  \lambda^2 \Big) \frac{1}{T}\sum_{t=0}^{T-1}\mathbb{E}[\| \bar{m}_{x, t} \|^2]    + 6p^2\Big(32L_{f}^2\alpha_y^2\eta^2+ 32L_{g, 1}^2\alpha_y^2\eta^2 \lambda^2\Big)\frac{1}{T}\sum_{t=0}^{T-1} \mathbb{E}[\| \bar{m}_{x, t} \|^2] \notag \\
	& \quad  + 6p^2\Big(16 L_{g, 1}^2\lambda^2\alpha_z^2\eta^2 \Big)\frac{1}{T}\sum_{t=0}^{T-1}\mathbb{E}[\|   \bar{m}_{z, t}  \|^2]  \notag \\
	& \quad +6p^2 (36\beta_x^2\eta^4 \delta^2+ 72 \lambda^2\beta_x^2\eta^4 \delta^2  + 36\beta_y^2\eta^4 \delta^2 + 36 \lambda^2\beta_y^2\eta^4 \delta^2  + 36\lambda^2\beta_z^2\eta^4\delta^2)\notag \\
	& \quad + 6p^2(24\beta_x^2\eta^4  \sigma^{2} + 48 \lambda^2\beta_x^2\eta^4  \sigma^{2} + 24\beta_y^2\eta^4  \sigma^{2} + 24 \lambda^2\beta_y^2\eta^4  \sigma^{2} +24\lambda^2\beta_z^2\eta^4\sigma^{2})\notag \\
	& \leq  576p^2\alpha_x^2\eta^2\Big(L_{f}^2+ L_{g, 1}^2  \lambda^2 \Big) \frac{1}{T}\sum_{t=0}^{T-1}\mathbb{E}[\| \bar{m}_{x, t} \|^2]   + 576p^2\alpha_y^2\eta^2 \Big(L_{f}^2+ L_{g, 1}^2\lambda^2\Big)\frac{1}{T}\sum_{t=0}^{T-1} \mathbb{E}[\| \bar{m}_{y, t} \|^2] \notag \\
	& \quad  + 576 p^2\alpha_z^2\eta^2 \Big(L_{f}^2+ L_{g, 1}^2\lambda^2\Big)\frac{1}{T}\sum_{t=0}^{T-1}\mathbb{E}[\|   \bar{m}_{z, t}  \|^2]  \notag \\
	& \quad +576p^2\beta_x^2\eta^4( 1+  \lambda^2 )\delta^2   +576p^2\beta_y^2\eta^4(  1+  \lambda^2 )   \delta^2 + 576p^2\beta_z^2\eta^4 \lambda^2\delta^2 \notag \\
	& \quad + 576p^2\beta_x^2 \eta^4(  1+ \lambda^2)\sigma^{2} + 576p^2\beta_y^2\eta^4(  1 + \lambda^2  )\sigma^{2}+576p^2\beta_z^2\eta^4\lambda^2\sigma^{2} \ . 
\end{align}

\end{proof}

\subsection{Proof of Theorem~\ref{theorem:convergence-rate}}\label{sec:theorem}

Based on these fundamental lemmas, we are ready to prove the convergence rate of our Algorithm~\ref{alg_sim} in Theorem~\ref{theorem:convergence-rate}. 

\begin{proof}
	At first, we propose a novel potential function as follows:
	\begin{align}
		& P_{t+1} = \mathbb{E}[\mathcal{L}_{\lambda}^{*}(\bar{x}_{t+1})] + c_1 \lambda \mathbb{E}[\|\bar{y}_{t+1} - y_{\lambda}^{*}(\bar{{x}}_{t+1})\|^2]  + c_2 \lambda \mathbb{E}[\|\bar{z}_{t+1} - y^{*}(\bar{{x}}_{t+1})\|^2] \notag  \\
		& \quad + c_3\mathbb{E}[\|\frac{1}{N}\sum_{n=1}^{N}\nabla_{1} f^{(n)}(x^{(n)}_{t+1}, y^{(n)}_{t+1})- \frac{1}{N}\sum_{n=1}^{N} u^{(n)}_{1, t+1} \|^2]  \notag  + c_4 \lambda^2 \mathbb{E}[\|\frac{1}{N}\sum_{n=1}^{N}\nabla_{1} g^{(n)}(x^{(n)}_{t+1}, y^{(n)}_{t+1})- \frac{1}{N}\sum_{n=1}^{N} u^{(n)}_{2, t+1} \|^2] \notag  \\
		& \quad + c_5\lambda^2\mathbb{E}[\|\frac{1}{N}\sum_{n=1}^{N}\nabla_{1} g^{(n)}(x^{(n)}_{t+1}, z^{(n)}_{t+1})- \frac{1}{N}\sum_{n=1}^{N} u^{(n)}_{3, t+1} \|^2] \notag   + c_6\mathbb{E}[\|\frac{1}{N}\sum_{n=1}^{N} \nabla_{2} f^{(n)}({x}_{t+1}^{(n)}, {y}_{t+1}^{(n)}) -\frac{1}{N}\sum_{n=1}^{N} v^{(n)}_{1, t+1}  \|^2] \notag  \\
		& \quad + c_7\lambda^2 \mathbb{E}[\|\frac{1}{N}\sum_{n=1}^{N}  \nabla_{2} g^{(n)}({x}_{t+1}^{(n)}, {y}_{t+1}^{(n)})-\frac{1}{N}\sum_{n=1}^{N}     v^{(n)}_{2, t+1} \|^2]  + c_8\lambda^2 \mathbb{E}[\| \frac{1}{N}\sum_{n=1}^{N} \nabla_{2} g^{(n)}({x}^{(n)}_{t+1}, {z}^{(n)}_{t+1}) - \frac{1}{N}\sum_{n=1}^{N}{w}^{(n)}_{1, t+1} \|^2]  \ .
	\end{align}
	
	Then, we have
	\begin{align} \label{eq:p-t+1-p-t}
		& \quad P_{t+1} - P_{t} \notag  \\
		& \leq -\frac{\alpha_{x}\eta}{2} \mathbb{E}[\|\nabla \mathcal{L}_{\lambda}^{*}(\bar{x}_{t})\|^2] \notag  \\
		& + \Big(- \frac{\alpha_{x}\eta}{4} + \frac{100\eta\alpha^2_{x} (L_f+\lambda L_{g, 1})^2}{3\alpha_y\lambda^2\mu^3} c_1 + \frac{25\eta\alpha^2_{x}  L^2_{g,1}}{6\alpha_z \mu^3} c_2 + C_1\Big) \mathbb{E}[\| \bar{m}_{x, t} \|^2] \notag  \\
		& +\Big(- \frac{3\eta\alpha^2_y\lambda}{4} c_1  + C_2\Big)\mathbb{E}[\|\bar{m}_{y, t}\|^2] \notag  \\
		& + \Big( - \frac{3\eta\alpha^2_z\lambda}{4} c_2+ C_3\Big) \mathbb{E}[\|\bar{m}_{z, t}\|^2] \notag  \\
		& + \Big(\frac{9\alpha_{x}\eta}{2}  (L_f^2+\lambda^2L_{g,1}^2) -\frac{\eta\alpha_y\mu\lambda^2}{8} c_1\Big) \mathbb{E}[\|{y}_{\lambda}^{*}(\bar{x}_{t})-\bar{y}_{t}\|^2] \notag  \\
		& + \Big(\frac{9\alpha_{x}\eta}{2}  \lambda^2L_{g,1}^2 -\frac{\eta\alpha_z\lambda^2\mu}{4} c_2\Big) \mathbb{E}[\|{y}^{*}(\bar{x}_{t})-\bar{z}_{t}\|^2] \notag  \\
		& + C_4 \frac{1}{N}\sum_{n=1}^{N} \mathbb{E}[\| x^{(n)}_{t} - \bar{x}_{t}\|^2] + C_5 \frac{1}{N}\sum_{n=1}^{N} \mathbb{E}[\| y^{(n)}_{t} - \bar{y}_{t}\|^2]   + C_6 \frac{1}{N}\sum_{n=1}^{N} \mathbb{E}[\| z^{(n)}_{t} - \bar{z}_{t}\|^2] \notag  \\
		& +  C_1 \frac{1}{N}\sum_{n=1}^{N}\mathbb{E}[\| u^{(n)}_{1, t} - \bar{u}_{1, t} \|^2]  + \lambda^2 C_1 \frac{1}{N}\sum_{n=1}^{N}\mathbb{E}[\| u^{(n)}_{2, t} - \bar{u}_{2, t} \|^2]  + \lambda^2 C_1 \frac{1}{N}\sum_{n=1}^{N}\mathbb{E}[\| u^{(n)}_{3, t} - \bar{u}_{3, t} \|^2] \notag  \\
		& +  C_2 \frac{1}{N}\sum_{n=1}^{N}\mathbb{E}[\| v^{(n)}_{1, t} - \bar{v}_{1, t} \|^2] + \lambda^2 C_2 \frac{1}{N}\sum_{n=1}^{N}\mathbb{E}[\| v^{(n)}_{2, t} - \bar{v}_{2, t} \|^2]  + \lambda^2 C_3 \frac{1}{N}\sum_{n=1}^{N}\mathbb{E}[\|    w^{(n)}_{1, t} -  \bar{w}_{1, t}   \|^2] \notag  \\
		& +  (\frac{9\alpha_{x}\eta}{2} -\beta_x\eta^2 c_3) \mathbb{E}[\|\frac{1}{N}\sum_{n=1}^{N}\nabla_{1} f^{(n)}(x^{(n)}_{t}, y^{(n)}_{t})- \frac{1}{N}\sum_{n=1}^{N} u^{(n)}_{1, t} \|^2] \notag  \\
		& + (\frac{9\alpha_{x}\eta\lambda^2}{2} - \beta_x\eta^2\lambda^2 c_4) \mathbb{E}[\|\frac{1}{N}\sum_{n=1}^{N}\nabla_{1} g^{(n)}(x^{(n)}_{t}, y^{(n)}_{t})- \frac{1}{N}\sum_{n=1}^{N} u^{(n)}_{2, t} \|^2] \notag  \\
		& + (\frac{9\alpha_{x}\eta\lambda^2}{2} - \beta_x\eta^2\lambda^2 c_5) \mathbb{E}[\|\frac{1}{N}\sum_{n=1}^{N}\nabla_{1} g^{(n)}(x^{(n)}_{t}, z^{(n)}_{t})- \frac{1}{N}\sum_{n=1}^{N} u^{(n)}_{3, t} \|^2] \notag  \\
		& + (\frac{100\eta\alpha_y}{3\mu} c_1 -\beta_y\eta^2 c_6) \mathbb{E}[\|\frac{1}{N}\sum_{n=1}^{N} \nabla_{2} f^{(n)}({x}_{t}^{(n)}, {y}_{t}^{(n)}) -\frac{1}{N}\sum_{n=1}^{N} v^{(n)}_{1, t}  \|^2] \notag  \\
		& + (\frac{100\eta\alpha_y}{3\mu} \lambda^2 c_1 - \beta_y\eta^2\lambda^2 c_7) \mathbb{E}[\|\frac{1}{N}\sum_{n=1}^{N}  \nabla_{2} g^{(n)}({x}_{t}^{(n)}, {y}_{t}^{(n)})-\frac{1}{N}\sum_{n=1}^{N}     v^{(n)}_{2, t} \|^2] \notag  \\
		& + (\frac{25\eta\alpha_z}{3\mu} \lambda^2 c_2 -\beta_z\eta^2\lambda^2 c_8) \mathbb{E}[\| \frac{1}{N}\sum_{n=1}^{N} \nabla_{2} g^{(n)}({x}^{(n)}_{t}, {z}^{(n)}_{t}) - \frac{1}{N}\sum_{n=1}^{N}{w}^{(n)}_{1, t} \|^2] \notag  \\
		& + 2\beta_x^2\eta^4\sigma^2 \frac{1}{N}c_3 + 2\beta_x^2\eta^4\sigma^2 \frac{1}{N}\lambda^2 c_4 + 2\beta_x^2\eta^4\sigma^2 \frac{1}{N}\lambda^2 c_5 + 2\beta_y^2\eta^4\sigma^2 \frac{1}{N}c_6 + 2\beta_y^2\eta^4\sigma^2 \frac{1}{N}\lambda^2 c_7 + 2\beta_z^2\eta^4\sigma^2 \frac{1}{N}\lambda^2 c_8  \ , 
	\end{align}
	where
	\begin{align}
		& C_1 = \frac{12\alpha_x^2\eta^2L_{f}^{2}}{N}c_3 + \frac{12\alpha_x^2\eta^2L_{g, 1}^{2}}{N}\lambda^2 c_4 + \frac{12\alpha_x^2\eta^2L_{g, 1}^{2}}{N}\lambda^2 c_5 + \frac{12\alpha_x^2\eta^2L_{f}^{2}}{N}c_6 + \frac{12\alpha_x^2\eta^2L_{g, 1}^{2}}{N}\lambda^2 c_7 + \frac{12\alpha_x^2\eta^2L_{g, 1}^{2}}{N}\lambda^2 c_8  \ , \notag  \\
		& C_2 = \frac{8\alpha_y^2\eta^2L_{f}^{2}}{N}c_3 + \frac{8\alpha_y^2\eta^2L_{g, 1}^{2}}{N}\lambda^2 c_4 +\frac{8\alpha_y^2\eta^2L_{f}^{2}}{N}c_6 +\frac{8\alpha_y^2\eta^2L_{g, 1}^{2}}{N}\lambda^2 c_7 \  , \notag  \\
		& C_3 = \frac{4\alpha_z^2\eta^2 L_{g, 1}^2}{N}\lambda^2 c_5 + \frac{4\alpha_z^2\eta^2 L_{g, 1}^2}{N}\lambda^2 c_8 \  , \notag  \\
		& C_4 = \frac{9\alpha_{x}\eta}{2}(L_f^2+\lambda^2L_{g,1}^2+\lambda^2L_{g,1}^2) + \frac{100\eta\alpha_y}{3\mu} (L_f^2 + \lambda^2L_{g, 1}^2)c_1 + \frac{25\eta\alpha_z}{3\mu} \lambda^2 L_{g, 1}^2 c_2 \ , \notag  \\
		& C_5 = \frac{9\alpha_{x}\eta}{2}(L_f^2+\lambda^2L_{g,1}^2) + \frac{100\eta\alpha_y}{3\mu} (L_f^2 + \lambda^2L_{g, 1}^2)c_1  \ , \notag  \\
		& C_6 = \frac{9\alpha_{x}\eta}{2}\lambda^2L_{g,1}^2 + \frac{25\eta\alpha_z}{3\mu} \lambda^2 L_{g, 1}^2 c_2 \  .
	\end{align}

	
	By setting $\lambda>1$ and 
	\begin{align} \label{eq:c-values}
		& c_1 = \frac{72\alpha_{x}(L_f^2+L_{g,1}^2)}{\alpha_y\mu }  \ ,   c_2 = \frac{36\alpha_{x}L_{g,1}^2}{\alpha_z \mu} \ ,   c_3 = \frac{9\alpha_{x}}{2\beta_x\eta} \ ,  c_4 = \frac{9\alpha_{x}}{2\beta_x\eta} \ , \notag  \\
		& c_5 = \frac{9\alpha_{x} }{2\beta_x\eta} \ , c_6 = \frac{1200\alpha_{x}(L_f^2+L_{g,1}^2)}{\mu^2 \beta_y\eta} \ ,  c_7 = \frac{1200\alpha_{x}(L_f^2+L_{g,1}^2)}{\mu^2 \beta_y\eta} \ ,  c_8 = \frac{1200\alpha_{x}(L_f^2+L_{g,1}^2)}{\mu^2 \beta_z\eta}   \ , 
	\end{align}
	we have
	\begin{align}
		& \quad P_{t+1} - P_{t} \notag  \\
		& \leq -\frac{\alpha_{x}\eta}{2} \mathbb{E}[\|\nabla \mathcal{L}_{\lambda}^{*}(\bar{x}_{t})\|^2] - \frac{9\alpha_{x}\eta}{2}  (L_f^2+L_{g,1}^2)\lambda^2  \mathbb{E}[\|{y}_{\lambda}^{*}(\bar{x}_{t})-\bar{y}_{t}\|^2] - \frac{9\alpha_{x}\eta}{2}  \lambda^2L_{g,1}^2  \mathbb{E}[\|{y}^{*}(\bar{x}_{t})-\bar{z}_{t}\|^2] \notag  \\
		& \quad + \Big(- \frac{\alpha_{x}\eta}{4} + \frac{100\eta\alpha^2_{x} (L_f+\lambda L_{g, 1})^2}{3\alpha_y\lambda^2\mu^3} c_1 + \frac{25\eta\alpha^2_{x}  L^2_{g,1}}{6\alpha_z \mu^3} c_2 + C_1\Big) \mathbb{E}[\| \bar{m}_{x, t} \|^2] \notag  \\
		& \quad+\Big(- \frac{3\eta\alpha^2_y\lambda}{4} c_1  + C_2\Big)\mathbb{E}[\|\bar{m}_{y, t}\|^2] \notag  \\
		& \quad+ \Big( - \frac{3\eta\alpha^2_z\lambda}{4} c_2+ C_3\Big) \mathbb{E}[\|\bar{m}_{z, t}\|^2] \notag  \\
		& \quad+ C_4 \frac{1}{N}\sum_{n=1}^{N} \mathbb{E}[\| x^{(n)}_{t} - \bar{x}_{t}\|^2] + C_5 \frac{1}{N}\sum_{n=1}^{N} \mathbb{E}[\| y^{(n)}_{t} - \bar{y}_{t}\|^2]   + C_6 \frac{1}{N}\sum_{n=1}^{N} \mathbb{E}[\| z^{(n)}_{t} - \bar{z}_{t}\|^2] \notag  \\
		& \quad+  C_1 \frac{1}{N}\sum_{n=1}^{N}\mathbb{E}[\| u^{(n)}_{1, t} - \bar{u}_{1, t} \|^2]  + \lambda^2 C_1 \frac{1}{N}\sum_{n=1}^{N}\mathbb{E}[\| u^{(n)}_{2, t} - \bar{u}_{2, t} \|^2]  + \lambda^2 C_1 \frac{1}{N}\sum_{n=1}^{N}\mathbb{E}[\| u^{(n)}_{3, t} - \bar{u}_{3, t} \|^2] \notag  \\
		& \quad+  C_2 \frac{1}{N}\sum_{n=1}^{N}\mathbb{E}[\| v^{(n)}_{1, t} - \bar{v}_{1, t} \|^2] + \lambda^2 C_2 \frac{1}{N}\sum_{n=1}^{N}\mathbb{E}[\| v^{(n)}_{2, t} - \bar{v}_{2, t} \|^2]  + \lambda^2 C_3 \frac{1}{N}\sum_{n=1}^{N}\mathbb{E}[\|    w^{(n)}_{1, t} -  \bar{w}_{1, t}   \|^2] \notag  \\
		& \quad + 2\beta_x^2\eta^4\sigma^2 \frac{1}{N}c_3 + 2\beta_x^2\eta^4\sigma^2 \frac{1}{N}\lambda^2 c_4 + 2\beta_x^2\eta^4\sigma^2 \frac{1}{N}\lambda^2 c_5 + 2\beta_y^2\eta^4\sigma^2 \frac{1}{N}c_6 + 2\beta_y^2\eta^4\sigma^2 \frac{1}{N}\lambda^2 c_7 + 2\beta_z^2\eta^4\sigma^2 \frac{1}{N}\lambda^2 c_8 \ .
	\end{align}
	
	By summing over $t$ from $0$ to $T -1$, we have
	\begin{align}
		& \quad \frac{\alpha_{x}\eta}{2} \frac{1}{T}\sum_{t=0}^{T-1}  \mathbb{E}[\|\nabla \mathcal{L}_{\lambda}^{*}(\bar{x}_{t})\|^2]  \notag  \\
		& \leq \frac{P_{0} - P_{T}}{T}  + \Big(- \frac{\alpha_{x}\eta}{4} + \frac{100\eta\alpha^2_{x} (L_f+\lambda L_{g, 1})^2}{3\alpha_y\lambda^2\mu^3} c_1 + \frac{25\eta\alpha^2_{x}  L^2_{g,1}}{6\alpha_z \mu^3} c_2 + C_1\Big) \frac{1}{T}\sum_{t=0}^{T-1}\mathbb{E}[\| \bar{m}_{x, t} \|^2] \notag  \\
		& \quad+\Big(- \frac{3\eta\alpha^2_y\lambda}{4} c_1  + C_2\Big)\frac{1}{T}\sum_{t=0}^{T-1}\mathbb{E}[\|\bar{m}_{y, t}\|^2] + \Big( - \frac{3\eta\alpha^2_z\lambda}{4} c_2+ C_3\Big) \frac{1}{T}\sum_{t=0}^{T-1}\mathbb{E}[\|\bar{m}_{z, t}\|^2] \notag  \\
		& \quad+ C_4 \frac{1}{T}\sum_{t=0}^{T-1}\frac{1}{N}\sum_{n=1}^{N} \mathbb{E}[\| x^{(n)}_{t} - \bar{x}_{t}\|^2] + C_5 \frac{1}{T}\sum_{t=0}^{T-1}\frac{1}{N}\sum_{n=1}^{N} \mathbb{E}[\| y^{(n)}_{t} - \bar{y}_{t}\|^2]   + C_6 \frac{1}{T}\sum_{t=0}^{T-1}\frac{1}{N}\sum_{n=1}^{N} \mathbb{E}[\| z^{(n)}_{t} - \bar{z}_{t}\|^2] \notag  \\
		& \quad+  C_1\frac{1}{T}\sum_{t=0}^{T-1} \frac{1}{N}\sum_{n=1}^{N}\mathbb{E}[\| u^{(n)}_{1, t} - \bar{u}_{1, t} \|^2]  + \lambda^2 C_1 \frac{1}{T}\sum_{t=0}^{T-1}\frac{1}{N}\sum_{n=1}^{N}\mathbb{E}[\| u^{(n)}_{2, t} - \bar{u}_{2, t} \|^2]  + \lambda^2 C_1 \frac{1}{T}\sum_{t=0}^{T-1}\frac{1}{N}\sum_{n=1}^{N}\mathbb{E}[\| u^{(n)}_{3, t} - \bar{u}_{3, t} \|^2] \notag  \\
		& \quad+  C_2 \frac{1}{T}\sum_{t=0}^{T-1}\frac{1}{N}\sum_{n=1}^{N}\mathbb{E}[\| v^{(n)}_{1, t} - \bar{v}_{1, t} \|^2] + \lambda^2 C_2\frac{1}{T}\sum_{t=0}^{T-1} \frac{1}{N}\sum_{n=1}^{N}\mathbb{E}[\| v^{(n)}_{2, t} - \bar{v}_{2, t} \|^2]  + \lambda^2 C_3 \frac{1}{T}\sum_{t=0}^{T-1}\frac{1}{N}\sum_{n=1}^{N}\mathbb{E}[\|    w^{(n)}_{1, t} -  \bar{w}_{1, t}   \|^2] \notag  \\
		& \quad + 2\beta_x^2\eta^4\sigma^2 \frac{1}{N}c_3 + 2\beta_x^2\eta^4\sigma^2 \frac{1}{N}\lambda^2 c_4 + 2\beta_x^2\eta^4\sigma^2 \frac{1}{N}\lambda^2 c_5 + 2\beta_y^2\eta^4\sigma^2 \frac{1}{N}c_6 + 2\beta_y^2\eta^4\sigma^2 \frac{1}{N}\lambda^2 c_7 + 2\beta_z^2\eta^4\sigma^2 \frac{1}{N}\lambda^2 c_8 \ . 
	\end{align}
	
	Based on Lemma~\ref{lemma:consensus-error-variables}, we have
	\begin{align}
		&  \frac{1}{T}\sum_{t=0}^{T-1}\frac{1}{N}\sum_{n=1}^{N}\mathbb{E}[\| x^{(n)}_{t} - \bar{x}_{t} \|^2]  \leq 3 p^2\alpha_x^2\eta^2 \frac{1}{T}\sum_{t=0}^{T-1} \frac{1}{N}\sum_{n=1}^{N}\mathbb{E}[\| u^{(n)}_{1, t} -\bar{u}_{1, t} \|^2] + 3 p^2\alpha_x^2\eta^2\lambda^2 \frac{1}{T}\sum_{t=0}^{T-1}\frac{1}{N}\sum_{n=1}^{N}\mathbb{E}[\| u^{(n)}_{2, t} -\bar{u}_{2, t} \|^2] \notag \notag  \\
		& \quad + 3 p^2\alpha_x^2\eta^2\lambda^2 \frac{1}{T}\sum_{t=0}^{T-1} \frac{1}{N}\sum_{n=1}^{N}\mathbb{E}[\| u^{(n)}_{3, t} -\bar{u}_{3, t} \|^2] \ ,  \notag  \\
		&   \frac{1}{T}\sum_{t=0}^{T-1}\frac{1}{N}\sum_{n=1}^{N}\mathbb{E}[\| y^{(n)}_{t} - \bar{y}_{t} \|^2]  \leq 2p^2\alpha_y^2\eta^2 \frac{1}{T}\sum_{t=0}^{T-1} \frac{1}{N}\sum_{n=1}^{N}\mathbb{E}[\| v^{(n)}_{1, t} -\bar{v}_{1, t} \|^2] + 2 p^2\alpha_y^2\eta^2\lambda^2 \frac{1}{T}\sum_{t=0}^{T-1} \frac{1}{N}\sum_{n=1}^{N}\mathbb{E}[\| v^{(n)}_{2, t} -\bar{v}_{2, t} \|^2] \ , \notag  \\
		&	  \frac{1}{T}\sum_{t=0}^{T-1} \frac{1}{N}\sum_{n=1}^{N}\mathbb{E}[\| z^{(n)}_{t} - \bar{z}_{t} \|^2]  \leq  p^2\lambda^2\alpha_z^2\eta^2  \frac{1}{T}\sum_{t=0}^{T-1}\frac{1}{N}\sum_{n=1}^{N}\mathbb{E}[\|   w^{(n)}_{1, t}   -  \bar{w}_{1, t}\|^2] \ . 
	\end{align}
	Then,  we have
	\begin{align}
		& \quad \frac{\alpha_{x}\eta}{2} \frac{1}{T}\sum_{t=0}^{T-1}  \mathbb{E}[\|\nabla \mathcal{L}_{\lambda}^{*}(\bar{x}_{t})\|^2]  \notag  \\
		& \leq \frac{P_{0} - P_{T}}{T}+ \Big(- \frac{\alpha_{x}\eta}{4} + \frac{100\eta\alpha^2_{x} (L_f+\lambda L_{g, 1})^2}{3\alpha_y\lambda^2\mu^3} c_1 + \frac{25\eta\alpha^2_{x}  L^2_{g,1}}{6\alpha_z \mu^3} c_2 + C_1\Big) \frac{1}{T}\sum_{t=0}^{T-1}\mathbb{E}[\| \bar{m}_{x, t} \|^2] \notag  \\
		& \quad+\Big(- \frac{3\eta\alpha^2_y\lambda}{4} c_1  + C_2\Big)\frac{1}{T}\sum_{t=0}^{T-1}\mathbb{E}[\|\bar{m}_{y, t}\|^2] + \Big( - \frac{3\eta\alpha^2_z\lambda}{4} c_2+ C_3\Big) \frac{1}{T}\sum_{t=0}^{T-1}\mathbb{E}[\|\bar{m}_{z, t}\|^2] \notag  \\
		& \quad+  \Big(C_1+  3 C_4p^2\alpha_x^2\eta^2 \Big)\frac{1}{T}\sum_{t=0}^{T-1} \frac{1}{N}\sum_{n=1}^{N}\mathbb{E}[\| u^{(n)}_{1, t} - \bar{u}_{1, t} \|^2]  \notag  \\
		& \quad + \Big( C_1+ 3 C_4p^2\alpha_x^2\eta^2\Big) \lambda^2\frac{1}{T}\sum_{t=0}^{T-1}\frac{1}{N}\sum_{n=1}^{N}\mathbb{E}[\| u^{(n)}_{2, t} - \bar{u}_{2, t} \|^2] \notag  \\
		& \quad + \Big( C_1+ 3 C_4p^2\alpha_x^2\eta^2\Big) \lambda^2\frac{1}{T}\sum_{t=0}^{T-1}\frac{1}{N}\sum_{n=1}^{N}\mathbb{E}[\| u^{(n)}_{3, t} - \bar{u}_{3, t} \|^2] \notag  \\
		& \quad+  \Big(C_2+  2C_5p^2\alpha_y^2\eta^2\Big) \frac{1}{T}\sum_{t=0}^{T-1}\frac{1}{N}\sum_{n=1}^{N}\mathbb{E}[\| v^{(n)}_{1, t} - \bar{v}_{1, t} \|^2] \notag  \\
		& \quad + \Big(C_2+  2C_5p^2\alpha_y^2\eta^2\Big) \lambda^2 \frac{1}{T}\sum_{t=0}^{T-1} \frac{1}{N}\sum_{n=1}^{N}\mathbb{E}[\| v^{(n)}_{2, t} - \bar{v}_{2, t} \|^2]  \notag  \\
		& \quad +\Big( C_3+ C_6 p^2\alpha_z^2\eta^2\Big) \lambda^2 \frac{1}{T}\sum_{t=0}^{T-1}\frac{1}{N}\sum_{n=1}^{N}\mathbb{E}[\|    w^{(n)}_{1, t} -  \bar{w}_{1, t}   \|^2] \notag  \\
		& \quad + 2\beta_x^2\eta^4\sigma^2 \frac{1}{N}c_3 + 2\beta_x^2\eta^4\sigma^2 \frac{1}{N}\lambda^2 c_4 + 2\beta_x^2\eta^4\sigma^2 \frac{1}{N}\lambda^2 c_5 + 2\beta_y^2\eta^4\sigma^2 \frac{1}{N}c_6 + 2\beta_y^2\eta^4\sigma^2 \frac{1}{N}\lambda^2 c_7 + 2\beta_z^2\eta^4\sigma^2 \frac{1}{N}\lambda^2 c_8 \notag  \\
		& \leq \frac{P_{0} - P_{T}}{T} + \Big(- \frac{\alpha_{x}\eta}{4} + \frac{100\eta\alpha^2_{x} (L_f+\lambda L_{g, 1})^2}{3\alpha_y\lambda^2\mu^3} c_1 + \frac{25\eta\alpha^2_{x}  L^2_{g,1}}{6\alpha_z \mu^3} c_2 + C_1\Big) \frac{1}{T}\sum_{t=0}^{T-1}\mathbb{E}[\| \bar{m}_{x, t} \|^2] \notag  \\
		& \quad+\Big(- \frac{3\eta\alpha^2_y\lambda}{4} c_1  + C_2\Big)\frac{1}{T}\sum_{t=0}^{T-1}\mathbb{E}[\|\bar{m}_{y, t}\|^2] + \Big( - \frac{3\eta\alpha^2_z\lambda}{4} c_2+ C_3\Big) \frac{1}{T}\sum_{t=0}^{T-1}\mathbb{E}[\|\bar{m}_{z, t}\|^2] \notag  \\
		& \quad+  \Big(C_1+  3 C_4p^2\alpha_x^2\eta^2 +C_2+  2C_5p^2\alpha_y^2\eta^2+C_3+ C_6 p^2\alpha_z^2\eta^2\Big)\Big(\frac{1}{T}\sum_{t=0}^{T-1} \frac{1}{N}\sum_{n=1}^{N}\mathbb{E}[\| u^{(n)}_{1, t} - \bar{u}_{1, t} \|^2] \notag  \\
		& \quad \quad +\lambda^2\frac{1}{T}\sum_{t=0}^{T-1}\frac{1}{N}\sum_{n=1}^{N}\mathbb{E}[\| u^{(n)}_{2, t} - \bar{u}_{2, t} \|^2] + \lambda^2\frac{1}{T}\sum_{t=0}^{T-1}\frac{1}{N}\sum_{n=1}^{N}\mathbb{E}[\| u^{(n)}_{3, t} - \bar{u}_{3, t} \|^2] \notag  \\
		& \quad\quad + \frac{1}{T}\sum_{t=0}^{T-1}\frac{1}{N}\sum_{n=1}^{N}\mathbb{E}[\| v^{(n)}_{1, t} - \bar{v}_{1, t} \|^2] +  \lambda^2 \frac{1}{T}\sum_{t=0}^{T-1} \frac{1}{N}\sum_{n=1}^{N}\mathbb{E}[\| v^{(n)}_{2, t} - \bar{v}_{2, t} \|^2]  + \lambda^2 \frac{1}{T}\sum_{t=0}^{T-1}\frac{1}{N}\sum_{n=1}^{N}\mathbb{E}[\|    w^{(n)}_{1, t} -  \bar{w}_{1, t}   \|^2] \Big)\notag  \\
		& \quad + 2\beta_x^2\eta^4\sigma^2 \frac{1}{N}c_3 + 2\beta_x^2\eta^4\sigma^2 \frac{1}{N}\lambda^2 c_4 + 2\beta_x^2\eta^4\sigma^2 \frac{1}{N}\lambda^2 c_5 + 2\beta_y^2\eta^4\sigma^2 \frac{1}{N}c_6 + 2\beta_y^2\eta^4\sigma^2 \frac{1}{N}\lambda^2 c_7 + 2\beta_z^2\eta^4\sigma^2 \frac{1}{N}\lambda^2 c_8 \notag  \\
		& \leq \frac{P_{0} - P_{T}}{T}  + \Big(- \frac{\alpha_{x}\eta}{4} + \frac{100\eta\alpha^2_{x} (L_f+\lambda L_{g, 1})^2}{3\alpha_y\lambda^2\mu^3} c_1 + \frac{25\eta\alpha^2_{x}  L^2_{g,1}}{6\alpha_z \mu^3} c_2 + C_1\Big) \frac{1}{T}\sum_{t=0}^{T-1}\mathbb{E}[\| \bar{m}_{x, t} \|^2] \notag  \\
		& \quad+\Big(- \frac{3\eta\alpha^2_y\lambda}{4} c_1  + C_2\Big)\frac{1}{T}\sum_{t=0}^{T-1}\mathbb{E}[\|\bar{m}_{y, t}\|^2] + \Big( - \frac{3\eta\alpha^2_z\lambda}{4} c_2+ C_3\Big) \frac{1}{T}\sum_{t=0}^{T-1}\mathbb{E}[\|\bar{m}_{z, t}\|^2] \notag  \\
		& \quad+  576p^2\alpha_x^2\eta^2\Big(L_{f}^2+ L_{g, 1}^2  \lambda^2 \Big) C_7\frac{1}{T}\sum_{t=0}^{T-1}\mathbb{E}[\| \bar{m}_{x, t} \|^2]   + 576p^2\alpha_y^2\eta^2 \Big(L_{f}^2+ L_{g, 1}^2\lambda^2\Big)C_7\frac{1}{T}\sum_{t=0}^{T-1} \mathbb{E}[\| \bar{m}_{y, t} \|^2]  \notag  \\
		& \quad  + 576 p^2\alpha_z^2\eta^2 \Big(L_{f}^2+ L_{g, 1}^2\lambda^2\Big)C_7\frac{1}{T}\sum_{t=0}^{T-1}\mathbb{E}[\|   \bar{m}_{z, t}  \|^2]  \notag  \\
		& \quad +576p^2\beta_x^2\eta^4( 1+  \lambda^2 )\delta^2C_7   +576p^2\beta_y^2\eta^4(  1+  \lambda^2 )   \delta^2C_7 + 576p^2\beta_z^2\eta^4 \lambda^2\delta^2 C_7 \notag  \\
		& \quad + 576p^2\beta_x^2 \eta^4(  1+ \lambda^2)\sigma^{2}C_7 + 576p^2\beta_y^2\eta^4(  1 + \lambda^2  )\sigma^{2} C_7+576p^2\beta_z^2\eta^4\lambda^2\sigma^{2}C_7\notag  \\
		& \quad + 2\beta_x^2\eta^4\sigma^2 \frac{1}{N}c_3 + 2\beta_x^2\eta^4\sigma^2 \frac{1}{N}\lambda^2 c_4 + 2\beta_x^2\eta^4\sigma^2 \frac{1}{N}\lambda^2 c_5 + 2\beta_y^2\eta^4\sigma^2 \frac{1}{N}c_6 + 2\beta_y^2\eta^4\sigma^2 \frac{1}{N}\lambda^2 c_7 + 2\beta_z^2\eta^4\sigma^2 \frac{1}{N}\lambda^2 c_8 \notag  \\
		& = \frac{P_{0} - P_{T}}{T}  + \left(- \frac{\alpha_{x}\eta}{4} + \frac{100\eta\alpha^2_{x} (L_f+\lambda L_{g, 1})^2}{3\alpha_y\lambda^2\mu^3} c_1 + \frac{25\eta\alpha^2_{x}  L^2_{g,1}}{6\alpha_z \mu^3} c_2 + C_1 + 576p^2\alpha_x^2\eta^2\Big(L_{f}^2+ L_{g, 1}^2  \lambda^2 \Big) C_7\right) \frac{1}{T}\sum_{t=0}^{T-1}\mathbb{E}[\| \bar{m}_{x, t} \|^2] \notag  \\
		& \quad+\left(- \frac{3\eta\alpha^2_y\lambda}{4} c_1  + C_2 + 576p^2\alpha_y^2\eta^2 \Big(L_{f}^2+ L_{g, 1}^2\lambda^2\Big)C_7\right)\frac{1}{T}\sum_{t=0}^{T-1}\mathbb{E}[\|\bar{m}_{y, t}\|^2] \notag  \\
		& \quad + \left( - \frac{3\eta\alpha^2_z\lambda}{4} c_2+ C_3+ 576 p^2\alpha_z^2\eta^2 \Big(L_{f}^2+ L_{g, 1}^2\lambda^2\Big)C_7\right) \frac{1}{T}\sum_{t=0}^{T-1}\mathbb{E}[\|\bar{m}_{z, t}\|^2] \notag  \\
		& \quad +576p^2\beta_x^2\eta^4( 1+  \lambda^2 )\delta^2C_7   +576p^2\beta_y^2\eta^4(  1+  \lambda^2 )   \delta^2C_7 + 576p^2\beta_z^2\eta^4 \lambda^2\delta^2 C_7 \notag  \\
		& \quad + 576p^2\beta_x^2 \eta^4(  1+ \lambda^2)\sigma^{2}C_7 + 576p^2\beta_y^2\eta^4(  1 + \lambda^2  )\sigma^{2} C_7+576p^2\beta_z^2\eta^4\lambda^2\sigma^{2}C_7\notag  \\
		& \quad + 2\beta_x^2\eta^4\sigma^2 \frac{1}{N}c_3 + 2\beta_x^2\eta^4\sigma^2 \frac{1}{N}\lambda^2 c_4 + 2\beta_x^2\eta^4\sigma^2 \frac{1}{N}\lambda^2 c_5 + 2\beta_y^2\eta^4\sigma^2 \frac{1}{N}c_6 + 2\beta_y^2\eta^4\sigma^2 \frac{1}{N}\lambda^2 c_7 + 2\beta_z^2\eta^4\sigma^2 \frac{1}{N}\lambda^2 c_8 \ , 
	\end{align}
	where  the third step follows from  Lemma~\ref{lemma:consensus-error-momentum-each-step} and 
	\begin{align}
		C_7 =  C_1+  3 C_4p^2\alpha_x^2\eta^2 +C_2+  2C_5p^2\alpha_y^2\eta^2+C_3+ C_6 p^2\alpha_z^2\eta^2 \ .
	\end{align}
	
	In the following, we select hyperparameters to eliminate $ \frac{1}{T}\sum_{t=0}^{T-1}\mathbb{E}[\| \bar{m}_{x, t} \|^2]$, $ \frac{1}{T}\sum_{t=0}^{T-1}\mathbb{E}[\| \bar{m}_{y, t} \|^2]$, and $ \frac{1}{T}\sum_{t=0}^{T-1}\mathbb{E}[\| \bar{m}_{z, t} \|^2]$ from the upper bound. 

    At first, based on Eq.~(\ref{eq:c-values}), $L_{g, 1}>\mu$, and $\lambda\geq \max\{\frac{2L_f}{\mu}, 1\}$, we have
    \begin{align}
	& C_1 \leq    \left(\frac{1}{\beta_x N} + \frac{1}{\beta_y N} + \frac{1}{\beta_zN}\right)\frac{28800\alpha_x^3\eta (L_{f}^{2}+L_{g, 1}^{2})^2}{\mu^2 }  \lambda^2  \ , \notag \\
	& C_2 \leq \left(\frac{1}{\beta_x N} + \frac{1}{\beta_y N} \right) \frac{19200\alpha_y^2\alpha_{x}\eta (L_f^2+L_{g,1}^2)^2}{ \mu^2  }\lambda^2  \ , \notag \\
	& C_3 \leq  \frac{1}{\beta_z N}\frac{4818\alpha_z^2\alpha_{x} \eta  (L_f^2+L_{g,1}^2)^2}{\mu^2 }\lambda^2  \  ,  \notag \\
	& C_4   \leq  \frac{1360\alpha_{x}\eta(L_f^2+L_{g,1}^2)^2}{\mu^2} \lambda^2  \ , \notag \\
	& C_5  \leq \frac{1210\alpha_{x}\eta (L_f^2+L_{g,1}^2)^2}{\mu^2} \lambda^2 \ ,  \notag \\
	& C_6 \leq  \frac{160 \alpha_{x}\eta(L_f^2+L_{g,1}^2)^2}{\mu^2} \lambda^2   \ ,  \notag \\
	& C_7  \leq  \left(\frac{1}{\beta_x N} + \frac{1}{\beta_y N} + \frac{1}{\beta_zN}\right)\frac{28800\alpha_x^3\eta (L_{f}^{2}+L_{g, 1}^{2})^2}{\mu^2 }  \lambda^2 +\left(\frac{1}{\beta_x N} + \frac{1}{\beta_y N} \right) \frac{19200\alpha_y^2\alpha_{x}\eta (L_f^2+L_{g,1}^2)^2}{ \mu^2  }\lambda^2 \notag \\
	& \quad +  \frac{1}{\beta_z N}\frac{4818\alpha_z^2\alpha_{x} \eta  (L_f^2+L_{g,1}^2)^2}{\mu^2 }\lambda^2 +   \frac{4080\alpha^3_{x}p^2\eta^3(L_f^2+L_{g,1}^2)^2}{\mu^2} \lambda^2  \notag \\
	& \quad +  \frac{2420\alpha_{x}\alpha_y^2p^2\eta^3 (L_f^2+L_{g,1}^2)^2}{\mu^2} \lambda^2+    \frac{160 \alpha_{x}\alpha_z^2p^2\eta^3(L_f^2+L_{g,1}^2)^2}{\mu^2} \lambda^2  \ . 
\end{align}

Then, to eliminate  $ \frac{1}{T}\sum_{t=0}^{T-1}\mathbb{E}[\| \bar{m}_{x, t} \|^2]$, we enforce
\begin{align}
	& \quad - \frac{\alpha_{x}\eta}{4} + \frac{100\eta\alpha^2_{x} (L_f+\lambda L_{g, 1})^2}{3\alpha_y\lambda^2\mu^3} c_1 + \frac{25\eta\alpha^2_{x}  L^2_{g,1}}{6\alpha_z \mu^3} c_2 + C_1 + 576p^2\alpha_x^2\eta^2\Big(L_{f}^2+ L_{g, 1}^2  \lambda^2 \Big) C_7 \notag \\
	& \leq - \frac{\alpha_{x}\eta}{4} + \frac{100\eta\alpha^2_{x} (L_f+\lambda L_{g, 1})^2}{3\alpha_y\lambda^2\mu^3} c_1 + \frac{25\eta\alpha^2_{x}  L^2_{g,1}}{6\alpha_z \mu^3} c_2 + C_1 + 576p^2\alpha_x^2\eta^2\Big(L_{f}^2+ L_{g, 1}^2\Big)   \lambda^2  C_7 \notag \\
	& \leq - \frac{\alpha_{x}\eta}{4} + \frac{100\eta\alpha^2_{x} (L_f+\lambda L_{g, 1})^2}{3\alpha_y\lambda^2\mu^3}  \frac{36\alpha_{x}(L_f^2+L_{g,1}^2)}{\alpha_y\mu }  + \frac{25\eta\alpha^2_{x}  L^2_{g,1}}{6\alpha_z \mu^3}  \frac{18\alpha_{x}L_{g,1}^2}{\alpha_z \mu} \notag \\
	& \quad  + \left(\frac{1}{\beta_x N} + \frac{1}{\beta_y N} + \frac{1}{\beta_zN}\right)\frac{28800\alpha_x^3\eta (L_{f}^{2}+L_{g, 1}^{2})^2}{\mu^2 }  \lambda^2 \notag \\
	& \quad + 576 p^2\alpha_x^2\eta^2\Big(L_{f}^2+ L_{g, 1}^2\Big)\lambda^2\left(\frac{1}{\beta_x N} + \frac{1}{\beta_y N} + \frac{1}{\beta_zN}\right)\frac{28800\alpha_x^3\eta (L_{f}^{2}+L_{g, 1}^{2})^2}{\mu^2 }  \lambda^2 \notag \\
	& \quad + 576 p^2\alpha_x^2\eta^2\Big(L_{f}^2+ L_{g, 1}^2\Big)\lambda^2\left(\frac{1}{\beta_x N} + \frac{1}{\beta_y N} \right) \frac{19200\alpha_y^2\alpha_{x}\eta (L_f^2+L_{g,1}^2)^2}{ \mu^2  }\lambda^2 \notag \\
	& \quad +  576 p^2\alpha_x^2\eta^2\Big(L_{f}^2+ L_{g, 1}^2\Big)\lambda^2 \frac{1}{\beta_z N}\frac{4818\alpha_z^2\alpha_{x} \eta  (L_f^2+L_{g,1}^2)^2}{\mu^2 }\lambda^2 \notag \\
	& \quad +   576 p^2\alpha_x^2\eta^2 \Big(L_{f}^2+ L_{g, 1}^2\Big)\lambda^2 \frac{4080\alpha^3_{x}p^2\eta^3(L_f^2+L_{g,1}^2)^2}{\mu^2} \lambda^2  \notag \\
	& \quad +  576 p^2\alpha_x^2\eta^2 \Big(L_{f}^2+ L_{g, 1}^2\Big)\lambda^2 \frac{2420\alpha_{x}\alpha_y^2p^2\eta^3 (L_f^2+L_{g,1}^2)^2}{\mu^2} \lambda^2 \notag \\
	& \quad +   576 p^2\alpha_x^2\eta^2\Big(L_{f}^2+ L_{g, 1}^2\Big)\lambda^2  \frac{160 \alpha_{x}\alpha_z^2p^2\eta^3(L_f^2+L_{g,1}^2)^2}{\mu^2} \lambda^2  \notag \\
	& \leq  0  \ . 
\end{align}
This can be done by setting
\begin{align} \label{eq:m-x}
	&  \frac{100\eta\alpha^2_{x} (L_f+\lambda L_{g, 1})^2}{3\alpha_y\lambda^2\mu^3}  \frac{36\alpha_{x}(L_f^2+L_{g,1}^2)}{\alpha_y\mu } \leq \frac{\alpha_{x}\eta}{36} \ , \notag \\
	&  \frac{25\eta\alpha^2_{x}  L^2_{g,1}}{6\alpha_z \mu^3}  \frac{18\alpha_{x}L_{g,1}^2}{\alpha_z \mu} \leq \frac{\alpha_{x}\eta}{36} \ , \notag \\
	&  \left(\frac{1}{\beta_x N} + \frac{1}{\beta_y N} + \frac{1}{\beta_zN}\right)\frac{28800\alpha_x^3\eta (L_{f}^{2}+L_{g, 1}^{2})^2}{\mu^2 }  \lambda^2 \leq \frac{\alpha_{x}\eta}{36} \ , \notag \\
	& 576 p^2\alpha_x^2\eta^2\Big(L_{f}^2+ L_{g, 1}^2\Big)\lambda^2\left(\frac{1}{\beta_x N} + \frac{1}{\beta_y N} + \frac{1}{\beta_zN}\right)\frac{28800\alpha_x^3\eta (L_{f}^{2}+L_{g, 1}^{2})^2}{\mu^2 }  \lambda^2\leq \frac{\alpha_{x}\eta}{36} \ , \notag \\
	& 576 p^2\alpha_x^2\eta^2\Big(L_{f}^2+ L_{g, 1}^2\Big)\lambda^2\left(\frac{1}{\beta_x N} + \frac{1}{\beta_y N} \right) \frac{19200\alpha_y^2\alpha_{x}\eta (L_f^2+L_{g,1}^2)^2}{ \mu^2  }\lambda^2\leq \frac{\alpha_{x}\eta}{36} \ , \notag \\
	& 576 p^2\alpha_x^2\eta^2\Big(L_{f}^2+ L_{g, 1}^2\Big)\lambda^2 \frac{1}{\beta_z N}\frac{4818\alpha_z^2\alpha_{x} \eta  (L_f^2+L_{g,1}^2)^2}{\mu^2 }\lambda^2 \leq \frac{\alpha_{x}\eta}{36} \ , \notag \\
	&   576 p^2\alpha_x^2\eta^2 \Big(L_{f}^2+ L_{g, 1}^2\Big)\lambda^2 \frac{4080\alpha^3_{x}p^2\eta^3(L_f^2+L_{g,1}^2)^2}{\mu^2} \lambda^2  \leq \frac{\alpha_{x}\eta}{36} \ , \notag \\
	&  576 p^2\alpha_x^2\eta^2 \Big(L_{f}^2+ L_{g, 1}^2\Big)\lambda^2 \frac{2420\alpha_{x}\alpha_y^2p^2\eta^3 (L_f^2+L_{g,1}^2)^2}{\mu^2} \lambda^2\leq \frac{\alpha_{x}\eta}{36} \ , \notag \\
	&  576 p^2\alpha_x^2\eta^2\Big(L_{f}^2+ L_{g, 1}^2\Big)\lambda^2  \frac{160 \alpha_{x}\alpha_z^2p^2\eta^3(L_f^2+L_{g,1}^2)^2}{\mu^2} \lambda^2 \leq \frac{\alpha_{x}\eta}{36} \ .
\end{align}

To solve the first inequality in Eq.~(\ref{eq:m-x}), we enforce
\begin{align}
    & \quad  \frac{100\eta\alpha^2_{x} (L_f+\lambda L_{g, 1})^2}{3\alpha_y\lambda^2\mu^3}  \frac{36\alpha_{x}(L_f^2+L_{g,1}^2)}{\alpha_y\mu }  \notag \\
    & \leq \frac{100\eta\alpha^2_{x} (L_f+L_{g, 1})^2\lambda^2 }{3\alpha_y\lambda^2\mu^3}  \frac{36\alpha_{x}(L_f+L_{g,1})^2}{\alpha_y\mu }  \notag \\
    & \leq \frac{\alpha_{x}\eta}{36} \ .
\end{align}
Then, we can obtain
\begin{align}
    & \alpha_x \leq \frac{\alpha_y\mu^2}{360  (L_f+ L_{g, 1})^2} \ . 
\end{align}
By solving the second inequality in Eq.~(\ref{eq:m-x}), we can obtain
\begin{align}
      \alpha_{x}  \leq \frac{\alpha_z \mu^2}{60 (L_f+ L_{g, 1})^2} \ .
\end{align}
By solving the third inequality in Eq.~(\ref{eq:m-x}), we can obtain
\begin{align}
    & \alpha_x \leq \frac{\mu }{100 \lambda (L_{f}+L_{g, 1})^{2}}\Bigg/\sqrt{\frac{1}{\beta_x N} + \frac{1}{\beta_y N} + \frac{1}{\beta_zN}} \  . 
\end{align}
To solve the fourth inequality in Eq.~(\ref{eq:m-x}), we enforce
\begin{align}
    & \quad 576 p^2\alpha_x^2\eta^2\Big(L_{f}^2+ L_{g, 1}^2\Big)\lambda^2\left(\frac{1}{\beta_x N} + \frac{1}{\beta_y N} + \frac{1}{\beta_zN}\right)\frac{28800\alpha_x^3\eta (L_{f}^{2}+L_{g, 1}^{2})^2}{\mu^2 }  \lambda^2 \notag \\
    & \leq  576 \times 28800\alpha_x\eta \alpha_x^4 p^2\eta^2\lambda^4\left(\frac{1}{\beta_x N} + \frac{1}{\beta_y N} + \frac{1}{\beta_zN}\right)\frac{ (L_{f}+L_{g, 1})^6}{\mu^2 }   \notag \\
    & \leq \frac{\alpha_{x}\eta}{36} \ .
\end{align}
Then, we can obtain
\begin{align}
    & \alpha_x \leq \frac{1}{160 \lambda (p \eta)^{1/2}} \frac{\mu^{1/2} }{ (L_{f}+L_{g, 1})^{3/2}} \Bigg/ \left(\frac{1}{\beta_x N} + \frac{1}{\beta_y N} + \frac{1}{\beta_zN}\right)^{1/4}  \ . 
\end{align}
To solve the fifth inequality in Eq.~(\ref{eq:m-x}), we enforce
\begin{align}
    & \quad 576 p^2\alpha_x^2\eta^2\Big(L_{f}^2+ L_{g, 1}^2\Big)\lambda^2\left(\frac{1}{\beta_x N} + \frac{1}{\beta_y N} \right) \frac{19200\alpha_y^2\alpha_{x}\eta (L_f^2+L_{g,1}^2)^2}{ \mu^2  }\lambda^2 \notag \\
    & \leq  576 \times 19200\alpha_y^2\alpha_x^2p^2\eta^2\lambda^4\left(\frac{1}{\beta_x N} + \frac{1}{\beta_y N} \right) \frac{ (L_f+L_{g,1})^6}{ \mu^2  } \alpha_{x}\eta  \notag \\
    & \leq  576 \times 19200\alpha_y^2  \frac{\alpha^2_y\mu^4}{360^2  (L_f+ L_{g, 1})^4} p^2\eta^2\lambda^4\left(\frac{1}{\beta_x N} + \frac{1}{\beta_y N} \right) \frac{ (L_f+L_{g,1})^6}{ \mu^2  } \alpha_{x}\eta  \notag \\
    & \leq \frac{\alpha_{x}\eta}{36} \ .
\end{align}
Then, we can obtain
\begin{align}
    & \alpha_y	\leq \frac{1}{8\lambda ( p\eta)^{1/2} }   \frac{1 }{\sqrt{\mu (L_f+ L_{g, 1})}}\Bigg/ \left(\frac{1}{\beta_x N} + \frac{1}{\beta_y N} \right)^{1/4}   \  .
\end{align}
To solve the sixth inequality in Eq.~(\ref{eq:m-x}), we enforce
\begin{align}
& \quad 576 p^2\alpha_x^2\eta^2\Big(L_{f}^2+ L_{g, 1}^2\Big)\lambda^2 \frac{1}{\beta_z N}\frac{4818\alpha_z^2\alpha_{x} \eta  (L_f^2+L_{g,1}^2)^2}{\mu^2 }\lambda^2  \notag \\
& \leq 576 \times 4818\alpha_z^2  \alpha_x^2p^2 \eta^2\lambda^4 \frac{1}{\beta_z N}\frac{ (L_f+L_{g,1})^6}{\mu^2 } \alpha_{x} \eta  \notag \\
& \leq 576 \times 4818\alpha_z^2  \frac{\alpha^2_z \mu^4}{60^2 (L_f+ L_{g, 1})^4} p^2 \eta^2\lambda^4 \frac{1}{\beta_z N}\frac{ (L_f+L_{g,1})^6}{\mu^2 } \alpha_{x} \eta  \notag \\
& \leq \frac{\alpha_{x}\eta}{36} \ .
\end{align}
Then, we can obtain
\begin{align}
  &  \alpha_z     \leq \frac{1}{ \lambda (p \eta)^{1/2}} \frac{1}{ \sqrt{\mu (L_f+L_{g,1})}}\Bigg/\left(\frac{1}{\beta_z N}\right)^{1/4} \ . 
\end{align}
To solve the seventh inequality in Eq.~(\ref{eq:m-x}), we enforce
\begin{align}
& \quad   576 p^2\alpha_x^2\eta^2 \Big(L_{f}^2+ L_{g, 1}^2\Big)\lambda^2 \frac{4080\alpha^3_{x}p^2\eta^3(L_f^2+L_{g,1}^2)^2}{\mu^2} \lambda^2 \notag \\
& \leq 576 \times 4080 \alpha_x^4p^4\eta^4 \lambda^4 \frac{(L_f+L_{g,1})^6}{\mu^2} \alpha_{x}\eta \notag \\
&\leq \frac{\alpha_{x}\eta}{36} \ .
\end{align}
Then, we can obtain
\begin{align}
&  \alpha_x  \leq \frac{1}{100\lambda p\eta } \frac{\mu^{1/2}}{(L_f+L_{g,1})^{3/2}}\ .
\end{align}
To solve the eighth inequality in Eq.~(\ref{eq:m-x}), we enforce
\begin{align}
    & \quad 576 p^2\alpha_x^2\eta^2 \Big(L_{f}^2+ L_{g, 1}^2\Big)\lambda^2 \frac{2420\alpha_{x}\alpha_y^2p^2\eta^3 (L_f^2+L_{g,1}^2)^2}{\mu^2} \lambda^2 \\
    & \leq 576 \times 2420\alpha_y^2\alpha_x^2  p^4\eta^4 \lambda^4 \frac{ (L_f+L_{g,1})^6}{\mu^2}\alpha_{x}\eta  \notag \\
    & \leq 576 \times 2420\alpha_y^2\frac{\alpha^2_y\mu^4}{360^2  (L_f+ L_{g, 1})^4}  p^4\eta^4 \lambda^4 \frac{ (L_f+L_{g,1})^6}{\mu^2}\alpha_{x}\eta  \notag \\
    & \leq \frac{\alpha_{x}\eta}{36} \ .
\end{align}
Then, we can obtain
\begin{align}
    & \alpha_y   \leq \frac{1}{ \lambda p\eta }\frac{1}{ \sqrt{\mu(L_f+L_{g,1})}} \ .
\end{align}
To solve the last inequality in Eq.~(\ref{eq:m-x}), we enforce
\begin{align}
    & \quad 576 p^2\alpha_x^2\eta^2\Big(L_{f}^2+ L_{g, 1}^2\Big)\lambda^2  \frac{160 \alpha_{x}\alpha_z^2p^2\eta^3(L_f^2+L_{g,1}^2)^2}{\mu^2} \lambda^2 \notag \\
    & \leq 576 \times 160 \alpha_z^2  \alpha_x^2p^4\eta^4\lambda^4  \frac{(L_f+L_{g,1})^6}{\mu^2}\alpha_{x} \eta \notag \\
    & \leq 576 \times 160 \alpha_z^2  \frac{\alpha^2_z \mu^4}{60^2 (L_f+ L_{g, 1})^4}  p^4\eta^4\lambda^4  \frac{(L_f+L_{g,1})^6}{\mu^2}\alpha_{x} \eta \notag \\
    & \leq \frac{\alpha_{x}\eta}{36} \ .
\end{align}
Then, we can obtain
\begin{align}
    &  \alpha_z \leq \frac{1}{ \lambda p\eta } \frac{1}{ \sqrt{\mu(L_f+L_{g,1})}}\ .
\end{align}
In summary, by setting
\begin{align}
	& \alpha_x\leq \min \Bigg\{\frac{\alpha_y\mu^2}{360  (L_f+ L_{g, 1})^2}, \frac{\alpha_z \mu^2}{60 (L_f+ L_{g, 1})^2}, \frac{\mu }{100 \lambda (L_{f}+L_{g, 1})^{2}}\Bigg/\sqrt{\frac{1}{\beta_x N} + \frac{1}{\beta_y N} + \frac{1}{\beta_zN}} ,  \notag \\
	& \quad \quad \frac{1}{160 \lambda (p \eta)^{1/2}} \frac{\mu^{1/2} }{ (L_{f}+L_{g, 1})^{3/2}} \Bigg/ \left(\frac{1}{\beta_x N} + \frac{1}{\beta_y N} + \frac{1}{\beta_zN}\right)^{1/4}, \frac{1}{100\lambda p\eta } \frac{\mu^{1/2}}{(L_f+L_{g,1})^{3/2}}\Bigg\}  \ , \notag \\
	& \alpha_y\leq \min \left\{ \frac{1}{8\lambda ( p\eta)^{1/2} }\frac{1 }{\sqrt{\mu (L_f+ L_{g, 1})}}\Bigg/ \left(\frac{1}{\beta_x N} + \frac{1}{\beta_y N} \right)^{1/4}, \frac{1}{ \lambda p\eta }\frac{1}{ \sqrt{\mu(L_f+L_{g,1})}} \right\}  \ , \notag \\
	& \alpha_z\leq \min \left\{\frac{1}{ \lambda (p \eta)^{1/2}} \frac{1}{ \sqrt{\mu (L_f+L_{g,1})}}\Bigg/\left(\frac{1}{\beta_z N}\right)^{1/4} ,  \frac{1}{ \lambda p\eta } \frac{1}{ \sqrt{\mu(L_f+L_{g,1})}} \right\}  \  ,
\end{align}
we can eliminate $ \frac{1}{T}\sum_{t=0}^{T-1}\mathbb{E}[\| \bar{m}_{x, t} \|^2]$.

Similarly, to eliminate  $ \frac{1}{T}\sum_{t=0}^{T-1}\mathbb{E}[\| \bar{m}_{y, t} \|^2]$, we enforce
\begin{align}
	& \quad - \frac{3\eta\alpha^2_y\lambda}{4} c_1  + C_2 + 576p^2\alpha_y^2\eta^2 \Big(L_{f}^2+ L_{g, 1}^2\lambda^2\Big)C_7 \notag \\
	& \leq - \frac{3\eta\alpha^2_y\lambda}{4} c_1  + C_2 + 576p^2\alpha_y^2\eta^2 \Big(L_{f}^2+ L_{g, 1}^2\Big)\lambda^2C_7 \notag \\
	& \leq - \frac{3\eta\alpha^2_y\lambda}{4} \frac{36\alpha_{x}(L_f^2+L_{g,1}^2)}{\alpha_y\mu }    + \left(\frac{1}{\beta_x N} + \frac{1}{\beta_y N} \right) \frac{19200\alpha_y^2\alpha_{x}\eta (L_f^2+L_{g,1}^2)^2}{ \mu^2  }\lambda^2 \notag \\
	& \quad + 576 p^2\alpha_y^2\eta^2\Big(L_{f}^2+ L_{g, 1}^2\Big)\lambda^2\left(\frac{1}{\beta_x N} + \frac{1}{\beta_y N} + \frac{1}{\beta_zN}\right)\frac{28800\alpha_x^3\eta (L_{f}^{2}+L_{g, 1}^{2})^2}{\mu^2 }  \lambda^2 \notag \\
	& \quad + 576 p^2\alpha_y^2\eta^2\Big(L_{f}^2+ L_{g, 1}^2\Big)\lambda^2\left(\frac{1}{\beta_x N} + \frac{1}{\beta_y N} \right) \frac{19200\alpha_y^2\alpha_{x}\eta (L_f^2+L_{g,1}^2)^2}{ \mu^2  }\lambda^2 \notag \\
	& \quad +  576 p^2\alpha_y^2\eta^2\Big(L_{f}^2+ L_{g, 1}^2\Big)\lambda^2 \frac{1}{\beta_z N}\frac{4818\alpha_z^2\alpha_{x} \eta  (L_f^2+L_{g,1}^2)^2}{\mu^2 }\lambda^2 \notag \\
	& \quad +   576 p^2\alpha_y^2\eta^2 \Big(L_{f}^2+ L_{g, 1}^2\Big)\lambda^2 \frac{4080\alpha^3_{x}p^2\eta^3(L_f^2+L_{g,1}^2)^2}{\mu^2} \lambda^2  \notag \\
	& \quad +  576 p^2\alpha_y^2\eta^2 \Big(L_{f}^2+ L_{g, 1}^2\Big)\lambda^2 \frac{2420\alpha_{x}\alpha_y^2p^2\eta^3 (L_f^2+L_{g,1}^2)^2}{\mu^2} \lambda^2 \notag \\
	& \quad +   576 p^2\alpha_y^2\eta^2\Big(L_{f}^2+ L_{g, 1}^2\Big)\lambda^2  \frac{160 \alpha_{x}\alpha_z^2p^2\eta^3(L_f^2+L_{g,1}^2)^2}{\mu^2} \lambda^2  \notag \\
	& \leq  0 \ . 
\end{align}
This can be done by setting
\begin{align}\label{eq:m-y}
	&   \left(\frac{1}{\beta_x N} + \frac{1}{\beta_y N} \right) \frac{19200\alpha_y^2\alpha_{x}\eta (L_f^2+L_{g,1}^2)^2}{ \mu^2  }\lambda^2 \leq  \frac{3\alpha_y\alpha_{x}\eta(L_f^2+L_{g,1}^2)}{\mu } \lambda  \ ,     \notag \\
	&  576 p^2\alpha_y^2\eta^2\Big(L_{f}^2+ L_{g, 1}^2\Big)\lambda^2\left(\frac{1}{\beta_x N} + \frac{1}{\beta_y N} + \frac{1}{\beta_zN}\right)\frac{28800\alpha_x^3\eta (L_{f}^{2}+L_{g, 1}^{2})^2}{\mu^2 }  \lambda^2 \leq  \frac{3\alpha_y\alpha_{x}\eta(L_f^2+L_{g,1}^2)}{\mu } \lambda  \ ,     \notag \\
	&  576 p^2\alpha_y^2\eta^2\Big(L_{f}^2+ L_{g, 1}^2\Big)\lambda^2\left(\frac{1}{\beta_x N} + \frac{1}{\beta_y N} \right) \frac{19200\alpha_y^2\alpha_{x}\eta (L_f^2+L_{g,1}^2)^2}{ \mu^2  }\lambda^2 \leq  \frac{3\alpha_y\alpha_{x}\eta(L_f^2+L_{g,1}^2)}{\mu } \lambda  \ ,     \notag \\
	& 576 p^2\alpha_y^2\eta^2\Big(L_{f}^2+ L_{g, 1}^2\Big)\lambda^2 \frac{1}{\beta_z N}\frac{4818\alpha_z^2\alpha_{x} \eta  (L_f^2+L_{g,1}^2)^2}{\mu^2 }\lambda^2\leq  \frac{3\alpha_y\alpha_{x}\eta(L_f^2+L_{g,1}^2)}{\mu } \lambda  \ ,     \notag \\
	&   576 p^2\alpha_y^2\eta^2 \Big(L_{f}^2+ L_{g, 1}^2\Big)\lambda^2 \frac{4080\alpha^3_{x}p^2\eta^3(L_f^2+L_{g,1}^2)^2}{\mu^2} \lambda^2 \leq  \frac{3\alpha_y\alpha_{x}\eta(L_f^2+L_{g,1}^2)}{\mu } \lambda  \ ,     \notag \\
	&  576 p^2\alpha_y^2\eta^2 \Big(L_{f}^2+ L_{g, 1}^2\Big)\lambda^2 \frac{2420\alpha_{x}\alpha_y^2p^2\eta^3 (L_f^2+L_{g,1}^2)^2}{\mu^2} \lambda^2 \leq  \frac{3\alpha_y\alpha_{x}\eta(L_f^2+L_{g,1}^2)}{\mu } \lambda  \ ,     \notag \\
	&   576 p^2\alpha_y^2\eta^2\Big(L_{f}^2+ L_{g, 1}^2\Big)\lambda^2  \frac{160 \alpha_{x}\alpha_z^2p^2\eta^3(L_f^2+L_{g,1}^2)^2}{\mu^2} \lambda^2  \leq  \frac{3\alpha_y\alpha_{x}\eta(L_f^2+L_{g,1}^2)}{\mu } \lambda  \ .   
\end{align}
By solving the first inequality in Eq.~(\ref{eq:m-y}), we can obtain
\begin{align}
	& \alpha_y  \leq   \frac{ \mu  }{6400 \lambda(L_f+L_{g,1})^2}\Bigg/  \left(\frac{1}{\beta_x N} + \frac{1}{\beta_y N} \right)   \  . 
\end{align}
To solve the second  inequality in Eq.~(\ref{eq:m-y}), we enforce
\begin{align}
	& \quad 576 p^2\alpha_y^2\eta^2\Big(L_{f}^2+ L_{g, 1}^2\Big)\lambda^2\left(\frac{1}{\beta_x N} + \frac{1}{\beta_y N} + \frac{1}{\beta_zN}\right)\frac{28800\alpha_x^3\eta (L_{f}^{2}+L_{g, 1}^{2})^2}{\mu^2 }  \lambda^2 \notag \\
	& \leq  576 \times 28800\alpha_y\alpha_x^2 p^2\eta^2\lambda^4\left(\frac{1}{\beta_x N} + \frac{1}{\beta_y N} + \frac{1}{\beta_zN}\right)\frac{ (L_{f}^{2}+L_{g, 1}^{2})^3}{\mu^2 } \alpha_y\alpha_{x}\eta \notag \\
	& \leq  576 \times 28800\alpha_y\frac{\alpha^2_y\mu^4}{360^2  (L_f+ L_{g, 1})^4} p^2\eta^2\lambda^4\left(\frac{1}{\beta_x N} + \frac{1}{\beta_y N} + \frac{1}{\beta_zN}\right)\frac{ (L_{f}^{2}+L_{g, 1}^{2})^3}{\mu^2 } \alpha_y\alpha_{x}\eta \notag \\
	&  \leq  \frac{3\alpha_y\alpha_{x}\eta(L_f^2+L_{g,1}^2)}{\mu } \lambda  \  .
\end{align}
Then, we can obtain
\begin{align}
	&\alpha_y \leq \frac{1}{10 \lambda (p\eta)^{2/3}}  \frac{  1}  {\mu }\Bigg/  \left(\frac{1}{\beta_x N} + \frac{1}{\beta_y N} + \frac{1}{\beta_zN}\right)^{1/3} \  .
\end{align}
To solve the third  inequality in Eq.~(\ref{eq:m-y}), we enforce
\begin{align}
	& \quad 576 p^2\alpha_y^2\eta^2\Big(L_{f}^2+ L_{g, 1}^2\Big)\lambda^2\left(\frac{1}{\beta_x N} + \frac{1}{\beta_y N} \right) \frac{19200\alpha_y^2\alpha_{x}\eta (L_f^2+L_{g,1}^2)^2}{ \mu^2  }\lambda^2 \notag \\
	& \leq 576 \times 19200\alpha_y^3 p^2\eta^2\lambda^4\left(\frac{1}{\beta_x N} + \frac{1}{\beta_y N} \right) \frac{ (L_f^2+L_{g,1}^2)^3}{ \mu^2  }\alpha_y\alpha_{x}\eta\notag \\
	&  \leq  \frac{3\alpha_y\alpha_{x}\eta(L_f^2+L_{g,1}^2)}{\mu } \lambda  \  . 
\end{align}
Then, we can obtain
\begin{align}
	& \alpha_y \leq  \frac{1}{ 200\lambda(p\eta)^{2/3}} \frac{ \mu^{1/3}   }{ (L_f+L_{g,1})^{4/3}} \Bigg/  \left(\frac{1}{\beta_x N} + \frac{1}{\beta_y N} \right)^{1/3} \  . 
\end{align}
To solve the fourth  inequality in Eq.~(\ref{eq:m-y}), we enforce
\begin{align}
	& \quad 576 p^2\alpha_y^2\eta^2\Big(L_{f}^2+ L_{g, 1}^2\Big)\lambda^2 \frac{1}{\beta_z N}\frac{4818\alpha_z^2\alpha_{x} \eta  (L_f^2+L_{g,1}^2)^2}{\mu^2 }\lambda^2 \notag \\
	& \leq 576\times 4818\alpha_z^2 \alpha_y\eta^2  p^2\lambda^4 \frac{1}{\beta_z N}\frac{ (L_f^2+L_{g,1}^2)^3}{\mu^2 }  \alpha_y\alpha_{x} \eta \notag \\ 
	& \leq 576\times 4818\alpha_z^2  \frac{1}{2\lambda \mu  (p\eta)^{4/3}  } \eta^2  p^2\lambda^4 \frac{1}{\beta_z N}\frac{ (L_f^2+L_{g,1}^2)^3}{\mu^2 }  \alpha_y\alpha_{x} \eta \notag \\ 
	& \leq  \frac{3\alpha_y\alpha_{x}\eta(L_f^2+L_{g,1}^2)}{\mu } \lambda  \  ,
\end{align}
where the third step follows from Eq.~(\ref{eq:y-an-ub}). Then, we can obtain
\begin{align}
	\alpha_z \leq \frac{1   }{1000 (p\eta)^{1/3} \lambda} \frac{\mu }{ (L_f^2+L_{g,1}^2)} \Bigg/\left( \frac{1}{\beta_z N} \right)^{1/2}  \    .
\end{align}
To solve the fifth  inequality in Eq.~(\ref{eq:m-y}), we enforce
\begin{align}
	& \quad 576 p^2\alpha_y^2\eta^2 \Big(L_{f}^2+ L_{g, 1}^2\Big)\lambda^2 \frac{4080\alpha^3_{x}p^2\eta^3(L_f^2+L_{g,1}^2)^2}{\mu^2} \lambda^2 \notag \\
	& \leq 576\times 4080\alpha^2_{x}\alpha_y p^4\eta^4  \lambda^4 \frac{(L_f^2+L_{g,1}^2)^3}{\mu^2}  \alpha_{x}\alpha_{y}\eta \notag \\
	& \leq 576\times 4080\frac{\alpha^2_y\mu^4}{360^2  (L_f+ L_{g, 1})^4}\alpha_y p^4\eta^4  \lambda^4 \frac{(L_f+L_{g,1})^6}{\mu^2}  \alpha_{x}\alpha_{y}\eta \notag \\
	&  \leq  \frac{3\alpha_y\alpha_{x}\eta(L_f^2+L_{g,1}^2)}{\mu } \lambda  \ . 
\end{align}
Then, we can obtain
\begin{align}\label{eq:y-an-ub}
	& \alpha_y \leq  \frac{1}{2\lambda \mu  (p\eta)^{4/3}  }   \ . 
\end{align}
To solve the sixth  inequality in Eq.~(\ref{eq:m-y}), we enforce
\begin{align}
	& \quad  576 p^2\alpha_y^2\eta^2 \Big(L_{f}^2+ L_{g, 1}^2\Big)\lambda^2 \frac{2420\alpha_{x}\alpha_y^2p^2\eta^3 (L_f^2+L_{g,1}^2)^2}{\mu^2} \lambda^2  \notag \\
	& \leq 576 \times 2420 \alpha_y^3  p^4\eta^4 \lambda^4 \frac{ (L_f^2+L_{g,1}^2)^3}{\mu^2} \alpha_{x}\alpha_{y}\eta  \notag \\ 
	& \leq  \frac{3\alpha_y\alpha_{x}\eta(L_f^2+L_{g,1}^2)}{\mu } \lambda  \  .
\end{align}
Then, we can obtain
\begin{align}
	&\alpha_y  \leq  \frac{1}{ 100 (p\eta)^{4/3} \lambda} \frac{\mu^{1/3} } { (L_f+L_{g,1})^{4/3}}  \  .
\end{align}
To solve the last inequality in 
Eq.~(\ref{eq:m-y}), we enforce
\begin{align}
	& \quad  576 p^2\alpha_y^2\eta^2\Big(L_{f}^2+ L_{g, 1}^2\Big)\lambda^2  \frac{160 \alpha_{x}\alpha_z^2p^2\eta^3(L_f^2+L_{g,1}^2)^2}{\mu^2} \lambda^2 \notag \\
	& \leq 576 \times 160 \alpha_z^2\alpha_yp^4\eta^4\lambda^4  \frac{(L_f^2+L_{g,1}^2)^3}{\mu^2} \alpha_{x} \alpha_y\eta \notag \\
		& \leq 576 \times 160 \alpha_z^2\frac{1}{2\lambda \mu  (p\eta)^{4/3}  } p^4\eta^4\lambda^4  \frac{(L_f^2+L_{g,1}^2)^3}{\mu^2} \alpha_{x} \alpha_y\eta \notag \\
	&   \leq  \frac{3\alpha_y\alpha_{x}\eta(L_f^2+L_{g,1}^2)}{\mu } \lambda  \ .   
\end{align}
Then, we can obtain
\begin{align}
	& \alpha_z \leq  \frac{ 1   }  {1000 (p\eta)^{4/3}\lambda}  \frac{\mu}{(L_f^2+L_{g,1}^2)} \ .  
\end{align}
In summary, by setting
\begin{align}
	& \alpha_{y} \leq \min\Big\{ \frac{ \mu  }{6400 \lambda(L_f+L_{g,1})^2}\Bigg/  \left(\frac{1}{\beta_x N} + \frac{1}{\beta_y N} \right)  , \frac{1}{10 \lambda (p\eta)^{2/3}}  \frac{  1}  {\mu }\Bigg/  \left(\frac{1}{\beta_x N} + \frac{1}{\beta_y N} + \frac{1}{\beta_zN}\right)^{1/3} , \notag \\
	& \quad \quad \frac{1}{ 200\lambda(p\eta)^{2/3}} \frac{ \mu^{1/3}   }{ (L_f+L_{g,1})^{4/3}} \Bigg/  \left(\frac{1}{\beta_x N} + \frac{1}{\beta_y N} \right)^{1/3}, \frac{1}{2\lambda \mu  (p\eta)^{4/3}  } , \frac{1}{ 100 (p\eta)^{4/3} \lambda} \frac{\mu^{1/3} } { (L_f+L_{g,1})^{4/3}}  \Big\} \ , \notag \\
	&  \alpha_{z} \leq \min\left\{\frac{1   }{1000 (p\eta)^{1/3} \lambda} \frac{\mu }{ (L_f^2+L_{g,1}^2)} \Bigg/\left( \frac{1}{\beta_z N} \right)^{1/2} , \frac{ 1   }  {1000 (p\eta)^{4/3}\lambda}  \frac{\mu}{(L_f^2+L_{g,1}^2)}  \right\} \ , 
\end{align}
we can  eliminate  $ \frac{1}{T}\sum_{t=0}^{T-1}\mathbb{E}[\| \bar{m}_{y, t} \|^2]$

Moreover, to eliminate  $ \frac{1}{T}\sum_{t=0}^{T-1}\mathbb{E}[\| \bar{m}_{z, t} \|^2]$, we enforce
\begin{align}
	&\quad  - \frac{3\eta\alpha^2_z\lambda}{4} c_2+ C_3+ 576 p^2\alpha_z^2\eta^2 \Big(L_{f}^2+ L_{g, 1}^2\lambda^2\Big)C_7 \notag \\
	& \leq - \frac{3\eta\alpha^2_z\lambda}{4} c_2+ C_3+ 576 p^2\alpha_z^2\eta^2 \Big(L_{f}^2+ L_{g, 1}^2\Big)\lambda^2C_7 \notag \\
	& \leq - \frac{27\alpha_z\alpha_{x}\eta}{2} \frac{L_{g,1}^2}{ \mu} \lambda+  \frac{1}{\beta_z N}\frac{4818\alpha_z^2\alpha_{x} \eta  (L_f^2+L_{g,1}^2)^2}{\mu^2 }\lambda^2\notag \\
	& \quad + 576 p^2\alpha_z^2\eta^2 \Big(L_{f}^2+ L_{g, 1}^2\Big)\lambda^2\left(\frac{1}{\beta_x N} + \frac{1}{\beta_y N} + \frac{1}{\beta_zN}\right)\frac{28800\alpha_x^3\eta (L_{f}^{2}+L_{g, 1}^{2})^2}{\mu^2 }  \lambda^2 \notag \\
	& \quad + 576 p^2\alpha_z^2\eta^2 \Big(L_{f}^2+ L_{g, 1}^2\Big)\lambda^2\left(\frac{1}{\beta_x N} + \frac{1}{\beta_y N} \right) \frac{19200\alpha_y^2\alpha_{x}\eta (L_f^2+L_{g,1}^2)^2}{ \mu^2  }\lambda^2 \notag \\
	& \quad +  576 p^2\alpha_z^2\eta^2 \Big(L_{f}^2+ L_{g, 1}^2\Big)\lambda^2 \frac{1}{\beta_z N}\frac{4818\alpha_z^2\alpha_{x} \eta  (L_f^2+L_{g,1}^2)^2}{\mu^2 }\lambda^2 \notag \\
	& \quad +   576 p^2\alpha_z^2\eta^2 \Big(L_{f}^2+ L_{g, 1}^2\Big)\lambda^2 \frac{4080\alpha^3_{x}p^2\eta^3(L_f^2+L_{g,1}^2)^2}{\mu^2} \lambda^2  \notag \\
	& \quad +  576 p^2\alpha_z^2\eta^2 \Big(L_{f}^2+ L_{g, 1}^2\Big)\lambda^2 \frac{2420\alpha_{x}\alpha_y^2p^2\eta^3 (L_f^2+L_{g,1}^2)^2}{\mu^2} \lambda^2 \notag \\
	& \quad +   576 p^2\alpha_z^2\eta^2\Big(L_{f}^2+ L_{g, 1}^2\Big)\lambda^2  \frac{160 \alpha_{x}\alpha_z^2p^2\eta^3(L_f^2+L_{g,1}^2)^2}{\mu^2} \lambda^2  \notag \\
	& \leq  0  \ . 
\end{align}
This can be done by setting
\begin{align} \label{eq:m-z}
	&  \frac{1}{\beta_z N}\frac{4818\alpha_z^2\alpha_{x} \eta  (L_f^2+L_{g,1}^2)^2}{\mu^2 }\lambda^2 \leq  \frac{\alpha_z\alpha_{x}\eta L_{g,1}^2}{ \mu} \lambda \ ,   \notag \\
	&  576 p^2\alpha_z^2\eta^2 \Big(L_{f}^2+ L_{g, 1}^2\Big)\lambda^2\left(\frac{1}{\beta_x N} + \frac{1}{\beta_y N} + \frac{1}{\beta_zN}\right)\frac{28800\alpha_x^3\eta (L_{f}^{2}+L_{g, 1}^{2})^2}{\mu^2 }  \lambda^2  \leq  \frac{\alpha_z\alpha_{x}\eta L_{g,1}^2}{ \mu} \lambda \ ,   \notag \\
	& 576 p^2\alpha_z^2\eta^2 \Big(L_{f}^2+ L_{g, 1}^2\Big)\lambda^2\left(\frac{1}{\beta_x N} + \frac{1}{\beta_y N} \right) \frac{19200\alpha_y^2\alpha_{x}\eta (L_f^2+L_{g,1}^2)^2}{ \mu^2  }\lambda^2  \leq  \frac{\alpha_z\alpha_{x}\eta L_{g,1}^2}{ \mu} \lambda \ ,   \notag \\
	&  576 p^2\alpha_z^2\eta^2 \Big(L_{f}^2+ L_{g, 1}^2\Big)\lambda^2 \frac{1}{\beta_z N}\frac{4818\alpha_z^2\alpha_{x} \eta  (L_f^2+L_{g,1}^2)^2}{\mu^2 }\lambda^2  \leq  \frac{\alpha_z\alpha_{x}\eta L_{g,1}^2}{ \mu} \lambda \ ,   \notag \\
	&    576 p^2\alpha_z^2\eta^2 \Big(L_{f}^2+ L_{g, 1}^2\Big)\lambda^2 \frac{4080\alpha^3_{x}p^2\eta^3(L_f^2+L_{g,1}^2)^2}{\mu^2} \lambda^2  \leq  \frac{\alpha_z\alpha_{x}\eta L_{g,1}^2}{ \mu} \lambda \ ,   \notag \\
	&   576 p^2\alpha_z^2\eta^2 \Big(L_{f}^2+ L_{g, 1}^2\Big)\lambda^2 \frac{2420\alpha_{x}\alpha_y^2p^2\eta^3 (L_f^2+L_{g,1}^2)^2}{\mu^2} \lambda^2 \leq  \frac{\alpha_z\alpha_{x}\eta L_{g,1}^2}{ \mu} \lambda \ ,   \notag \\
	&  576 p^2\alpha_z^2\eta^2\Big(L_{f}^2+ L_{g, 1}^2\Big)\lambda^2  \frac{160 \alpha_{x}\alpha_z^2p^2\eta^3(L_f^2+L_{g,1}^2)^2}{\mu^2} \lambda^2   \leq  \frac{\alpha_z\alpha_{x}\eta L_{g,1}^2}{ \mu} \lambda \  . 
\end{align}
By solving the first inequality in Eq.~(\ref{eq:m-z}), we obtain
\begin{align}
	& \alpha_z \leq  \frac{1}{ 4818\lambda}  \frac{\mu  L_{g,1}^2 }{ (L_f^2+L_{g,1}^2)^2} \Bigg/\frac{1}{\beta_z N}  \ . 
\end{align}
To solve the second  inequality in Eq.~(\ref{eq:m-z}), we enforce
\begin{align}
	& \quad 576 p^2\alpha_z^2\eta^2 \Big(L_{f}^2+ L_{g, 1}^2\Big)\lambda^2\left(\frac{1}{\beta_x N} + \frac{1}{\beta_y N} + \frac{1}{\beta_zN}\right)\frac{28800\alpha_x^3\eta (L_{f}^{2}+L_{g, 1}^{2})^2}{\mu^2 }  \lambda^2  \notag \\
	& \leq  576 \times 28800\alpha_z \alpha_x^2 p^2\eta^2 \lambda^4\left(\frac{1}{\beta_x N} + \frac{1}{\beta_y N} + \frac{1}{\beta_zN}\right)\frac{ (L_{f}^{2}+L_{g, 1}^{2})^3}{\mu^2 } \alpha_z \alpha_x\eta  \notag \\
	& \leq  576 \times 28800\alpha_z \frac{\alpha^2_z \mu^4}{60^2 (L_f+ L_{g, 1})^4} p^2\eta^2 \lambda^4\left(\frac{1}{\beta_x N} + \frac{1}{\beta_y N} + \frac{1}{\beta_zN}\right)\frac{ (L_{f}^{2}+L_{g, 1}^{2})^3}{\mu^2 } \alpha_z \alpha_x\eta  \notag \\
	&  \leq  \frac{\alpha_z\alpha_{x}\eta L_{g,1}^2}{ \mu} \lambda \  . 
\end{align}
Then, we can obtain
\begin{align}
	&  \alpha_z    \leq  \frac{1 }{ 20(p\eta)^{2/3}\lambda} \frac{L_{g,1}} { \mu} \Bigg/\left(\frac{1}{\beta_x N} + \frac{1}{\beta_y N} + \frac{1}{\beta_zN}\right) ^{1/3} \  . 
\end{align}
To solve the third  inequality in Eq.~(\ref{eq:m-z}), we enforce
\begin{align}
	& \quad  576 p^2\alpha_z^2\eta^2 \Big(L_{f}^2+ L_{g, 1}^2\Big)\lambda^2\left(\frac{1}{\beta_x N} + \frac{1}{\beta_y N} \right) \frac{19200\alpha_y^2\alpha_{x}\eta (L_f^2+L_{g,1}^2)^2}{ \mu^2  }\lambda^2 \notag \\
	& \leq 576 \times 19200\alpha_z \alpha_y^2 p^2\eta^2\lambda^4\left(\frac{1}{\beta_x N} + \frac{1}{\beta_y N} \right) \frac{ (L_f^2+L_{g,1}^2)^3}{ \mu^2  }\alpha_z \alpha_{x}\eta  \notag \\
	& \leq 576 \times 19200\alpha_z  \frac{ \mu^2  }{6400^2 \lambda^2(L_f+L_{g,1})^4}\Bigg/  \left(\frac{1}{\beta_x N} + \frac{1}{\beta_y N} \right)^2  p^2\eta^2\lambda^4\left(\frac{1}{\beta_x N} + \frac{1}{\beta_y N} \right) \frac{ (L_f^2+L_{g,1}^2)^3}{ \mu^2  }\alpha_z \alpha_{x}\eta  \notag \\
	&   \leq  \frac{\alpha_z\alpha_{x}\eta L_{g,1}^2}{ \mu} \lambda \  .
\end{align}
Then, we can obtain
\begin{align}
&\alpha_z  \leq \frac{ 4}{ p^2\eta^2\lambda}  \frac{ L_{g,1}^2}{ \mu (L_f+L_{g,1})^2 }\left(\frac{1}{\beta_x N} + \frac{1}{\beta_y N} \right)   \  .
\end{align}
To solve the fourth  inequality in Eq.~(\ref{eq:m-z}), we enforce
\begin{align}
	& \quad 576 p^2\alpha_z^2\eta^2 \Big(L_{f}^2+ L_{g, 1}^2\Big)\lambda^2 \frac{1}{\beta_z N}\frac{4818\alpha_z^2\alpha_{x} \eta  (L_f^2+L_{g,1}^2)^2}{\mu^2 }\lambda^2  \notag \\
	& \leq 576 \times 4818\alpha_z^3  p^2\eta^2 \lambda^4 \frac{1}{\beta_z N}\frac{ (L_f^2+L_{g,1}^2)^3}{\mu^2 } \alpha_z \alpha_{x} \eta  \notag \\
	&  \leq  \frac{\alpha_z\alpha_{x}\eta L_{g,1}^2}{ \mu} \lambda \  .
\end{align}
Then, we can obtain
\begin{align}
	&  \alpha_z  \leq  \frac{ 1}{150 (p\eta)^{2/3} \lambda} \frac{\mu^{1/3} L_{g,1}^{2/3} }{ L_f^2+L_{g,1}^2}\Bigg/\left(  \frac{1}{\beta_z N}\right)^{1/3}  \  .
\end{align}
To solve the fifth inequality in Eq.~(\ref{eq:m-z}), we enforce
\begin{align}
	& \quad 576 p^2\alpha_z^2\eta^2 \Big(L_{f}^2+ L_{g, 1}^2\Big)\lambda^2 \frac{4080\alpha^3_{x}p^2\eta^3(L_f^2+L_{g,1}^2)^2}{\mu^2} \lambda^2    \notag \\
	& \leq  576\times 4080\alpha_z\alpha^2_{x}p^4\eta^4 \lambda^4 \frac{(L_f^2+L_{g,1}^2)^3}{\mu^2} \alpha_z\alpha_{x}\eta   \notag \\
	& \leq  576\times 4080\alpha_z\frac{\alpha^2_z \mu^4}{60^2 (L_f+ L_{g, 1})^4} p^4\eta^4 \lambda^4 \frac{(L_f^2+L_{g,1}^2)^3}{\mu^2} \alpha_z\alpha_{x}\eta   \notag \\
	& \leq  \frac{\alpha_z\alpha_{x}\eta L_{g,1}^2}{ \mu} \lambda \  .
\end{align}
Then, we can obtain
\begin{align}
	&\alpha_z  \leq  \frac{1}{10(p\eta)^{4/3} \lambda}\frac{ L_{g,1}} {\mu}   \  .
\end{align}
To solve the sixth inequality in Eq.~(\ref{eq:m-z}), we enforce
\begin{align}
	& \quad 576 p^2\alpha_z^2\eta^2 \Big(L_{f}^2+ L_{g, 1}^2\Big)\lambda^2 \frac{2420\alpha_{x}\alpha_y^2p^2\eta^3 (L_f^2+L_{g,1}^2)^2}{\mu^2} \lambda^2  \notag \\
	& \leq  576\times 2420\alpha_z\alpha_y^2  p^4\eta^4 \lambda^4 \frac{(L_f^2+L_{g,1}^2)^3}{\mu^2} \alpha_{x}\alpha_z\eta  \notag \\
	& \leq  576\times 2420\alpha_z \frac{ \mu^2  }{6400^2 \lambda^2(L_f+L_{g,1})^4}\Bigg/  \left(\frac{1}{\beta_x N} + \frac{1}{\beta_y N} \right)^2p^4\eta^4 \lambda^4 \frac{(L_f^2+L_{g,1}^2)^3}{\mu^2} \alpha_{x}\alpha_z\eta  \notag \\
	& \leq  \frac{\alpha_z\alpha_{x}\eta L_{g,1}^2}{ \mu} \lambda \  .
\end{align}
Then, we can obtain
\begin{align}
	& \alpha_z    \leq  \frac{30}{ p^4\eta^4 \lambda  } \frac{ L_{g,1}^2}{\mu(L_f+L_{g,1})^2} \left(\frac{1}{\beta_x N} + \frac{1}{\beta_y N} \right)^2 \  .
\end{align}
To solve the last inequality in Eq.~(\ref{eq:m-z}), we enforce
\begin{align}
	& \quad  576 p^2\alpha_z^2\eta^2\Big(L_{f}^2+ L_{g, 1}^2\Big)\lambda^2  \frac{160 \alpha_{x}\alpha_z^2p^2\eta^3(L_f^2+L_{g,1}^2)^2}{\mu^2} \lambda^2    \notag \\
	& \leq 576\times 160 \alpha_z^3 p^4\eta^4\lambda^4  \frac{(L_f^2+L_{g,1}^2)^3}{\mu^2}  \alpha_{x}\alpha_z \eta   \notag \\
	& \leq  \frac{\alpha_z\alpha_{x}\eta L_{g,1}^2}{ \mu} \lambda \  . 
\end{align}
Then, we can obtain
\begin{align}
	&   \alpha_z  \leq \frac{1}{ 50(p\eta)^{4/3}\lambda}  \frac{\mu  L_{g,1}^{2/3}}  {(L_f^2+L_{g,1}^2)} \  . 
\end{align}
In summary, by setting
\begin{align}
	& \alpha_z\leq \min\Big\{\frac{1}{ 4818\lambda}  \frac{\mu  L_{g,1}^2 }{ (L_f^2+L_{g,1}^2)^2} \Bigg/\frac{1}{\beta_z N},  \frac{1 }{ 20(p\eta)^{2/3}\lambda} \frac{L_{g,1}} { \mu} \Bigg/\left(\frac{1}{\beta_x N} + \frac{1}{\beta_y N} + \frac{1}{\beta_zN}\right) ^{1/3}  ,  \notag \\
	& \quad \quad  \frac{ 4}{ p^2\eta^2\lambda}  \frac{ L_{g,1}^2}{ \mu (L_f+L_{g,1})^2 }\left(\frac{1}{\beta_x N} + \frac{1}{\beta_y N} \right) , \frac{ 1}{150 (p\eta)^{2/3} \lambda} \frac{\mu^{1/3} L_{g,1}^{2/3} }{ L_f^2+L_{g,1}^2}\Bigg/\left(  \frac{1}{\beta_z N}\right)^{1/3}  ,  \notag \\
	& \quad \quad   \frac{1}{10(p\eta)^{4/3} \lambda}\frac{ L_{g,1}} {\mu},  \frac{1}{ 50(p\eta)^{4/3}\lambda}  \frac{\mu  L_{g,1}^{2/3}}  {(L_f^2+L_{g,1}^2)}  \Big\} \ ,  
\end{align}
we can eliminate  $ \frac{1}{T}\sum_{t=0}^{T-1}\mathbb{E}[\| \bar{m}_{z, t} \|^2]$. 

As a result, by setting
\begin{align}\label{eq:hyperparameters}
	& \alpha_x\leq \min \Bigg\{\frac{\alpha_y\mu^2}{360  (L_f+ L_{g, 1})^2}, \frac{\alpha_z \mu^2}{60 (L_f+ L_{g, 1})^2}, \frac{\mu }{100 \lambda (L_{f}+L_{g, 1})^{2}}\Bigg/\sqrt{\frac{1}{\beta_x N} + \frac{1}{\beta_y N} + \frac{1}{\beta_zN}} ,  \notag \\
	& \quad \quad \frac{1}{160 \lambda (p \eta)^{1/2}} \frac{\mu^{1/2} }{ (L_{f}+L_{g, 1})^{3/2}} \Bigg/ \left(\frac{1}{\beta_x N} + \frac{1}{\beta_y N} + \frac{1}{\beta_zN}\right)^{1/4}, \frac{1}{100\lambda p\eta } \frac{\mu^{1/2}}{(L_f+L_{g,1})^{3/2}}, 10\Bigg\}  \ , \notag \\
	& \alpha_y\leq \min \Bigg\{ \frac{1}{8\lambda ( p\eta)^{1/2} }\frac{1 }{\sqrt{\mu (L_f+ L_{g, 1})}}\Bigg/ \left(\frac{1}{\beta_x N} + \frac{1}{\beta_y N} \right)^{1/4}, \frac{1}{ \lambda p\eta }\frac{1}{ \sqrt{\mu(L_f+L_{g,1})}} , \notag \\
    & \quad \quad  \frac{ \mu  }{6400 \lambda(L_f+L_{g,1})^2}\Bigg/  \left(\frac{1}{\beta_x N} + \frac{1}{\beta_y N} \right)  , \frac{1}{10 \lambda (p\eta)^{2/3}}  \frac{  1}  {\mu }\Bigg/  \left(\frac{1}{\beta_x N} + \frac{1}{\beta_y N} + \frac{1}{\beta_zN}\right)^{1/3} , \notag \\
	& \quad \quad \frac{1}{ 200\lambda(p\eta)^{2/3}} \frac{ \mu^{1/3}   }{ (L_f+L_{g,1})^{4/3}} \Bigg/  \left(\frac{1}{\beta_x N} + \frac{1}{\beta_y N} \right)^{1/3}, \frac{1}{2\lambda \mu  (p\eta)^{4/3}  } , \frac{1}{ 100 (p\eta)^{4/3} \lambda} \frac{\mu^{1/3} } { (L_f+L_{g,1})^{4/3}}, \frac{1}{6(L_f + \lambda L_{g, 1})}, 10 \Bigg\}  \ , \notag \\
	& \alpha_z\leq \min \Bigg\{\frac{1}{ \lambda (p \eta)^{1/2}} \frac{1}{ \sqrt{\mu (L_f+L_{g,1})}}\Bigg/\left(\frac{1}{\beta_z N}\right)^{1/4} ,  \frac{1}{ \lambda p\eta } \frac{1}{ \sqrt{\mu(L_f+L_{g,1})}}  , \notag \\
    & \quad \quad \frac{1   }{1000 (p\eta)^{1/3} \lambda} \frac{\mu }{ (L_f^2+L_{g,1}^2)} \Bigg/\left( \frac{1}{\beta_z N} \right)^{1/2} , \frac{ 1   }  {1000 (p\eta)^{4/3}\lambda}  \frac{\mu}{(L_f^2+L_{g,1}^2)} , \notag \\
    & \quad \quad \frac{1}{ 4818\lambda}  \frac{\mu  L_{g,1}^2 }{ (L_f^2+L_{g,1}^2)^2} \Bigg/\frac{1}{\beta_z N},  \frac{1 }{ 20(p\eta)^{2/3}\lambda} \frac{L_{g,1}} { \mu} \Bigg/\left(\frac{1}{\beta_x N} + \frac{1}{\beta_y N} + \frac{1}{\beta_zN}\right) ^{1/3}  ,  \notag \\
	& \quad \quad  \frac{ 4}{ p^2\eta^2\lambda}  \frac{ L_{g,1}^2}{ \mu (L_f+L_{g,1})^2 }\left(\frac{1}{\beta_x N} + \frac{1}{\beta_y N} \right) , \frac{ 1}{150 (p\eta)^{2/3} \lambda} \frac{\mu^{1/3} L_{g,1}^{2/3} }{ L_f^2+L_{g,1}^2}\Bigg/\left(  \frac{1}{\beta_z N}\right)^{1/3}  ,  \notag \\
	& \quad \quad   \frac{1}{10(p\eta)^{4/3} \lambda}\frac{ L_{g,1}} {\mu},  \frac{1}{ 50(p\eta)^{4/3}\lambda}  \frac{\mu  L_{g,1}^{2/3}}  {(L_f^2+L_{g,1}^2)}, \frac{1}{6\lambda L_{g, 1}},10\Bigg\}  \  , \notag \\
    & \eta \leq \min\Bigg\{\frac{1}{500 p \sqrt{L_{f}^2+ L_{g, 1}^2 \lambda^2} }, \frac{1}{2\alpha_x L}, \frac{1}{\sqrt{\beta_x}}, \frac{1}{\sqrt{\beta_y}}, \frac{1}{\sqrt{\beta_z}}\Bigg\} \ , \notag \\
    & \beta_x\leq \frac{{L_f^2+ L_{g, 1}^{2}  \lambda^2}}{N} \ , \beta_y\leq \frac{{L_f^2+ L_{g, 1}^{2}  \lambda^2}}{N} \ ,  \beta_z\leq \frac{{L_f^2+ L_{g, 1}^{2}  \lambda^2}}{N} \ ,  \notag \\
    & \lambda \geq \max\left\{1, \frac{2L_f}{\mu}\right\} \ , 
\end{align}
we have 
\begin{align}
		& \quad  \frac{1}{T}\sum_{t=0}^{T-1}  \mathbb{E}[\|\nabla \mathcal{L}_{\lambda}^{*}(\bar{x}_{t})\|^2]  \notag  \\
		& \leq  \frac{2}{\alpha_{x}\eta}\frac{P_{0} - P_{T}}{T}  +\frac{2C_7}{\alpha_{x}\eta}\Big(576p^2\beta_x^2\eta^4( 1+  \lambda^2 )\delta^2 +576p^2\beta_y^2\eta^4(  1+  \lambda^2 )   \delta^2 + 576p^2\beta_z^2\eta^4 \lambda^2\delta^2  \notag  \\
		& \quad \quad + 576p^2\beta_x^2 \eta^4(  1+ \lambda^2)\sigma^{2} + 576p^2\beta_y^2\eta^4(  1 + \lambda^2  )\sigma^{2} +576p^2\beta_z^2\eta^4\lambda^2\sigma^{2}\Big)\notag  \\
		& \quad + 2\beta_x^2\eta^4\sigma^2 \frac{1}{N}c_3\frac{2}{\alpha_{x}\eta} + 2\beta_x^2\eta^4\sigma^2 \frac{1}{N}\lambda^2 c_4\frac{2}{\alpha_{x}\eta} + 2\beta_x^2\eta^4\sigma^2 \frac{1}{N}\lambda^2 c_5\frac{2}{\alpha_{x}\eta} \notag \\
        & \quad + 2\beta_y^2\eta^4\sigma^2 \frac{1}{N}c_6\frac{2}{\alpha_{x}\eta} + 2\beta_y^2\eta^4\sigma^2 \frac{1}{N}\lambda^2 c_7\frac{2}{\alpha_{x}\eta} + 2\beta_z^2\eta^4\sigma^2 \frac{1}{N}\lambda^2 c_8 \frac{2}{\alpha_{x}\eta}\notag \\
        & \leq  \frac{2}{\alpha_{x}\eta}\frac{P_{0} - P_{T}}{T} \notag  \\
        & \quad +2\Bigg[\left(\frac{1}{\beta_x N} + \frac{1}{\beta_y N} + \frac{1}{\beta_zN}\right)\frac{28800\alpha_x^2 (L_{f}^{2}+L_{g, 1}^{2})^2}{\mu^2 }  \lambda^2 +\left(\frac{1}{\beta_x N} + \frac{1}{\beta_y N} \right) \frac{19200\alpha_y^2(L_f^2+L_{g,1}^2)^2}{ \mu^2  }\lambda^2 \notag \\
	& \quad\quad  +  \frac{1}{\beta_z N}\frac{4818\alpha_z^2 (L_f^2+L_{g,1}^2)^2}{\mu^2 }\lambda^2 +   \frac{4080\alpha^2_{x}p^2\eta^2(L_f^2+L_{g,1}^2)^2}{\mu^2} \lambda^2  \notag \\
	& \quad \quad +  \frac{2420\alpha_y^2p^2\eta^2 (L_f^2+L_{g,1}^2)^2}{\mu^2} \lambda^2+    \frac{160 \alpha_z^2p^2\eta^2(L_f^2+L_{g,1}^2)^2}{\mu^2} \lambda^2 \Bigg]\notag \\
    & \quad \quad \times \Big(576p^2\beta_x^2\eta^4( 1+  \lambda^2 )\delta^2 +576p^2\beta_y^2\eta^4(  1+  \lambda^2 )   \delta^2 + 576p^2\beta_z^2\eta^4 \lambda^2\delta^2  \notag  \\
		& \quad \quad \quad + 576p^2\beta_x^2 \eta^4(  1+ \lambda^2)\sigma^{2} + 576p^2\beta_y^2\eta^4(  1 + \lambda^2  )\sigma^{2} +576p^2\beta_z^2\eta^4\lambda^2\sigma^{2}\Big)\notag  \\
		& \quad + 18\beta_x\eta^2\sigma^2 \frac{1}{N} + 18\beta_x\eta^2\sigma^2 \frac{1}{N}\lambda^2  + 18\beta_x\eta^2\sigma^2 \frac{1}{N}\lambda^2  \notag \\
        & \quad + 4800\beta_y\eta^2\sigma^2 \frac{1}{N}\frac{(L_f^2+L_{g,1}^2)}{\mu^2 } + 4800\beta_y\eta^2\sigma^2 \frac{1}{N}\lambda^2 \frac{(L_f^2+L_{g,1}^2)}{\mu^2 } + 4800\beta_z\eta^2\sigma^2 \frac{1}{N}\lambda^2 \frac{(L_f^2+L_{g,1}^2)}{\mu^2 }  \ . 
	\end{align}
When $t=0$, we have
\begin{align}\label{eq:remove-all-terms}
	& P_{0} = \mathbb{E}[\mathcal{L}_{\lambda}^{*}(\bar{x}_{0})] + c_1 \lambda \mathbb{E}[\|\bar{y}_{0} - y_{\lambda}^{*}(\bar{{x}}_{0})\|^2]  + c_2 \lambda \mathbb{E}[\|\bar{z}_{0} - y^{*}(\bar{{x}}_{0})\|^2] \notag  \\
	& \quad + c_3\mathbb{E}[\|\frac{1}{N}\sum_{n=1}^{N}\nabla_{1} f^{(n)}(x^{(n)}_{0}, y^{(n)}_{0})- \frac{1}{N}\sum_{n=1}^{N} u^{(n)}_{1, 0} \|^2]  \notag  + c_4 \lambda^2 \mathbb{E}[\|\frac{1}{N}\sum_{n=1}^{N}\nabla_{1} g^{(n)}(x^{(n)}_{0}, y^{(n)}_{0})- \frac{1}{N}\sum_{n=1}^{N} u^{(n)}_{2, 0} \|^2] \notag  \\
	& \quad + c_5\lambda^2\mathbb{E}[\|\frac{1}{N}\sum_{n=1}^{N}\nabla_{1} g^{(n)}(x^{(n)}_{0}, z^{(n)}_{0})- \frac{1}{N}\sum_{n=1}^{N} u^{(n)}_{3, 0} \|^2] \notag   + c_6\mathbb{E}[\|\frac{1}{N}\sum_{n=1}^{N} \nabla_{2} f^{(n)}({x}_{0}^{(n)}, {y}_{0}^{(n)}) -\frac{1}{N}\sum_{n=1}^{N} v^{(n)}_{1, 0}  \|^2] \notag  \\
	& \quad + c_7\lambda^2 \mathbb{E}[\|\frac{1}{N}\sum_{n=1}^{N}  \nabla_{2} g^{(n)}({x}_{0}^{(n)}, {y}_{0}^{(n)})-\frac{1}{N}\sum_{n=1}^{N}     v^{(n)}_{2, 0} \|^2]  + c_8\lambda^2 \mathbb{E}[\| \frac{1}{N}\sum_{n=1}^{N} \nabla_{2} g^{(n)}({x}^{(n)}_{0}, {z}^{(n)}_{0}) - \frac{1}{N}\sum_{n=1}^{N}{w}^{(n)}_{1, 0} \|^2]  \ \notag \\
	& \leq \mathbb{E}[\mathcal{L}_{\lambda}^{*}(\bar{x}_{0})] +  \frac{36\alpha_{x}(L_f^2+L_{g,1}^2)}{\alpha_y\mu }\lambda \mathbb{E}[\|\bar{y}_{0} - y_{\lambda}^{*}(\bar{{x}}_{0})\|^2]  + \frac{18\alpha_{x}L_{g,1}^2}{\alpha_z \mu}  \lambda \mathbb{E}[\|\bar{z}_{0} - y^{*}(\bar{{x}}_{0})\|^2] \notag  \\
	& \quad + \frac{9\alpha_{x}}{2\beta_x\eta}\frac{\sigma^2}{BN}  + \frac{9\alpha_{x}}{2\beta_x\eta} \lambda^2 \frac{\sigma^2}{BN} + \frac{9\alpha_{x} }{2\beta_x\eta} \lambda^2\frac{\sigma^2}{BN}   + \frac{1200\alpha_{x}(L_f^2+L_{g,1}^2)}{\mu^2 \beta_y\eta}\frac{\sigma^2}{BN}\notag  \\
	& \quad + \frac{1200\alpha_{x}(L_f^2+L_{g,1}^2)}{\mu^2 \beta_y\eta}\lambda^2 \frac{\sigma^2}{BN} +\frac{1200\alpha_{x}(L_f^2+L_{g,1}^2)}{\mu^2 \beta_z\eta}\lambda^2\frac{\sigma^2}{BN} \  ,
\end{align}
where $B$ is the batch size in the initial step. 
Then, we have
\begin{align}
& \quad  \frac{1}{T}\sum_{t=0}^{T-1}  \mathbb{E}[\|\nabla \mathcal{L}(\bar{x}_{t})\|^2]  \notag  \\
    & \leq 2\frac{1}{T}\sum_{t=0}^{T-1}  \mathbb{E}[\|\nabla \mathcal{L}_{\lambda}^{*}(\bar{x}_{t})\|^2] + 2\frac{1}{T}\sum_{t=0}^{T-1}  \mathbb{E}[\|\nabla \mathcal{L}_{\lambda}^{*}(\bar{x}_{t}) -\nabla \mathcal{L}(\bar{x}_{t}) \|^2] \notag  \\
    & \leq  \frac{4}{\alpha_{x}\eta}\frac{\mathbb{E}[\mathcal{L}_{\lambda}^{*}(\bar{x}_{0}) ]}{T}  +  \frac{4}{\eta T}\frac{36(L_f^2+L_{g,1}^2)}{\alpha_y\mu }\lambda \mathbb{E}[\|\bar{y}_{0} - y_{\lambda}^{*}(\bar{{x}}_{0})\|^2]  + \frac{4}{\eta T}\frac{18L_{g,1}^2}{\alpha_z \mu}  \lambda \mathbb{E}[\|\bar{z}_{0} - y^{*}(\bar{{x}}_{0})\|^2] \notag  \\
	& \quad + 2\frac{1}{T}\sum_{t=0}^{T-1}  \mathbb{E}[\|\nabla \mathcal{L}_{\lambda}^{*}(\bar{x}_{t}) -\nabla \mathcal{L}(\bar{x}_{t}) \|^2] + \frac{18}{\beta_x\eta^2 T}\frac{\sigma^2}{BN}  + \frac{36}{\beta_x\eta^2 T} \lambda^2 \frac{\sigma^2}{BN}   + \frac{4800(L_f^2+L_{g,1}^2)}{\mu^2 \beta_y\eta^2 T}\frac{\sigma^2}{BN}\notag  \\
	& \quad + \frac{4800(L_f^2+L_{g,1}^2)}{\mu^2 \beta_y\eta^2 T}\lambda^2 \frac{\sigma^2}{BN} +\frac{4800(L_f^2+L_{g,1}^2)}{\mu^2 \beta_z\eta^2 T}\lambda^2\frac{\sigma^2}{BN} \notag \\
        & \quad +4\Bigg[\left(\frac{1}{\beta_x N} + \frac{1}{\beta_y N} + \frac{1}{\beta_zN}\right)\frac{28800\alpha_x^2 (L_{f}^{2}+L_{g, 1}^{2})^2}{\mu^2 }  \lambda^2 +\left(\frac{1}{\beta_x N} + \frac{1}{\beta_y N} \right) \frac{19200\alpha_y^2(L_f^2+L_{g,1}^2)^2}{ \mu^2  }\lambda^2 \notag \\
	& \quad\quad  +  \frac{1}{\beta_z N}\frac{4818\alpha_z^2 (L_f^2+L_{g,1}^2)^2}{\mu^2 }\lambda^2 +   \frac{4080\alpha^2_{x}p^2\eta^2(L_f^2+L_{g,1}^2)^2}{\mu^2} \lambda^2  \notag \\
	& \quad \quad +  \frac{2420\alpha_y^2p^2\eta^2 (L_f^2+L_{g,1}^2)^2}{\mu^2} \lambda^2+    \frac{160 \alpha_z^2p^2\eta^2(L_f^2+L_{g,1}^2)^2}{\mu^2} \lambda^2 \Bigg]\notag \\
    & \quad \quad \times \Big(576p^2\beta_x^2\eta^4( 1+  \lambda^2 )\delta^2 +576p^2\beta_y^2\eta^4(  1+  \lambda^2 )   \delta^2 + 576p^2\beta_z^2\eta^4 \lambda^2\delta^2  \notag  \\
		& \quad \quad \quad + 576p^2\beta_x^2 \eta^4(  1+ \lambda^2)\sigma^{2} + 576p^2\beta_y^2\eta^4(  1 + \lambda^2  )\sigma^{2} +576p^2\beta_z^2\eta^4\lambda^2\sigma^{2}\Big)\notag  \\
		& \quad + 36\beta_x\eta^2\sigma^2 \frac{1}{N} + 36\beta_x\eta^2\sigma^2 \frac{1}{N}\lambda^2  + 36\beta_x\eta^2\sigma^2 \frac{1}{N}\lambda^2  \notag \\
        & \quad + 9600\beta_y\eta^2\sigma^2 \frac{1}{N}\frac{(L_f^2+L_{g,1}^2)}{\mu^2 } + 9600\beta_y\eta^2\sigma^2 \frac{1}{N}\lambda^2 \frac{(L_f^2+L_{g,1}^2)}{\mu^2 } + 9600\beta_z\eta^2\sigma^2 \frac{1}{N}\lambda^2 \frac{(L_f^2+L_{g,1}^2)}{\mu^2 }  \ .
	\end{align}

According to Lemma~\ref{lemma:approximation} and Eq.~(\ref{eq:hyperparameters}), to achieve the $\epsilon$-accuracy solution where $\epsilon\in (0, 1)$, we can set
\begin{align}
 & \lambda = O\left(\frac{\kappa^3}{\epsilon}\right) \ , \notag \\
    & \beta_x = O\left(\frac{1}{N}\right) \ , \beta_y = O\left(\frac{1}{N}\right) \ ,  \beta_z = O\left(\frac{1}{N}\right) \ , \notag\\
    & \alpha_x = O\left(\frac{1}{\lambda \kappa^3}\right) = O\left(\frac{\epsilon}{ \kappa^6}\right) \ , \alpha_y = O\left(\frac{1}{\lambda \kappa}\right) = O\left(\frac{\epsilon}{ \kappa^4}\right) \ , \alpha_z = O\left(\frac{1}{\lambda \kappa}\right) = O\left(\frac{\epsilon}{ \kappa^4}\right) \ ,\notag \\
    & \eta = O\left(\frac{N\epsilon^2}{\kappa^4}\right)\ ,  p =  O\left(\frac{\kappa}{N\epsilon}\right)  \ , B = O\left(\frac{\kappa^6}{N\epsilon^3}\right) \ , T = O\left(\frac{\kappa^{10}}{N\epsilon^5}\right) \ .  
\end{align}
Then, by initializing $y_0$ and $z_0$ such that $\mathbb{E}[\|{y}_{0} - y_{\lambda}^{*}({{x}}_{0})\|^2]\leq O(\epsilon^2)$ and $\mathbb{E}[\|{z}_{0} - y^{*}({{x}}_{0})\|^2]\leq O(\epsilon^2)$, it is easy to know that
\begin{align}
	& \quad  \frac{1}{T}\sum_{t=0}^{T-1}  \mathbb{E}[\|\nabla \mathcal{L}(\bar{x}_{t})\|^2]  \notag  \\
	& \leq  O\left(\epsilon^2\right) \mathbb{E}[\mathcal{L}_{\lambda}^{*}(\bar{x}_{0}) ] +O\left(\kappa^2\epsilon^3\right)   + O\left(\kappa^2\epsilon^3\right)  \notag  \\
	& \quad + O\left(\epsilon^2\right) +O\left(\frac{\epsilon^4}{\kappa^8}\right)\sigma^2 +O\left(\frac{\epsilon^2}{\kappa^2}\right)\sigma^2 + O\left(\frac{\epsilon^4}{\kappa^6}\right)\sigma^2 + O\left(\epsilon^2\right) \sigma^2+O\left(\epsilon^2\right) \sigma^2 \notag \\
	& \quad +4\Bigg[O\left(\frac{1}{\kappa^4}\right) +O\left(1\right)  +O\left(1\right)  + O\left(\frac{\epsilon^2}{\kappa^{10}}\right)    + O\left(\frac{\epsilon^2}{\kappa^{6}}\right)  + O\left(\frac{\epsilon^2}{\kappa^{6}}\right)   \Bigg]\notag \\
	& \quad \quad \times \left(O\left(\frac{\epsilon^{4}}{\kappa^{8}}\right)\delta^2 +O\left(\frac{\epsilon^{4}}{\kappa^{8}}\right)\delta^2+ O\left(\frac{\epsilon^{4}}{\kappa^{8}}\right)\delta^2+ O\left(\frac{\epsilon^{4}}{\kappa^{8}}\right)\delta^2\sigma^{2} + O\left(\frac{\epsilon^{4}}{\kappa^{8}}\right)\delta^2\sigma^{2} + O\left(\frac{\epsilon^{4}}{\kappa^{8}}\right)\delta^2\sigma^{2}\right)\notag  \\
	& \quad + O\left(\frac{\epsilon^{4}}{\kappa^{8}}\right)\sigma^2+ O\left(\frac{\epsilon^{2}}{\kappa^{2}}\right)\sigma^2 + O\left(\frac{\epsilon^{2}}{\kappa^{2}}\right)\sigma^2  +  O\left(\frac{\epsilon^{4}}{\kappa^{6}}\right)\sigma^2+ O\left(\epsilon^{2}\right)\sigma^2 + O\left(\epsilon^{2}\right)\sigma^2  \ ,
\end{align}
which indicates 
\begin{align}
    & \frac{1}{T}\sum_{t=0}^{T-1}  \mathbb{E}[\|\nabla \mathcal{L}(\bar{x}_{t})\|^2]  \leq O\left(\epsilon^2\right)  \ . 
\end{align}


\end{proof}

